\documentclass[twoside,11pt]{article}
\usepackage{blindtext}
\usepackage{booktabs} 

 \usepackage[abbrvbib, preprint]{jmlr2e}

\usepackage{tikz}
\usepackage{tikz-cd}
\usepackage{caption}
\usepackage{natbib}
\usepackage{times}
\usepackage[T1]{fontenc}
\usepackage{lmodern}
\usepackage{microtype}
\usepackage{thmtools} 
\usepackage{thm-restate}

\usepackage{svg}
\usepackage{bbm}
\usepackage{xspace}
\usepackage{booktabs}
\newcommand{\nth}[1]{#1\textsuperscript{th}\xspace} 

\usepackage[arrow]{xy}

\let\oldS\S
\renewcommand{\S}{\oldS\!}

\usepackage{mathtools,  verbatim}
\usepackage{array}
\newcolumntype{R}{>{$}r<{$}} 
\newcolumntype{C}{>{$}c<{$}} 
\newcommand\numberthis{\addtocounter{equation}{1}\tag{\theequation}}

\newcommand{\reals}{\mathbb{R}}

\newcommand{\prop}[1]{\mathrm{prop}[#1]}

\newcommand{\faces}{\mathrm{faces}}

\newcommand{\simplex}{\Delta_\Y}

\newcommand{\A}{\mathcal{A}}
\newcommand{\B}{\mathcal{B}}
\newcommand{\C}{\mathcal{C}}

\newcommand{\E}{\mathbb{E}}
\newcommand{\F}{\mathcal{F}}

\newcommand{\I}{\mathcal{I}}

\newcommand{\R}{\mathcal{R}}
\newcommand{\Sc}{\mathcal{S}}
\newcommand{\U}{\mathcal{U}}
\newcommand{\V}{\mathcal{V}}
\newcommand{\X}{\mathcal{X}}
\newcommand{\Y}{\mathcal{Y}}

\newcommand{\inprod}[2]{\langle #1, #2 \rangle}
\newcommand{\relint}{\mathrm{relint}}

\newcommand{\toto}{\rightrightarrows}

\newcommand{\conv}{\mathrm{conv}\,}
\newcommand{\cone}{\mathrm{cone}\,}

\newcommand{\ones}{\mathbbm{1}}

\newcommand{\sign}{\mathrm{sign}}
\newcommand{\signstar}{\mathrm{sign}^*}
\renewcommand{\emptyset}{\varnothing}
\newcommand{\onespi}[2]{\ones_{#1,#2}}
\renewcommand{\boxed}[1]{\overline{#1}}
\newcommand{\Vfaces}{\V^{\textnormal{face}}}

\newcommand{\ellabs}{\ell_{\textnormal{abs}}^f}
\newcommand{\ellabsg}{\ell_{\textnormal{abs}}^g}
\newcommand{\fzo}{{f^{\textnormal{0-1}}}}
\newcommand{\Spi}[1]{\{\pi_1,\ldots,\pi_{#1}\}}

\newcommand{\ignore}[1]{}

\DeclareMathOperator*{\argmin}{arg\,min}

\newcommand{\Jac}{\mathrm{Jac}}

\newcommand{\mis}{\mathrm{mis}}
\newcommand{\abs}{\mathrm{abs}}
\newcommand{\Lvecf}{L^{\vec f}}

\newcommand{\ellvecf}{\ell^{\vec f}}

\newcommand{\ellabsvecf}{\ell_{\textnormal{abs}}^{\vec f}}
\newcommand{\ellabsvecg}{\ell_{\textnormal{abs}}^{\vec g}}

\newcommand{\Iue}{\I^\epsilon_u}

\newcommand{\elljac}{\ell^{\vec J}}

\newcommand{\block}[1]{\llcorner #1 \urcorner}

\newcommand{\ellabsfdy}[1]

\newcommand{\psithresh}{\psi^\tau}

\newtheorem{example}{Example} 
\newtheorem{theorem}{Theorem}
\newtheorem{lemma}[theorem]{Lemma} 
\newtheorem{proposition}[theorem]{Proposition}

\newtheorem{definition}[theorem]{Definition}

\renewcommand{\boxed}[1]{\overline{#1}}
\renewcommand{\emptyset}{\varnothing}

\newtheorem{lemmac}[theorem]{Lemma}
\newtheorem{propositionc}[theorem]{Proposition}

\newtheorem{condition}{Condition}

\usepackage{lastpage}
\jmlrheading{27}{2026}{1-\pageref{LastPage}}{6/25; Revised 6/25}{6/25}{21-0000}{Jessie Finocchiaro, Rafael M. Frongillo, Enrique Nueve}

\ShortHeadings{Structured Prediction with Abstention via the Lov\'asz Hinge}{Structured Prediction with Abstention via the Lov\'asz Hinge}
\firstpageno{1}

\begin{document}


\title{Structured Prediction with Abstention via the Lov\'asz Hinge}

\author{\name Jessie Finocchiaro  \email finocch@bc.edu \\
       \addr Department of Computer Science\\
       Boston College
       \AND
       \name Rafael Frongillo \email raf@colorado.edu \\
       \addr Department of Computer Science\\
       University of Colorado Boulder
	  \AND
	   \name Enrique Nueve \email enrique.nueveiv@colorado.edu \\
       \addr Department of Computer Science\\
       University of Colorado Boulder 
}

\editor{My editor}

\maketitle


\begin{abstract}
  The Lov\'asz hinge is a convex loss function proposed for binary structured classification, in which $k$ related binary predictions jointly evaluated by a submodular function.
  Despite its prevalence in image segmentation and related tasks, the consistency of the Lov\'asz hinge has remained open.
  We show that the Lov\'asz hinge is inconsistent with its desired target unless the set function used for evaluation is modular.
  Leveraging the embedding framework of \citet{finocchiaro2024embeddingJMLR}, we find the target loss for which the Lov\'asz hinge is consistent.
  This target, which we call the structured abstain problem, is a variant of selective classification for structured prediction that allows one to abstain on any subset of the $k$ binary predictions.
  We derive a family of link functions, each of which is simultaneously consistent for all polymatroids, a subset of submodular set functions.
  We then give sufficient conditions on the polymatroid for the structured abstain problem to be tightly embedded by the Lov\'asz hinge, meaning no target prediction is redundant.
  We experimentally demonstrate the potential of the structured abstain problem for interpretability in structured classification tasks. 
  Finally, for the multiclass setting, we show that one can combine the binary encoding construction of \citet{ramaswamy2018consistent} with our link construction to achieve an efficient consistent surrogate for a natural multiclass generalization of the structured abstain problem.
\end{abstract}

\begin{keywords}
  Structured Prediction, Consistency, Classification with Abstention, Selective Classification, Polyhedral Loss
\end{keywords}

\section{Introduction}

Structured prediction addresses a wide variety of machine learning tasks in which the error of several related predictions is best measured jointly, according to some underlying structure of the problem, rather than independently~\citep{osokin2017structured,gao2011consistency,hazan2010direct,tsochantaridis2005large}.
This structure could be spatial (e.g., images and video), sequential (e.g., text), combinatorial (e.g., subgraphs), or a combination of the above.
As traditional target losses measure error independently, such as 0-1 loss, more complex target losses are often introduced to capture the joint structure of these problems.

As with most classification-like settings, optimizing a given discrete target loss is typically intractable.
We therefore seek surrogate losses which are both convex, and thus efficient to optimize, and statistically consistent, roughly meaning that minimizing the surrogate and applying a link function yields the same predictions as if one had minimized the discrete loss directly.
A third important consideration, especially in structured prediction, is the dimension of the surrogate prediction space.
In structured prediction, the number of possible labels and/or target predictions is often exponentially large.
For example, in the structured binary classification problem, one makes $k$ simultaneous binary predictions, yielding $2^k$ possible labels.
In these settings, it is crucial to find a surrogate whose prediction space is low-dimensional relative to the number of labels.

In general, however, we lack surrogates satisfying all three desiderata: convex, consistent, and low-dimensional~\citep{mcallester2007generalization,nowozin2014optimal}.
One promising low-dimensional surrogate for structured binary classification, the Lov\'asz hinge, is convex via the Lov\'asz extension for submodular set functions~\citep{yu2018lovasz,yu2015lovaszarxiv}.
Despite the fact that this surrogate and its generalizations~\citep{berman2018lovasz} have been widely used, e.g.\ in image segmentation and processing~\citep{athar2020stem,chen2020afod,neven2019instance}, its consistency has thus far not been established.

Using the embedding framework of \citet{finocchiaro2024embeddingJMLR}, we show the inconsistency of Lov\'asz hinge for structured binary classification (\S~\ref{sec:inconsistency}).
Our proof proceeds by first determining a problem for which the Lov\'asz hinge is actually consistent: the \emph{structured abstain problem}, a variation of structured binary prediction in which one may abstain on a subset of the predictions (\S~\ref{sec:our-embedding}).
While the embedding framework shows that a calibrated link must exist, in our case actually deriving such a link is nontrivial.
We derive a family of link functions, each of which is calibrated simultaneously for all polymatroids, the subset of submodular functions under consideration (\S~\ref{sec:constructing-link}).
We also prove that the structured abstain problem is ``tightly'' embedded by the Lov\'asz hinge, after culling joint predictions that abstain on exactly one individual prediction (\S~\ref{sec:tightness}).
Turning to experiments, we demonstrate that abstention regions often correspond to regions with high uncertainty, which could be useful for interpretability in structured classificaton tasks (\S~\ref{sec:experiments}).
Finally, in the multiclass setting, we combine the binary encoding construction of \citet{ramaswamy2018consistent} with our link construction to achieve a dimension-efficient and consistent surrogate for a natural multiclass generalization of the structured abstain problem (\S~\ref{sec:multiclass}).

\section{Background}

We begin with notation, relevant definitions, and running examples. 

\subsection{Notation}

See Tables~\ref{tab:notation} and~\ref{tab:notation-proofs} in \S~\ref{app:omitted-proofs} for full tables of notation.
Throughout, we consider joint predictions over $k$ related individual predictions, yielding $n = 2^k$ total labels, each with label $y \in \Y := \{-1,1\}^k$.
Predictions are denoted $r\in\R$; we often take $\R=\Y$, or consider predictions $v \in \V := \{-1,0,1\}^k$ or $u \in \reals^k$.
Loss functions measure the quality of these predictions against the observed label $y\in\Y$.
In general, we denote a discrete loss $\ell : \R \times \Y \to \reals_+$ and surrogate $L : \reals^k \times \Y \to \reals_+$.
We also occasionally restrict a general loss $L$ to a domain $\Sc \subseteq \R$ and define $L|_{\Sc} : (u,y) \mapsto L(u,y)$ for all $u \in \Sc$.

Let $[k] := \{1,\ldots,k\}$.
When translating from vector functions to set functions, we use the shorthand $\{ u \leq c \} := \{ i \in [k] \mid u_i \leq c\}$ for $u\in\reals^k$, $c\in\reals$.
Additionally, for any $S\subseteq [k]$, we let $\ones_S \in \{0,1\}^k$ with $(\ones_S)_i = 1 \iff i\in S$ be the 0-1 indicator for $S$.
Let $\Sc_k$ denote the set of permutations of $[k]$.
For any permutation $\pi \in \Sc_k$, and any $i\in\{0,1,\ldots,k\}$, define $\onespi{\pi}{i} = \ones_{\Spi{i}}$, where $\onespi{\pi}{0} = 0 \in \reals^k$.

For $u,u'\in\reals^k$, the Hadamard (element-wise) product $u\odot u'\in\reals^k$ is given by $(u \odot u')_i = u_iu'_i$.
We extend $\odot$ to sets in the natural way; e.g., for $U\subseteq \reals^k$ and $u'\in\reals^k$, we define $U \odot u' = \{u\odot u' \mid u\in U\}$.

We often decompose elements of $u \in \reals^k$ by their sign and absolute value.
To this end, we define $\sign : \reals^k \to \V$ to be the (element-wise) sign of $u$, noting $\sign(0) = 0$.
In contrast, we use the function $\signstar : \reals^k \to \Y$ to denote an arbitrary function that agrees with $\sign$ when $|u_i|\neq 0$ and break ties arbitrarily (but deterministically) at $0$.
We frequently use the fact that $|u| = u\odot\signstar(u) = u\odot\sign(u)$.
We define $\boxed u = \sign(u) \odot \min(|u|, \ones)$ to ``clip'' $u$ to $[-1,1]^k$.
Finally, for $u\in\reals^k$ we define $u_+\in\reals^k$ by $(u_+)_i = \max(u_i,0)$.

\subsection{Submodular functions and the Lov\'asz extension}

A set function $f:2^{[k]}\to\reals$ is \emph{submodular} if for all $S,T\subseteq [k]$ we have $f(S) + f(T) \geq f(S\cup T) + f(S\cap T)$.
If this inequality is strict whenever $S$ and $T$ are incomparable, meaning $S\not\subseteq T$ and $T\not\subseteq S$, then we say $f$ is \emph{strictly submodular}.
A function is \emph{modular} if the submodular inequality holds with equality for all $S,T\subseteq [k]$.
The function $f$ is \emph{increasing} if we have $f(S\cup T) \geq f(S)$ for all disjoint $S,T\subseteq [k]$, and \emph{strictly increasing} if the inequality is strict whenever $T\neq\emptyset$.
Also, we say $f$ is \emph{normalized} if $f(\emptyset) = 0$.
Holistically, we say $f$ is a  \emph{polymatroid} if it is submodular, normalized, and increasing for all $S, T \subseteq [k]$ \citep{bach2013learning}.

We let $\vec f = \{f_y \mid 2^{[k]}\to\reals_{+}\}_{y\in\Y}$ denote a collection of polymatroid functions, so that each may be differently conditioned on $y\in\Y$, where recall $\Y = \{-1,1\}^k$.
Denote the set of polymatroid collections by $\vec \F_k$.
Lastly, we let $\F_k \subset 
\vec \F_k$ to be the set of polymatroids $\vec f\in\vec\F_k$ where $f_y$ is the same for all $y$; formally, $\F_k= \{\vec{f}\mid f_y =f_{y'} \;\forall \; y,y'\in \Y\}$.
Given that polymatroids in $\mathcal{F}_k$ do not depend on $y$, we will sometimes write $f \in \mathcal{F}_k$.
See \S~\ref{polyexamples} for examples differentiating $\vec \F_k$ from $\F_k$.

The structured binary classification problem is given by the discrete loss $\ell^{\vec f}:\R\times\Y\to\reals$ with $\R=\Y$, given by
\begin{equation}
\label{eq:discrete-set-loss}
\ell^{\vec f}(r,y) = f_y(\{ r\odot y < 0 \}) = f_y(\mis (r,y))~,
\end{equation}
where $\mis (r,y)=\{i \in [k] \mid   r_i\neq y_i \}$.
In words, $\ell^{\vec f}$ measures the joint error of the $k$ predictions by applying $f_y$ to the set of mispredictions when $y$ is the true label, i.e., indices corresponding to incorrect individual binary predictions.

A classic object related to submodular functions is the \emph{Lov\'asz extension} to $\reals^k$~\citep{lovasz1983submodular}, which is known to be convex when (and only when) $f$ is submodular~\cite[Proposition 3.6]{bach2013learning}.
For any permutation $\pi\in\Sc_k$, define $P_\pi = \{x\in\reals^k_+ \mid x_{\pi_1} \geq \cdots \geq x_{\pi_k}\}$, the set of nonnegative vectors ordered by $\pi$.
The Lov\'asz extension $F$ of a polymatroid function $f:2^{[k]}\to\reals$ can be formulated in several equivalent ways \citep[Definition 3.1]{bach2013learning}, including the following which we take to be our definition:
\begin{equation}\label{eq:lovasz-ext}
  F(x) = \max_{\pi\in \Sc_k} \sum_{i=1}^k x_{\pi_i} (f(\Spi{i})-f(\Spi{i-1}))~.
\end{equation}
Given any $x\in\reals^k_+$, the argmax in eq.~\eqref{eq:lovasz-ext} is the set $\{\pi\in\Sc_k \mid x \in P_\pi\}$, i.e., the set of all permutations that order the elements of $x$.
For any $\pi \in \Sc_k$ such that $x\in P_\pi$, we may therefore write
\begin{equation}
\label{eq:lovasz-ext-pi-u}
F(x) = \sum_{i=1}^k x_{\pi_i} (f(\Spi{i})-f(\Spi{i-1}))~.
\end{equation}

For fixed $\vec f$, let $F_y$ denote the Lov\'asz extension for any $f_y \in \vec f$.
\citet{yu2018lovasz} define the \emph{Lov\'asz hinge} as the loss $\Lvecf:\reals^d\times\Y\to\reals_+$ given as follows.
\begin{equation}
\label{eq:lovasz-hinge}
\Lvecf (u,y) =  F_y((\ones - u \odot y)_+)~.
\end{equation}
If $\vec{f}\in \vec \F_k \setminus \F_k$, the evaluation of $F_y$ is dependent on the label $y\in\Y$ given to $\Lvecf$.
The Lov\'asz hinge is proposed as a surrogate for the structured binary classification problem in eq.~\eqref{eq:discrete-set-loss}, using the link $\signstar$ to map surrogate predictions $u \in \reals^k$ back to the discrete prediction space $\R = \Y$.
From eq.~\eqref{eq:lovasz-ext}, the Lov\'asz extension is polyhedral (piecewise-linear and convex) as the pointwise maximum of a finite collection of affine functions.
Since this statement holds for each $y \in \Y$, the loss $L^{\vec f}$ is polyhedral as well.

\subsection{Running examples}\label{polyexamples}


We will refer to three running examples of submodular functions and their corresponding surrogates.

\subsubsection{Weighted Hamming loss}\label{sec:wh-loss}

First, consider the case where $f \in \F_k $  is modular. 
Modular set functions can be parameterized by any $w\in\reals^k_+$, so that $f^{w} (S) = \sum_{i\in S} w_i \in \F_k$. 
In this case $\ell^f$ reduces to \emph{weighted Hamming loss}, and $L^f$ to weighted hinge, the consistency of which is known~\cite[Theorem 15]{gao2011consistency}.
\begin{align}
  L^{f^w}(u,y)
  &= \max_{\pi\in \Sc_k} \sum_{i=1}^k ((1-u\odot y)_+)_{\pi_i} (f^{w} (\Spi{i})-f^{w} (\Spi{i-1}))
    \nonumber
  \\
  &= \sum_{i=1}^k (1-u_i y_i)_+ (w_i)~.
  \label{eq:weighted-hamming}
\end{align}

\subsubsection{The Binary Encoded Prediction (BEP) surrogate}\label{sec:bep-example}
For the other example, consider $\fzo$ given by $\fzo(\emptyset)=0$ and $\fzo(S)=1$ for all $S\neq\emptyset$.
Here the Lov\'asz hinge reduces to
\begin{align}
  L^{\fzo}(u,y)
  &= \max_{\pi\in \Sc_k} \sum_{i=1}^k ((1-u\odot y)_+)_{\pi_i} (\fzo(\Spi{i})-\fzo(\Spi{i-1}))
    \nonumber
  \\
  &= \max_{\pi \in \Sc_k} \left((1 - u \odot y)_+\right)_{\pi_1} \nonumber \\
  &= \max_{i\in[k]} \; (1-u_i y_i)_+~.
  \label{eq:BEP-simple-form}
\end{align}

$L^\fzo$ is equivalent to the BEP surrogate proposed by \citet{ramaswamy2018consistent} for multiclass classification with an abstain option.
The target loss for this problem is $\ell_{1/2}:[n]\cup \{\bot \} \times [n] \to \reals_{+}$ defined by $\ell_{1/2}(r,y)=0$ if $r=y$, $1/2$ if $r=\bot$, and $1$ otherwise.
Here, the prediction $\bot$ corresponds to ``abstaining'' if no $y\in \Y$  has $p_y \geq 1/2$.
The BEP surrogate is given by 
\begin{align}
  \label{eq:BEP}
L_\frac{1}{2}(u,\hat y) &= \left(\max_{j\in [k]}\; B(\hat y)_j u_j +1\right)_+
\end{align}
where $B:[n]\rightarrow \{-1,1\}^k$ is an arbitrary injection.
Substituting $y=-B(\hat y)$ in eq.~\eqref{eq:BEP}, and moving the $(\cdot)_+$ inside, we recover eq.~\eqref{eq:BEP-simple-form}.


\subsubsection{Jaccard loss}
A common metric for image segmentation is Jaccard loss, which is defined as 1 minus the Intersection over Union (IoU) between the sets of predicted and true foreground pixels.
Unlike the previous two examples $f^{w}$ and $\fzo$, which were both elements of $\F_{k}$, the polymatroid corresponding to Jaccard loss depends nontrivially on $y$~\citep{yu2018lovasz,yu2015lovaszarxiv}.

Formally, let $k$ be the number of pixels in an image.
The set of possible labels is $\Y = \{-1,1\}^k$ where each $y\in\Y$ corresponds to a segmentation: $y_i =1$ means the $i$-th pixel is labeled foreground, and labeled background if $y_i=-1$.
Jaccard loss $\Jac:\Y\times\Y\to\reals_+$ is given by 
\begin{align}\label{eq:jaccard-loss-mispred}
\Jac(y',y) = \frac{|\mis(y',y)|}{|\{i\in[k] \mid  y_i = 1\}\cup \mis(y',y)|}~,
\end{align}
where $y'\in \Y$ expresses a prediction and $y$ is the true label, with the convention $0/0=0$.
For fixed $y$, Jaccard loss is submodular in the set $\mis(y',y)$ \citep{berman2018yes}. 
We can rewrite eq.~\eqref{eq:jaccard-loss-mispred} as $\Jac(y',y) = J_y(\mis(y',y))$ where $J_y(S) = |S|/|S\cup\{i\in[k] \mid  y_i=1\}|$, where $S\in 2^{[k]}$ corresponds to the set of mispredictions.
In other words, we have $\mathrm{Jac} = \elljac$ for $\vec J = \{J_y\}_{y\in\Y}$.


\subsection{Property elicitation and calibration}

Property elicitation is a useful tool to analyze the consistency of surrogates.
Indeed, for polyhedral (piecewise linear and convex) surrogates such as the Lov\'asz hinge~\eqref{eq:lovasz-hinge}, \citet[Theorem 8]{finocchiaro2022embedding} shows that indirect property elicitation is equivalent to statistical consistency.
Informally, a property is elicited by a loss if it captures the exact minimizers of the loss given a (conditional) label distribution.

\begin{definition}[Elicitation]\label{def:prop-elicits}
	A property $\Gamma : \simplex \to 2^{\R} \setminus \{\emptyset\}$ is a function mapping distributions over labels to predictions.
	A loss $L : \R \times \Y \to \reals_+$ \emph{elicits} a property $\Gamma$ if, for all $p \in \simplex$,
	\begin{align*}
	\Gamma(p) = \argmin_{r \in \R} \E_{Y \sim p}L(r,Y)~.~
	\end{align*}
	Moreover, if $\E_{Y \sim p}L(\cdot,Y)$ attains its infimum for all $p \in \simplex$, we say $L$ is \emph{minimizable}, and elicits some unique property, denoted $\prop{L}$.
\end{definition}

In order to connect property elicitation to statistical consistency, we work through the notion of calibration, which is equivalent to consistency in our setting~\citep{bartlett2006convexity,zhang2004statistical,ramaswamy2016convex}.
One advantage of calibration over consistency is that it eliminates the need to consider features $x \in \X$, allowing us to focus on expected loss over labels through the conditional distribution $p \in \simplex$.
We often denote $L(u;p) := \E_{Y \sim p} L(u,Y)$, and $\ell(r;p) := \E_{Y \sim p}\ell(r,Y)$.

\begin{definition}[Calibration]\label{def:calibration}
	Let $\ell:\R\times\Y\to\reals$ with $|\R|<\infty$.
	A surrogate $L : \reals^d \times \Y \to \reals_+$ and link $\psi : \reals^d \to \R$ pair $(L, \psi)$ is \emph{calibrated} with respect to $\ell$ if for all $p\in\simplex$,
	\begin{align*}
	\inf_{u : \psi(u) \not \in \prop{\ell}(p)} L(u;p) > \inf_{u \in \reals^d} L(u;p)~.
	\end{align*}
\end{definition}

We use calibration as an equivalent property to study statistical consistency in this paper,
and when restricting to polyhedral losses, (indirect) property elicitation implies calibration.

\begin{definition}[Indirect elicitation]\label{def:indirectelicit}
Let minimizable loss $L:\reals^{d}\times \Y \to \reals_{+}$ and link $\psi :\reals^{d} \to  \R$ be given.
The pair $(L,\psi )$ indirectly elicits a property $\gamma :\Delta_{\Y}\toto \R$ if for all $u\in \reals^{d}$, we have $\Gamma_{u}\subseteq \gamma_{\psi (u)}$, where $\Gamma =\prop{L}$.
Moreover, we say $L$ indirectly elicits $\gamma$ if such a $\psi$ exists, i.e., if for all $u\in \reals^{d}$ there exists $r\in \R$ such that $\Gamma_{u}\subseteq \gamma_{r}$.
\end{definition}

A useful observation is that, when target predictions are \emph{dominated}, we can eliminate them from the link function.

\begin{lemma}\label{lemma:safe-link}
Given $(L,\psi)$ calibrated with respect to $\ell$ (eliciting $\gamma$). 
If prediction $r\in\R$ is dominated by prediction $r'\in\R$, e.g., $\gamma_{r}\subseteq \gamma_{r'}$, the the pair $(L,\psi')$ is calibrated with respect to $\ell$, where the link $\psi'$ is defined as $\psi' (u)=r'$ if $\psi (u)=r$ and $\psi'(u)=\psi(u)$ otherwise.
\end{lemma}
\begin{proof}
    Fix $p\in \Delta_{\Y}$. If $\psi (u)=v$ is not dominated, then $\psi (u)\in\gamma (p) \Leftrightarrow \psi' (u)\in \gamma (p)$ by equality. 
    If $\psi (u)=r$, then $\psi (u) \in \gamma (p) \Rightarrow \psi'(u)\in\gamma (p)$.
    Contrapositively, $\psi'(u)\notin \gamma (p) \Rightarrow \psi (u)\notin \gamma (p)$.
    Thus, $S':=\{u \mid\psi' (u)\notin \gamma (p) \}\subseteq S := \{u\mid\psi (u)\notin \gamma (p) \}$.
    As $S'\subseteq S$,
    $$\inf_{u\in S'}\mathbb{E}_{p}L(u,Y)\geq  \inf_{u\in S}\mathbb{E}_{p}L(u,Y)> \inf_{u}\mathbb{E}_{p}L(u,Y)~.$$
    As $p$ was arbitrary, $(L,\psi )$ is calibrated with respect to $\ell$.
\end{proof}

\subsection{The embedding framework}

We will lean heavily on the embedding framework of~\citet{finocchiaro2019embedding,finocchiaro2024embeddingJMLR}.
Given a discrete target loss, and a surrogate loss over $\reals^k$, an embedding maps target predictions $r \in \R$ into $\reals^k$ so that the surrogate behaves the same as the target on the embedded points.
The authors show that every polyhedral surrogate embeds some discrete loss, and show that an embedding implies consistency.
To define embeddings, we first need a notion of representative sets, which allows one to ignore some target predictions that dominated.

\begin{definition}\label{def:representative-set}
Let $\Gamma :\Delta_{\Y}\toto \R$.
We say $\Sc \subseteq \R$ is \emph{representative} for $\Gamma$ if we have $\Gamma (p)\cap \Sc \neq \emptyset$  for all $p\in \simplex$.
We further say $\Sc$ is a minimum representative set if it has the smallest cardinality among all representative sets. 
Given a minimizable loss $L:\R \times \Y \to \reals_{+} $, we say $\Sc$ is a (minimum) representative set for $L$ if it is a (minimum) representative set for $\prop L$.
\end{definition}

\begin{definition}[Embedding]\label{def:loss-embed}
	The loss $L:\reals^d \times \Y \to\reals_+$ \emph{embeds} a loss $\ell:\R \times \Y\to\reals_+$ if there exists a representative set $\Sc$ for $\ell$ and an injective embedding $\varphi:\Sc\to\reals^d$ such that
	(i) $L(\varphi(r),y) = \ell(r,y)$ for all $r\in\Sc$ and $y \in \Y$, and
	(ii) for all $p\in\simplex,r\in\Sc$ we have
	\begin{equation}\label{eq:embed-loss}
	r \in \prop{\ell}(p) \iff \varphi(r) \in \prop{L}(p)~.
	\end{equation}
If $\Sc$ is a minimal representative set, we say $L$ tightly embeds $\ell$.
\end{definition}

Embeddings are intimately tied to polyhedral losses as they have finite representative sets, i.e., representative sets with finite cardinality.
A central tool for the present work, however, is the converse: every polyhedral loss embeds some discrete target loss, namely, its restriction to a finite representative set.

\begin{theorem}[{\citep[Thm.\ 3, Prop.\ 1]{finocchiaro2022embedding}}]\label{thm:poly-embeds-discrete}
  A loss $L$ with a finite representative set $\Sc$ embeds $L|_\Sc$.
  Moreover, every polyhedral $L$ has a finite representative set.
\end{theorem}

In particular, if a surrogate $L:\reals^k \times \Y \to \reals_+$ embeds a discrete target $\ell:\R\times \Y\to\reals_+$, then there exists a calibrated link function $\psi:\reals^k\to\R$ such that $(L,\psi)$ is consistent with respect to $\ell$.
The proof is constructive, via their notion of \emph{separated link} functions, a fact we will make use of in Theorem~\ref{thm:embed-implies-consistent}.

\section{Lov\'asz hinge embeds the structured abstain problem}\label{sec:our-embedding}

Since the Lov\'asz hinge is polyhedral, Theorem~\ref{thm:poly-embeds-discrete} implies that it embeds a discrete loss, which may or may not be the same as the intended target $\ellvecf$.
As we saw in \S~\ref{sec:bep-example}, one special case, $L^\fzo$, reduces to the BEP surrogate, which is consistent with respect to multiclass classification with an abstain option $\bot$.
Moreover, the abstain report $\bot$ is not redundant, which implies that $\Lvecf$ cannot embed $\ellvecf$ in general.
Intuitively, whatever $\Lvecf$ embeds, it must allow the algorithm to abstain in some sense.
We formalize this intuition by showing $\Lvecf$ embeds the discrete loss $\ellabsvecf$, a variant of structured binary classification which allows abstention on any subset of the $k$ labels.

Our proof proceeds by constructing a finite representative set $\V$ for $\Lvecf$.
We first show that the filled hypercube $[-1,1]^k$ is representative (\S~\ref{subsec:filled-hypercube-rep}).
We then provide an affine decomposition of $\Lvecf$ on this hypercube (\S~\ref{sec:affine-decomposition}).
This observation implies that, for every $p \in \simplex$, some element of $\V$ must minimize $\Lvecf(\cdot; p)$ (\S~\ref{subsec:embedding-SAP}).
Thus, $\Lvecf$ embeds $\Lvecf|_\V$.
Defining $\ellabsvecf := \Lvecf|_\V$, consistency with respect to $\ellabsvecf$ follows.
See \S~\ref{app:omitted-sec-3} for all omitted proofs.

\subsection{The filled hypercube is representative}\label{subsec:filled-hypercube-rep}

As a first step, we show that the filled hypercube $R := [-1,1]^k$ is representative for $L^{\vec f}$. We then use this fact to find a \emph{finite} representative set for $L^{\vec f}$ and apply Theorem~\ref{thm:poly-embeds-discrete}.
In fact, we show a stronger statement: surrogate predictions outside the filled hypercube $[-1,1]^k$ are dominated on each label.
Recall that $\boxed u := \min(|u|, \ones) \odot \sign (u)$ is the ``clipping'' of $u$ to the hypercube.

\ignore{\begin{lemmac}\label{lem:lovasz-hypercube-dominates}\end{lemmac}}
\begin{restatable}{lemmac}{lovaszhypercubedominates}
  \label{lem:lovasz-hypercube-dominates}
  For any $u\in\reals^k$, we have $L^{\vec f}(\boxed u,y) \leq L^{\vec f}(u,y)$ for all $y\in\Y$.
\end{restatable}
\noindent
Using this result, we may now simplify the Lov\'asz hinge.
When $u\in[-1,1]^k$, we have $(\ones - u \odot y)_+ = (\ones - u \odot y)$, so
\begin{equation}\label{eq:lovasz-hinge-simplified}
  L^{\vec f}|_{R}(u,y) = F_y(\ones - u \odot y)~.
\end{equation}

\subsection{Affine decomposition of $L^{\vec f}$}\label{sec:affine-decomposition}

Having simplfied the form of the Lov\'asz hinge, we now give an affine decomposition of $L^{\vec{f}}$ on $[-1,1]^k$, which we use to attain a finite representative set for $\Lvecf$.
Recall that for any $\pi\in\Sc_k$ we define $P_\pi = \{x\in\reals^k_+ \mid x_{\pi_1} \geq \cdots \geq x_{\pi_k}\}$ be the $x \in \reals^k_+$ ordered by the permutation $\pi$.
Letting $V_{\pi} = \{\onespi{\pi}{i} \mid i \in \{0,\ldots,k\}\} \subset \V$ be the indicators of vertices of the first $j$ coordinates of $\pi$,
we have $P_\pi = \cone (V_\pi )$, the conic hull of $V_\pi$, meaning every $x \in P_\pi$ can be written as a conic combination of elements of $V_\pi$.
For all $i\in\{0,\ldots,k\}$, define $\alpha_i : \reals^k_+ \to \reals$ as follows: for any $x\in\reals^k_+$, define $\alpha_0(x) = 1-x_{[1]} \in \reals$, $\alpha_k(x) = x_{[k]} \geq 0$, and $\alpha_i(x) = x_{[i]} - x_{[i+1]} \geq 0$ for $i\in\{1,\ldots,k-1\}$ as the conic weights defining $x$, where $x_{[k]}$ denotes the $k^{th}$ largest element of $x$.
Then
\begin{align}
x
&= \sum_{i=1}^{k} \alpha_i(x) \, \onespi{\pi}{i}
= \alpha_0(x) 0 + \sum_{i=1}^k \alpha_i(x) \onespi{\pi}{i}
= \sum_{i=0}^{k} \alpha_i(x) \, \onespi{\pi}{i}
~,
\label{eq:alpha-combination}
\end{align}
where we recall that $\ones_{\pi,0} = \vec 0 \in\reals^k$.
Since $\alpha_i(x) \geq 0$ for all $i \in \{1,\ldots,k\}$, the first equality implies that $x$ can be written as a conic combination of elements of $V_\pi$.
In the case $x_{[1]} \leq 1$, we have $\alpha_i(x) \geq 0$ for all $i\in\{0,\ldots,k\}$.
Since $\sum_{i=0}^k \alpha_i(x) = 1$, in that case the latter equality in eq.~\eqref{eq:alpha-combination} is a convex combination, and therefore,
$P_\pi \cap [0,1]^k = \conv (V_\pi )$.

It is clear from eq.~\eqref{eq:lovasz-ext-pi-u} that $F_y$ is affine on $P_\pi$ for each $\pi\in\Sc_k$.
We now identify the regions within $[-1,1]^k$ where $L^{\vec f}(\cdot,y)$ is affine for all labels $y\in\Y$ simultaneously, leveraging these polyhedra and symmetry in $y$.

Motivated by the above, for any $y\in\Y$ and $\pi\in\Sc_k$, define
\begin{align}
	\label{eq:V-pi}
  V_{\pi,y} &= V_\pi \odot y = \{\onespi{\pi}{i} \odot y \mid i \in \{0,\ldots,k\}\} \subset \V~,
  \\
  \label{eq:P-pi}
  P_{\pi,y} &= \conv (V_{\pi,y}) = \conv (V_\pi) \odot y \subset [-1,1]^k~.
\end{align}
Since $V_{\pi,y}$ is a set of affinely independent vectors, each $P_{\pi,y}$ is a simplex.
Observe that for the case $y=\ones$, we have $P_{\pi,\ones} = P_\pi \cap [0,1]^k$.
The other $P_{\pi,y}$ sets are simply reflections of $P_{\pi,\ones}$, so we may write $P_{\pi,y} = P_{\pi,\ones} \odot y$.
We now show that these regions union to the filled hypercube $[-1,1]^k$, and $L^{\vec f}(\cdot,y)$ is affine on $P_{\pi,y}$ for each $y \in \Y$.

\begin{figure}[h]
\centering
	\begin{minipage}[b]{0.45\textwidth}
	\centering
	\includegraphics[width=\textwidth]{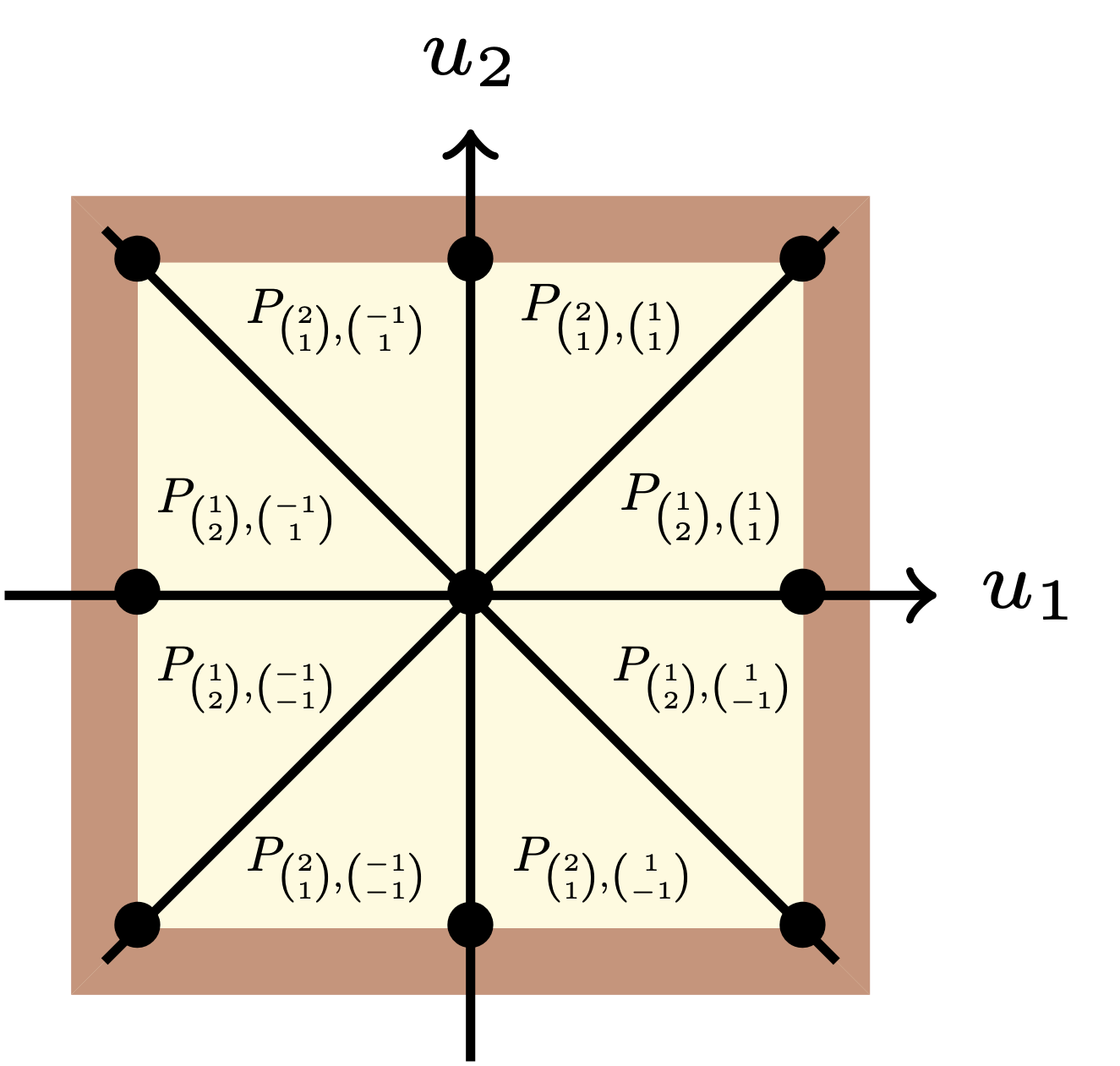}
\end{minipage}
\hfill
  \begin{minipage}[b]{0.48\textwidth}
  \centering
    \includegraphics[width=\textwidth]{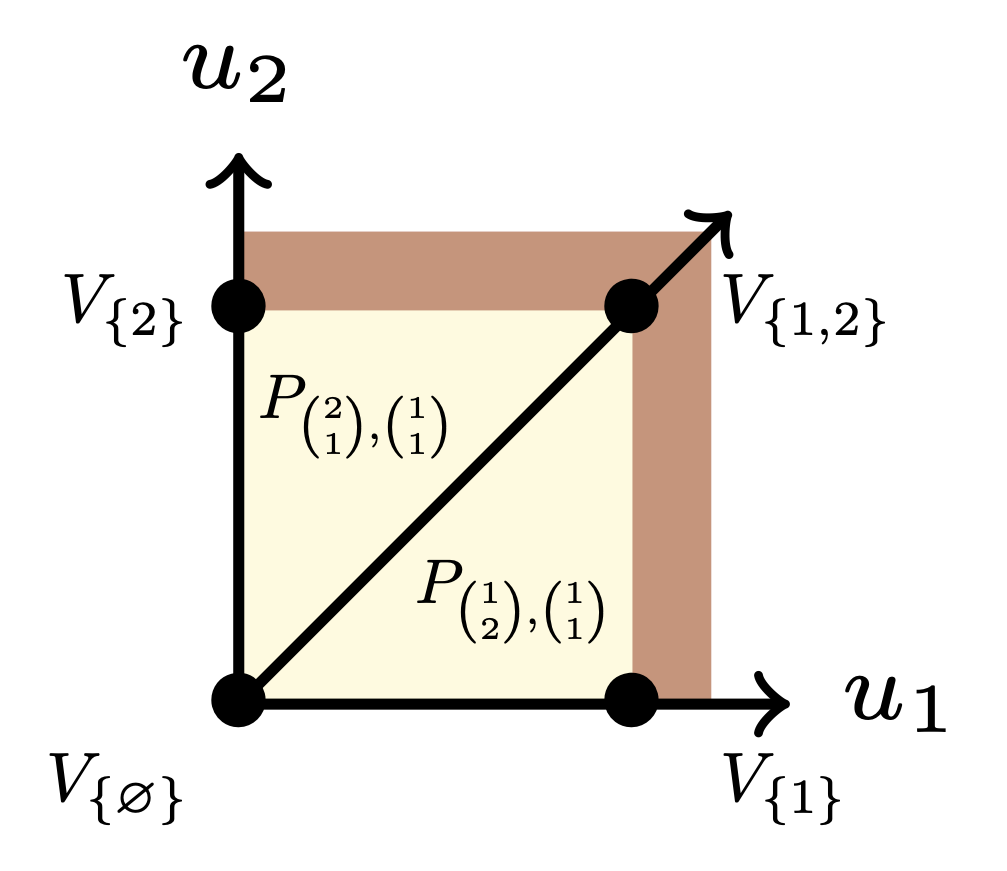}
  \end{minipage}
  \caption{(Left) A plot of all $P_{\pi ,y}$ regions  and (Right) $V_{\pi}$ labels in the positive orthant where $k=2$.}
\label{fig:pi-regions}
\end{figure}

\ignore{\begin{lemmac}\label{lem:p-pi-y}\end{lemmac}}
\begin{restatable}{lemmac}{ppiy}\label{lem:p-pi-y}
	The sets $P_{\pi,y}$ satisfy the following.
	\begin{enumerate}
		\item[(i)]  $\cup_{y\in\Y,\pi\in\Sc_k} P_{\pi,y} = [-1,1]^k$.
		\item[(ii)] For all $\vec{f} \in \vec\F_k$, $y,y',y'' \in \Y$, and $\pi\in \Sc_k$, the function $L^{\vec f}(u,y')=F_{y''}((\ones -u\odot y')_{+})$ is affine on $P_{\pi,y}$ with respect to $u\in\reals^{k}$.
	\end{enumerate}
\end{restatable}

\subsection{Embedding the structured abstain problem}\label{subsec:embedding-SAP}

Leveraging the above affine decomposition, we now show that the finite set $\V = \{-1,0,1\}^k$ is representative for $L^{\vec f}$.
By Theorem~\ref{thm:poly-embeds-discrete}, it will then follow that $L^{\vec f}$ embeds $\ellabsvecf =: L^{\vec f}|_\V$.
We call $\ellabsvecf$ the \emph{structured abstain problem} because the predictions $v\in\V$ allow one to ``abstain'' on an index $i$ by predicting $v_i = 0$.
We define the abstain set via $\abs (v) =\{i\in[k]\mid v_i=0  \}$.

\ignore{\begin{lemmac}\label{lem:UcapR-union-faces}\end{lemmac}}
\begin{restatable}{lemmac}{UcapRunionfaces}\label{lem:UcapR-union-faces}
	Given a polyhedral loss function $L : \reals^k \times \Y \to \reals_+$, let $\C$ be a collection of polyhedral subsets of $\reals^k$ such that for all $y\in\Y$, $L(\cdot,y)$ is affine on each $C\in\C$, and denote $\faces(C)$ as the set of faces of $C$.
	Let $\hat{C} = \cup \C$ be the union of these polyhedral subsets.
	Then for all $p\in\simplex$, $\prop{L}(p)\cap \hat{C} = \cup \mathfrak{F}$ for some $\mathfrak{F} \subseteq \cup_{C \in \C} \faces(C)$.
\end{restatable}

In light of Lemma~\ref{lem:p-pi-y}, we may apply Lemma~\ref{lem:UcapR-union-faces} to the Lov\'asz hinge.
Restricting to $[-1,1]^k$, and using the fact that the vertices of the $P_{\pi,y}$ regions of $\V := \{-1,0,1\}^k$, we can show that $\V$ is representative.

\begin{proposition}\label{prop:lovasz-V-rep}
For any collection of polymatroid functions $\vec f$,
$\V=\{-1,0,1 \}^k$ is a finite representative set for $\Lvecf$.
\end{proposition}


\begin{proof}
    Fix $\vec{f}=\{f_y\}_{y\in\Y}\in \vec{\F_k}$ and 
    define $L^{\vec{f}}$ as the Lov\'asz hinge on $\vec f$.
    Let $\C=\{P_{\pi,y} \mid\forall \; \pi \in \Sc_k,y\in\Y\}$ and $\hat{C}=\cup \C=\cup_{\pi\in\Sc_k ,y\in\Y}P_{\pi ,y}=[-1,1]^k$ by Lemma~\ref{lem:p-pi-y}(i). 
    Since $f_{y'}$ is a polymatroid for all $y' \in \Y$, the function $ \Lvecf (\cdot,y')$ is affine on $P_{\pi, y}$ for all $y, y' \in \Y$ and $\pi\in \Sc_k$ by Lemma~\ref{lem:p-pi-y}(ii).
    Furthermore, since $\Lvecf (u;p) = \sum_{y\in\Y}p_y \Lvecf (u,y)$ is the sum of positively weighted polyhedral losses, the expected loss $\Lvecf (u;p)$ is also affine on $P_{\pi , y}$ for all $y\in \Y$ and $\pi \in \Sc_k$.
    We now apply Lemma~\ref{lem:UcapR-union-faces} to Lov\'asz hinge with $\C$ and $\hat{C}$.
    Applying Lemma~\ref{lem:UcapR-union-faces} yields $\prop{\Lvecf}(p)\cap \hat{C} = \cup \mathfrak{F}$ for all $p \in \simplex$, where $\mathfrak{F} \subseteq \cup_{\pi,y} \faces(P_{\pi,y})$.
    Recall from the construction of $P_{\pi,y}$ that for every $(\pi, y)$ and $\mathfrak{F} \in \faces(P_{\pi,y})$, there is some element $v \in \V$ such that $v$ is a facet of $F$ for some $F \in \mathfrak{F}$.
    Therefore, for all $p\in\simplex$, $\prop{\Lvecf}(p)\cap \V\neq \emptyset$, which in turn implies that $\V$ is representative for $\Lvecf$.
\end{proof}

From Proposition~\ref{prop:lovasz-V-rep} and Theorem~\ref{thm:poly-embeds-discrete}, the Lov\'asz hinge $\Lvecf$ embeds $\Lvecf|_\V$.
Our main result of this section simply rewrites $\Lvecf|_\V$ to reveal a more intuitive target problem.

\begin{theorem}
	\label{thm:lovasz-embeds}
	The Lov\'asz hinge $L^{\vec f}$ embeds $\ellabsvecf:\V\times\Y\to\reals_+$ given by
	\begin{equation}
    \label{eq:lovasz-embeds}
    \ellabsvecf (v,y) = f_y(  \mis (v,y)\setminus \abs (v) ) + f_y(  \mis (v,y) )~.
	\end{equation}
\end{theorem}
\begin{proof}
  From Proposition~\ref{prop:lovasz-V-rep} and Theorem~\ref{thm:poly-embeds-discrete}, $L^{\vec f}$ embeds $L^{\vec f}|_\V$.
  It therefore remains only to establish the set-theoretic form of $L^{\vec f}|_\V$ as the loss $\ellabsvecf$ in eq.~\eqref{eq:lovasz-embeds}.

  Let $v\in\V,y\in\Y$ be given.
  We may write
  \begin{align*}
    \ones - v\odot y = 0\cdot\ones_{\{ v\odot y > 0 \}} + 1\cdot\ones_{\{ v\odot y = 0 \}} + 2\cdot\ones_{\{ v\odot y < 0 \}}~.
  \end{align*}
	Now combining eq.~\eqref{eq:lovasz-hinge-simplified} and \citet[Prop 3.1(h)]{bach2013learning}, we may therefore write
  \begin{align*}
    L^{\vec f}(v,y)
    &= F_y(\ones - v\odot y)
    \\
    &= (2-1) f_y(\{ v\odot y < 0 \}) + (1-0) f_y(\{ v\odot y < 0 \}\cup\{ v\odot y = 0 \}) + 0 f_y([k])
    \\
    &= f_y(\{ v\odot y < 0 \}) + f_y(\{ v\odot y \leq 0 \})  \\
    &= f_y( \mis (v,y)\setminus \abs (v)  ) + f_y(   \mis (v,y))~.
  \end{align*}
\end{proof}

We can interpret $\ellabsvecf$ as a structured abstain problem, where an algorithm is allowed to abstain by predicting zero instead of $\pm 1$.
To make this interpretation more clear, for any $v \in \V$, let $r := \signstar(v)$, which is forced to choose a label $\pm 1$ for each zero prediction.
The corresponding set of mispredictions for fixed $y \in \Y$ would be $\mis (r,y)= \{ r \odot y < 0 \}$.
We can rewrite eq.~\eqref{eq:lovasz-embeds} in terms of these sets as $\ellabsvecf (v,y) = f_y(\mis (r,y) \setminus \abs (v)) + f_y(\mis (r,y) \cup \abs (v))$.
Contrasting with $\ellabsvecf(r,y) = 2 f_y(\mis (r,y)) = f_y(\mis (r,y)) + f_y(\mis (r,y))$, the abstain option allows one to reduce loss in the first term at the expense of a sure loss in the second term.
Intuitively, when there is uncertainty about the labels of a set of indices $\abs (v) \subseteq [k]$, by submodularity the algorithm would prefer to abstain on $\abs (v)$ than take a chance on predicting.

For $r\in\Y$, we have $\ellabsvecf(r,y) = 2\ellvecf(r,y)$, meaning $\ellabsvecf$ matches (twice) $\ellvecf$ on $\Y$.
If the ``abstain'' predictions $v \in \V \setminus \Y$ are dominated, then we would indeed have consistency.
However, we show in the sequel that whenever each $f_y \in \vec f$ is strictly submodular, there are situations where abstaining is uniquely optimal (relative to $\V$), implying inconsistency with respect to $\ellvecf$.


\section{Inconsistency for structured binary classification}\label{sec:inconsistency}

We have shown in \S~\ref{sec:our-embedding} that the Lov\'asz hinge embeds the structured abstain problem $\ellabsvecf$.
This result is not sufficient to conclude that the Lov\'asz hinge is inconsistent for $\ellvecf$, i.e., structured binary classification, the initially proposed task for the Lov\'asz hinge. 
In this section, we turn to proving inconsistency.

We first focus on the symmetric case, showing that the pair $(L^f, \psi)$ is inconsistent with respect to $\ell^f$ when $f$ is not modular, for any link $\psi$, and is consistent in the trivial case where every $f$ is modular and $\psi = \signstar$.
For the non-symmetric case, we prove the inconsistency of the Lov\'asz hinge and $\psi=\signstar$ for a particular class of non-symmetric polymatroids which includes Jaccard loss, a common polymatroid used with the Lov\'asz hinge for image segmentation.

Our proof technique for the symmetric case in \S~\ref{subsec:symmetric-Fk} leverages the fact that the average value of $f$ is strictly greater than $f([k]) / 2$ unless $f$ is modular.
This technique does not carry over to the non-symmetric case, however, necessitating a different approach in \S~\ref{subsec:asymmetric-Fk}.

\subsection{The symmetric case $\mathcal{F}_k$}\label{subsec:symmetric-Fk}
For $f\in\mathcal{F}_k$, we leverage the embedded loss $\ellabs$ to show that $L^{f}$ is inconsistent for its intended target $\ell^{f}$, except when $f$ is modular.
While the modular case is already well understood (\S~\ref{sec:wh-loss}), this result says that $L^{f}$ is inconsistent for all other cases.

As $L^{f}$ embeds $\ellabs$ by Theorem~\ref{thm:lovasz-embeds}, we focus on predictions $v \in \V \setminus \Y$ to show inconsistency, i.e., those that abstain on at least one individual prediction.
Intuitively, if such a prediction is ever optimal, then $(L^{ f}, \signstar)$ has a ``blind spot'' with respect to the indices in $\abs (v)$.
We can leverage this blind spot to ``fool'' $L^{ f}$, by making it link to an incorrect prediction.
In particular, we will focus on the uniform distribution $\bar p$ on $\Y$, and perturb it slightly to $p'$ to find an $L^{f}$-optimal prediction $v\in\V$ which maps to a $\ell$-suboptimal prediction $\signstar(v)$.
In fact, we will show that one can always find such a counterexample violating consistency, unless $f$ is modular.

Given our focus on the uniform distribution, denoted by $\bar p$, the following definition will be useful: for any set function $f$, let $\bar f := 2^{-k} \sum_{S\subseteq [k]} f(S) \in\reals$.
The next two lemmas relate $\bar f$ and $f([k])$ to expected loss and modularity.
The proofs follow from adding the submodularity inequality over all possible subsets, and observing that at least one of them is strict when $f$ is non-modular.

\ignore{\begin{lemmac}\label{lem:2-bar-f}\end{lemmac}}
\begin{restatable}{lemmac}{twobarf}
	\label{lem:2-bar-f}
	For all $f\in\mathcal{F}_k$ and $v \in \V$, $\ellabs(v;\bar p) \geq f([k])$.
 	For all $r \in \Y$, $\ellabs(r;\bar p) = 2\bar f$.
\end{restatable}
\vspace*{-15pt}
\ignore{\begin{lemmac}\label{lem:bar-f}\end{lemmac}}
\begin{restatable}{lemmac}{barf}
	\label{lem:bar-f}
	Let $f \in \mathcal{F}_{k}$.
	Then $\bar f \geq f([k])/2$, and $\bar f = f([k])/2$ if and only if $f$ is modular.
\end{restatable}

Typical proofs of inconsistency identify a particular pair of distributions $p,p'\in\simplex$ for which the same surrogate prediction $u$ is optimal, yet two distinct target predictions are uniquely optimal for each, $r$ for $p$ and $r'$ for $p'$.
As $u$ cannot link to both $r$ and $r'$, one concludes that the surrogate cannot be consistent.
We follow this same general approach, but face one additional hurdle: we wish to show inconsistency of $L^{f}$ for \emph{all} non-modular $f$ simultaneously.
In particular, the distributions $p,p'$ may need to depend on the choice $f$, so at first glance it may seem that such an argument would be quite complex.
We achieve a relatively straightforward analysis by defining $p,p'$ based on only a single parameter of $f$, the ratio of $f([k])$ to $\bar f$.
The optimal surrogate prediction itself may be entirely governed by $f$, but will lead to inconsistency regardless.

In Lemma~\ref{lem:lovasz-symmetry}, we observe a strong symmetry in $L^f$ that $L^{f}(u\odot y',y\odot y') = L^{ f}(u,y)$.
In what follows, we observe that $\prop{L^{f}}$ has the same symmetry.
For $p\in\Delta(\Y)$ and $r\in\Y$, define $p\odot r \in \Delta(\Y)$ by $(p\odot r)_y = p_{y\odot r}$.
\ignore{\begin{lemmac}\label{lem:lovasz-property-symmetry}\end{lemmac}}
\begin{restatable}{lemmac}{lovaszpropertysymmetry}\label{lem:lovasz-property-symmetry}
  For all $f\in\mathcal{F}_k$, $p\in\Delta(\Y)$, and $r\in\Y$, $\prop{L^{f}}(p\odot r) = \prop{L^{ f}}(p)\odot r$.
\end{restatable}

\begin{theorem}\label{thm:inconsistent-sym}
Let $\psi:\reals^k \to \Y$ be a link function.
	Then $(L^{f},\psi )$ is consistent with respect to $\ell^f$ if and only if $f$ is modular and $\psi = \signstar$.
\end{theorem}
\begin{proof}
  When $f$ is modular, we may write $f=f^w$ for some $w\in\reals_+^k$.
  Here $L^{f^w}$ is weighted hinge loss (eq.~\eqref{eq:weighted-hamming}), which is known to be consistent for $\ell^{f^w}$, which is weighted Hamming loss~\citep[Theorem 15]{gao2011consistency}.
  (Briefly, for all $p\in\simplex$ the loss $L^{f^w}(\cdot;p)$ is linear on $[-1,1]^k$, so it is minimized at a vertex $r\in\Y$.
    Hence $\Y$ is representative, so Theorem~\ref{thm:poly-embeds-discrete} gives that $L^{f^w}$ embeds $L^{f^w}|_\Y = 2\ell^{f^w}$.
	Consistency follows from Theorem~\ref{thm:embed-implies-consistent}.
	
	Now suppose $f$ is submodular but not modular.
	As $f$ is increasing, we will assume without loss of generality that $f(\{i\}) > 0$ for all $i\in [k]$, which is equivalent to $f(S) > 0$ for all $S\neq\emptyset$; otherwise, $f(T) = f(T\setminus\{i\})$ for all $T\subseteq [k]$, so discard $i$ from $[k]$ and continue.
	In particular, we have $\{\emptyset\} = \argmin_{S\subseteq [k]} f(S)$.
  Let $\epsilon= \tfrac 1 2 (1 - f([k])/2\bar{f})$ which takes values in $(0,\tfrac 1 2]$ by Lemma~\ref{lem:bar-f} and normalization.
  This value is chosen so that $(1-\epsilon)2\bar{f} = (\bar{f} + f([k])/2) > f([k])$ again by Lemma~\ref{lem:bar-f}, a fact we will use below.
  For $y\in\Y$, let $p^y = (1-\epsilon) \bar p + \epsilon \delta_y$, where again $\bar p$ is the uniform distribution, and $\delta_y$ is the point distribution on $y$.

  First, for all $r\in\Y$ with $r\neq y$, we have $\mis (r,y) \neq \emptyset = \mis (y,y) $.
  Since $\{\emptyset\} = \argmin_{S\subseteq [k]} f(S)$, we have
  \begin{align*}
    \ell^f(r;p^y)
    &= (1-\epsilon) 2 \bar f + \epsilon \, 2f( \mis (r,y) )\\
    &> (1-\epsilon)2 \bar f + \epsilon \, 2f( \mis (y,y)   )\\
    &=\ell^f(y;p^y)~,
  \end{align*}
  giving $\prop{\ell^f}(p^y) = \{y\}$.
  On the other hand, from Lemma~\ref{lem:bar-f}  and the fact that $\ellabs$ agrees with $\ell^f$, we have for all $r\in\Y$,
	\begin{align*}
    \ellabs(r;p^y)
    \geq \ellabs(y;p^y)
    =(1-\epsilon)2 \bar f
    > f([k]) = \ellabs(0;p^y)~.
	\end{align*}
	We conclude there exists some optimal prediction $v \in \prop{\ellabs}(p^y) \setminus \Y$.
  By Theorem~\ref{thm:lovasz-embeds}, $v\in\prop{L^f}(p^y)$ as well.

  As $v\notin\Y$, in particular, $\{v = 0\} \neq \emptyset$.
  Now define $y'\in\Y$ which disagrees with $y$ on $\{i : v_i = 0\}$; formally,
$y'_i = v_i$ if $v_i \neq 0$ and $y'_i = -y_i$ if $v_i = 0$.
  Although $y'\neq y$ (as $\{v=0\} \neq \emptyset$), we have by construction that $v \odot (y \odot y') = v$.
  Furthermore, $p^y \odot (y\odot y') = p^{y'}$.
  By Theorem~\ref{thm:lovasz-embeds} and Lemma~\ref{lem:lovasz-property-symmetry} then, $v \in \prop{L^f}(p^{y'})$.
  As above, however, we also have $\{y'\} = \prop{\ell^f}(p^{y'})$.
  As $\psi(v)$ cannot be both $y$ and $y'$, at least one of $p^y$ and $p^{y'}$ exhibits the inconsistency of $L^f$ for $\ell^f$.
  Specifically, calibration is violated (Definition~\ref{def:calibration}) as $v$ achieves the optimal $L^f$-loss for both $p^y$ and $p^{y'}$, but for at least one, $\psi$ links to a prediction not in $\prop{\ell^f}$.
\end{proof}


\subsection{The non-symmetric case $\vec{\mathcal{F}}_k$}\label{subsec:asymmetric-Fk}

As discussed above, the Lov\'asz hinge is generally not consistent for its desired target.
However, Theorem \ref{thm:inconsistent-sym} only showed inconsistency of $( L^f, \signstar)$ in symmetric case $f \in \F_k$. 
Notably, Theorem \ref{thm:inconsistent-sym} does not disprove consistency via $\signstar$ for the non-symmetric setting $\vec f \in \vec\F$ which includes the commonly used Jaccard loss.
We now extend the inconsistency of the Lov\'asz hinge over a particular subset of non-symmetric $\vec\F$ which includes the Jaccard loss.

In trying to extend the proof of Theorem~\ref{thm:inconsistent-sym}, we immediately encounter an issue: even beyond the modular case, there are degenerate choices of non-modular $\vec f$ for which the Lov\'asz hinge is trivially consistent.

\begin{example}
For all $y\in\Y$, let $f_y(S) = 1$ if $\{i\in[k] \mid  y_i =1\} \subseteq S$, and $0$  otherwise.
For intuition, if we consider image segmentation as a running example, this submodular loss assigns error $1$ if any foreground pixel is mispredicted, and $0$ otherwise, which is an indicator of perfect foreground-pixel recall.
Then for instance $\ellvecf(\ones,y) = f (\{i : y_i = -1\}) =  \fzo (\{\{i : y_i = 1\} \subseteq \{i : y_i = -1\}\}) = 0$ for all $y\in\Y$, suggesting that predicting everything as foreground has perfect recall, and is therefore always optimal.
One can check that $\Lvecf$ is also optimized at $\ones$, regardless of $y$.
Thus, we trivially have calibration, and therefore consistency, by linking everything to $\ones$.
\end{example}

To proceed, therefore, we must first restrict the class of non-symmetric polymatroids $\vec \F_k$ to rule out such degenerate cases.
For $S\subseteq [k]$ let $\chi_S \in \{-1,1\}^k$ be given by $(\chi_S)_i = 1$ if $i\in S$ and $-1$ otherwise.
We denote the compliment of $S$ by $\overline S = [k]\setminus S$.
We will impose the following condition.

\begin{condition}\label{cond:inconsistency}
  For all $y\in\Y$, $S\subseteq [k]$, we have $f_{y}([k]) > f_{y}(\emptyset)$ and
  \begin{align}
    f_y(S)+f_{-y}(\overline S) \geq f_y([k])~,
  \end{align}
  with a strict inequality whenever $S \notin \{\emptyset,[k]\}$ and $y\notin\{-\ones,-\chi_S,\ones\}$.
  \end{condition}


Condition~\ref{cond:inconsistency} essentially states that (a) being completely wrong should be strictly worse than being completely correct, and (b) the sum of two complementary errors is at least the error of being completely wrong.
Condition~\ref{cond:inconsistency} implies that $f_y$ is not modular for any $y$.
  

To confirm that Condition~\ref{cond:inconsistency} is not too restrictive, we  now show that $\vec J$ satisfties Condition~\ref{cond:inconsistency}, where we recall $\vec J$ is the collection of Jaccard polymatroids on each $y \in \Y$.

\begin{lemma}\label{lem:jaccard-condition}
  The collection $\vec J$ satisfies Condition~\ref{cond:inconsistency}.
\end{lemma}
\begin{proof} 
  For ease of notation, let us reparameterize $y$ by $T = \{i\in[k]:y_i=1\}$, i.e. $y=\chi_T$, so that $J_T(S) = |S|/|S\cup T|$.
  For all $T\subseteq[k]$, we have $J_T (\emptyset)=\frac{0}{|T|} = 0$, where we recall the convention $0/0 = 0$ in this context; hence Jaccard is normalized.
  Thus, $J_T ([k])=|[k]|/|[k]|=1> 0 =J_T (\emptyset)$.
  Now for the second part of the condition, we have
  \begin{equation}
    \label{eq:1}
    J_T(S) + J_{\overline T}(\overline S) = \frac{|S|}{|S|+|T\setminus S|} + \frac{|\overline S|}{|\overline S| + |S\setminus T|}
    =
    \frac{a+c}{a+c+b} + \frac{b+d}{a+b+d}~,
  \end{equation}
  where  $a=|S\setminus T|$, $b=|T\setminus S|$, $c=|S\cap T|$, and $d=|[k]\setminus(S\cup T)| = k - (a+b+c)$.
  By the convention $0/0 = 0$, this expression evaluates to 1 whenever $S$ or $T$ are either $\emptyset$ or $[k]$.
  It also evaluates to 1 if $T = \overline S$, as then the denominators are both $k$.
  In the remaining cases, $c+d>0$, $a+c>0$, and $b+d>0$.
  Cross-multiplying and simplifying, we have $J_T(S) + J_{\overline T}(\overline S) > 1 \iff (a+c)(b+d) > ab$, which follows from these inequalities.
\end{proof}


We now show that $(\Lvecf, \signstar)$ is inconsistent with respect to $\ellvecf$ for $\vec f$ satisfying Condition~\ref{cond:inconsistency}.
Recall that the proof of inconsistency in the symmetric case (Theorem~\ref{thm:inconsistent-sym}) shows $(L^{f},\psi)$ is inconsistent with respect to $\ell^f$ for \emph{any} link $\psi$.
The crux of that proof is through indirect elicitation, showing that the surrogate level set $\prop{L^f}_0$ spans the relative interiors of multiple non-abstaining level sets $\prop{\ell^f}_y$ and $\prop{\ell^f}_{y'}$.
The identity of $y$ and $y'$ are irrelevant, because at the uniform distribution, we have $\prop{\ell^f}(\bar p) = \Y$.
In the asymmetric case, however, it is not guaranteed that $\prop{\ellvecf}(\bar p) = \Y$.
Indeed, it is not even obvious whether there is a $p\in\simplex$ such that $\prop{\ellvecf}(p) = \Y$.
In principle, one could place a stronger restriction than Condition~\ref{cond:inconsistency} on $\vec f$ to ensure the existence of such a $p$,
at which point our proof would immediately extend to any link $\psi$.
We opt instead to place no further restriction and show inconsistency with respect to $\ellvecf$ for the most commonly applied link, $\signstar$.


\begin{theorem}\label{thm:opt-ones-abs}	
If $k\geq 3$, and $\vec f$ satisfies Condition~\ref{cond:inconsistency}, then the pair $(\Lvecf,\signstar )$ is inconsistent for $\ellvecf$.
\end{theorem}
\begin{proof}
In our setting, consistency is equivalent to calibration.
To show non-calibration, we first construct a distribution $p_\epsilon$ (close to $\overline p$) such that $\{\vec 0\} = \prop{\Lvecf}(p_\epsilon)$, and therefore $\Lvecf(\vec 0; p_\epsilon) = \inf_u \Lvecf(u; p_\epsilon)$ since $\Lvecf$ always attains its infimum.
Define $\hat y := \signstar (\vec 0)$, breaking ties arbitrarily (but deterministically).
We then evaluate a few cases over $\prop{\ellvecf}(p_\epsilon)$, and construct a sequence of predictions whose loss approaches $\Lvecf(\vec 0; p_\epsilon)$, yet whose sign is not $\hat y$. 

\noindent\textbf{Find $p_\epsilon$ such that $\prop{\Lvecf}(p_\epsilon) = \{\vec0 \}$}

Because $\Lvecf$ embeds $\ellabsvecf$ via the identity, it suffices to find $p_\epsilon$ such that $\{\vec 0\} = \prop{\ellabsvecf}(p_\epsilon)$ in order to conclude that $\V\cap \prop{\Lvecf}(p_\epsilon) =\{\vec{0}\} $.
In particular, consider the distribution $p_\epsilon :=  \epsilon \delta_{-\ones}+(1-\epsilon)\overline{p}$, where $\delta_{-\ones}\in\simplex$ is the point distribution over $-\ones\in\Y$ and $\epsilon>0$ is a small value upper bounded later.

\begin{align*}
\ellabsvecf(v;p_{\epsilon}) &= \epsilon \ellabsvecf(v, -\ones) + (1-\epsilon)\ellabsvecf(v; \overline p)\\
&=\epsilon \big( f_{\emptyset}(S_v\setminus A_v)+f_{\emptyset}(S_v\cup A_v)\big)+(1-\epsilon)\ellabsvecf(v;\overline{p})~.    
\end{align*}
If $v = \vec 0$, then $S_v=\emptyset$ and $A_v=[k]$,  which implies $f_{\emptyset}(S_v \setminus A)+f_{\emptyset}(S_v \cup A)=f_{\emptyset}(\emptyset) + f_{\emptyset}([k]) = f_{\emptyset}([k])$ by normalization of the polymatroid functions.
Compare this to $v = \ones$, where $S_v=[k]$ and $A_v=\emptyset$, we get $f_{\emptyset}(S_v \setminus A)+f_{\emptyset}(S_v \cup A)=2f_{\emptyset}([k])$.
Furthermore, since $\vec{f}$ satisfies Condition~\ref{cond:inconsistency} and is a collection of polymatroids (which we recall are normalized, nonnegative, and increasing), we get
\begin{equation}\label{eq:strict-poly-cond1}
0=f_{\emptyset}(\emptyset)< f_{\emptyset}([k])<2 f_{\emptyset}([k])~.
\end{equation}

Lemma \ref{lem:opt-ones-zero-abs} in \S~\ref{app:omitted-sec-4} states that for a collection of polymatroids $\vec{f}$ which satisfy Condition~\ref{cond:inconsistency} and when $k\geq 3$, we have $\{\vec 0\} \subseteq \prop{\ellabsvecf}(\overline p) \subseteq \{\vec 0, \ones\}$.
Observe this implies that under the uniform distribution $\overline{p}$ either the prediction set $\{\vec{0}\}=\prop{\ellabsvecf}(\overline p)$ or $\{\vec{0},\ones \}=\prop{\ellabsvecf}(\overline p)$ holds true. 
In the case $\{\vec{0}\}=\prop{\ellabsvecf}(\overline p)$, inconsistency follows since the $\sign$ link is unable to map to $\vec{0}$, which is the unique optional prediction.
We proceed to show that in the scenario where $\{\vec{0},\ones \}=\prop{\ellabsvecf}(\overline p)$, $\overline{p}$ can be perturbed to form  the aforementioned distribution $p_{\epsilon}$  
where $\{\vec{0}\}$ is strictly optimal, which would make the pair $(\Lvecf,  \signstar)$ inconsistent for $\ellabsvecf$ since it would be unable to map to $\vec{0}$. 

Assuming the scenario $\{\vec{0},\ones \}=\prop{\ellabsvecf}(\overline p)$, which implies $\ellabsvecf(\ones;\overline p) = \ellabsvecf(\vec 0;\overline p)$, we can choose a sufficiently small $\epsilon >0$ such that for some corresponding $p_{\epsilon}$ we have

\begin{align}\label{eq:non-sym-inq-a}
\begin{split}
    \ellabsvecf(\ones;p_{\epsilon})
    &= 2\epsilon f_{\emptyset}([k]) +(1-\epsilon)\ellabsvecf(\ones;\overline p) 
     \\
    &= 2\epsilon f_{\emptyset}([k]) +(1-\epsilon)\ellabsvecf(\vec 0;\overline p)  \\
      &>  \epsilon f_{\emptyset}([k]) +(1-\epsilon)\ellabsvecf(\vec 0;\overline p)    \\
    &= \ellabsvecf(\vec 0;p_{\epsilon})  
\end{split}
\end{align}
where the inequality comes from  eq.  (\ref{eq:strict-poly-cond1}).

Moreover, we chose $\epsilon > 0$ sufficiently small so that
\begin{equation}\label{eq:non-sym-inq-b}
0<\epsilon\big(f_{\emptyset}(S(v)\setminus A(v))+f_{\emptyset}(S(v)\cup A(v))\big) < (1-\epsilon)\ellabsvecf(v;\overline{p})
\end{equation}
for all $v\in\V$.
Hence, for said $\epsilon$ and corresponding $p_{\epsilon}$, we have $\ones \not \in \prop{\ellabsvecf}(p_\epsilon)$, which implies $\prop{\ellabsvecf}(p_\epsilon) = \{\vec 0\}$. 
Since $f_y$ is normalized for all $y\in\Y$, in both conditions  the terms where $\epsilon$ is  multiplied by, the value is greater than zero, meaning we could make any of the terms arbitrarily small by multiplying by some $\epsilon >0$. 
Hence, an $\epsilon$ which satisfies both eq. (\ref{eq:non-sym-inq-a}) and eq. (\ref{eq:non-sym-inq-b}) inequality conditions always exists.

Now, as $\Lvecf$ embeds $\ellabsvecf$, we additionally have $\Lvecf(v;p_{\epsilon})>\Lvecf(\vec 0;p_{\epsilon})$ for all $v \in \V$, e.g., $\prop{\Lvecf}(p_\epsilon) \cap \V = \{\vec 0\}$.
As $\Lvecf$ is minimizable, this implies $\Lvecf(\vec 0; p_\epsilon) = \inf_{u'}\Lvecf(u'; p_\epsilon)$.

\noindent\textbf{Translating from embedding $\Lvecf$ to the proposed $\ellvecf$}

We now evaluate whether or not $\hat y$ optimizes the expected loss of $\ellvecf$ with respect to the distribution $p_\epsilon$.
In the case that $\hat y \not \in \prop{\ellvecf}(p_\epsilon)$, we trivially violate calibration.
However, if $\hat y \in \prop{\ellvecf}(p_\epsilon)$, we consider two cases: when $\prop{\ellvecf}(p_\epsilon) \neq \Y$ and $\prop{\ellvecf}(p_\epsilon) = \Y$.

\textbf{Case 1: $\prop{\ellvecf}(p_\epsilon) \neq \Y$.}
Consider any $y \not \in \prop{\ellvecf}(p_\epsilon)$.
Take the sequence $\{x_j\} = \{y / j\}_{j \in \mathbb{N}} \to \vec 0$ with $\sign(x_j) = y$ for all $j \in \mathbb{N}$.
Therefore, $\Lvecf(x_j; p_\epsilon) \to \Lvecf(\vec 0; p_\epsilon) = \inf_{u}\Lvecf(u; p_\epsilon)$ by continuity of $\Lvecf$.
Finally, we have
\begin{align*}
\inf_{u : y = \signstar(u) \not \in \prop{\ellvecf}(p_\epsilon)}\Lvecf(u; p_\epsilon) = \inf_u \Lvecf(u;p_\epsilon)~,
\end{align*}
which violates calibration.

\textbf{Case 2: $\prop{\ellvecf}(p_\epsilon) = \Y$.}
This case implies $\hat y = \signstar (\vec 0) \in \prop{\ellvecf}(p_\epsilon)$, as is $-\hat y$.
Then take $p_{\epsilon' } := \epsilon' p_{-\hat y}+ (1-\epsilon') p_\epsilon $ for some $\epsilon' >0$.
As $\{\vec 0\} \in \prop{\Lvecf}(p_\epsilon)$, $p_\epsilon$ is in the relative interior of the level set $\Gamma_{\vec 0}$, we can choose a sufficiently small $\epsilon' > 0$ such that $\{\vec 0\} \in \prop{\Lvecf }(p_{\epsilon'})$ for all $p_{\epsilon'} \in B(p_\epsilon, \epsilon')$: in particular, $\{\vec 0\} = \V\cap \prop{\Lvecf}(p_{\epsilon'})$, which implies
\begin{align*}
\ellvecf(\hat y; p_{\epsilon'}) &=\epsilon' f_{-\hat y}(\hat y)+(1-\epsilon') \ellvecf(\hat y;p_{\epsilon}) \\
&>\epsilon' f_{-\hat y}(-\hat y)+(1-\epsilon')\ellvecf(-\hat y;p_{\epsilon}) \\
&=\ellvecf(-\hat y;p_{\epsilon'})~,
\end{align*}
with the strict inequality following as $\ellvecf(\hat y;p_{\epsilon})=\ellvecf(-\hat y;p_{\epsilon})$.
Therefore, $\hat y = \signstar (\vec 0) \not \in \prop{\ellvecf}(p_{\epsilon'})$. 
We now have
\begin{align*}
\inf_{u : \signstar (u) \not \in \prop{\ellvecf}(p_{\epsilon'})}\Lvecf(u; p_{\epsilon'}) = \Lvecf(\vec 0; p_{\epsilon'}) = \inf_u \Lvecf(u;p_{\epsilon'})~,
\end{align*}
which again violates calibration.
\end{proof}



\section{Constructing a calibrated link for $\ellabsvecf$}\label{sec:constructing-link}

Having shown inconsistency with respect to $\ellvecf$, we now turn towards \emph{consistency} with respect to $\ellabsvecf$.
As $\Lvecf$ embeds $\ellabsvecf$ from Theorem~\ref{thm:lovasz-embeds}, the embedding framework of \citet{finocchiaro2022embedding} further implies $\Lvecf$ is consistent with respect to $\ellabsvecf$ via \emph{some} link function (see Theorem~\ref{thm:embed-implies-consistent} below).
Yet, the design of such a link function is not immediately clear.
Indeed, natural choices turn out to be inconsistent in general, such as the threshold link $\psi_c$ for $c > 0$ used by the BEP surrogate (\S~\ref{sec:bep-example}), given by $(\psi_c(u))_i=0$ whenever $|u_i|<c$ and $(\psi_c(u))_i = \sign(u_i)$ otherwise (Figure~\ref{fig:k2-links}).
We instead carefully follow the construction of an $\epsilon$-separated link from~\citet{finocchiaro2022embedding}, resulting in a family of consistent link functions.
Interestingly, these links do not depend on $\vec f$, and therefore they are calibrated with respect to $\ellabsvecf$ for all $\vec f \in \vec\F$ simultaneously.
See \S~\ref{app:omitted-proofs} for omitted proofs.

\subsection{Approach via separated link functions}
\label{sec:separ-link-funct}

Let us recount the construction of the \emph{$\epsilon$-thickened link} from~\citet{finocchiaro2024embeddingJMLR}.

\begin{definition}[{\citep[Construction 1]{finocchiaro2024embeddingJMLR}}] \label{def:eps-thick-link}
	Let a polyhedral loss $L:\reals^d\times\Y\to\reals_+$ that embeds some discrete loss $\ell:\R\times\Y\to\reals_+$ be given, along with $\epsilon > 0$, and a norm $\|\cdot\|$.
  The \emph{$\epsilon$-thickened link envelope} $\Psi:\reals^d\toto\R$ is constructed as follows.
	Define $\U = \{\prop{L}(p) : p \in \simplex\}$ and, for each $U \in \U$, let $R_U = \{r \in \R : \varphi(r) \in U\}$, the predictions whose embedding points are in $U$.
	Initialize by setting $\Psi(u) = \R$ for all $u\in\reals^d$.
	Then for each $U \in \U$, and all points $u$ such that $\inf_{u^* \in U} \|u^*-u\| < \epsilon$, update $\Psi(u) = \Psi(u) \cap R_U$.
\end{definition}

We say a link envelope $\Psi$ is nonempty pointwise if $\Psi(u) \neq \emptyset$ for all $u\in\reals^d$.
Similarly, a link function $\psi$ is pointwise contained in $\Psi$ if $\psi(u) \in \Psi(u)$ for all $u\in\reals^d$.

\begin{theorem}[{\citep[Theorems 17, 18]{finocchiaro2024embeddingJMLR}}]\label{thm:embed-implies-consistent}
  Let $L:\reals^k \times \Y \to \reals_+$ embed a discrete target $\ell:\R\times \Y\to\reals_+$, and let $\Psi$ be defined as in Definition~\ref{def:eps-thick-link}.
  Then $\Psi$ is nonempty pointwise for all sufficiently small $\epsilon$.
  Furthermore, for any link function $\psi$ pointwise contained in $\Psi$, the pair $(L,\psi)$ is consistent with respect to $\ell$.
\end{theorem}

Essentially, this construction ``thickens'' around each potentially optimal set and ensures any surrogate prediction that is close to these regions must be linked to a representative prediction contained in that set.
One can consider $\Psi$ the resulting \emph{link envelope}, from which a calibrated link may be chosen pointwise.

To apply this construction to the Lov\'asz hinge $L^{\vec f}$, let $\Psi^{\vec{f}}$ be the envelope $\Psi$ from Definition~\ref{def:eps-thick-link} applied to $L^{\vec{f}}$.
We immediately encounter a complication: 
as the link envelope $\Psi^{\vec f}$ depends on the choice of $\vec{f}$, it is entirely possible that \emph{no single link function} is contained in the envelopes $\Psi^{\vec{f}}$ for all $\vec{f} \in \vec{\F_k}$, i.e., is simultaneously calibrated for $L^{\vec{f}}$ for \emph{all} such $\vec{f}$.
If no simultaneous link exists, the construction and analysis has to be tailored carefully to each $\vec{f}\in\vec{\F_k}$.
Interestingly, we show that such a simultaneous link does exist.

To find a link which is simultaneously calibrated for all $\vec{f}$, we identify certain structure which is common to all Lov\'asz hinges $L^{\vec{f}}$.
We encode this structure in a common link envelope $\hat \Psi$, and then show in Proposition~\ref{prop:psi-containment} that, for all $\vec{f} \in \vec{\F_k}$ and $u \in \reals^k$, we have $\hat \Psi(u) \subseteq \Psi^{\vec{f}}(u)$.
We then show that $\hat \Psi$ is nonempty for sufficiently small $\epsilon$, meaning it contains a link option pointwise.
This link is therefore contained in link envelope $\Psi^{\vec{f}}$ for all $\vec{f}$, which implies it is calibrated with respect to $\ellabsvecf$ for all $\vec{f} \in \vec{\F_k}$ simultaneously.

\subsection{The common link envelope $\hat \Psi$}\label{sec:calibrated-via-separation}

\begin{figure}[!]
	\centering
	\includegraphics[width=0.5\textwidth]{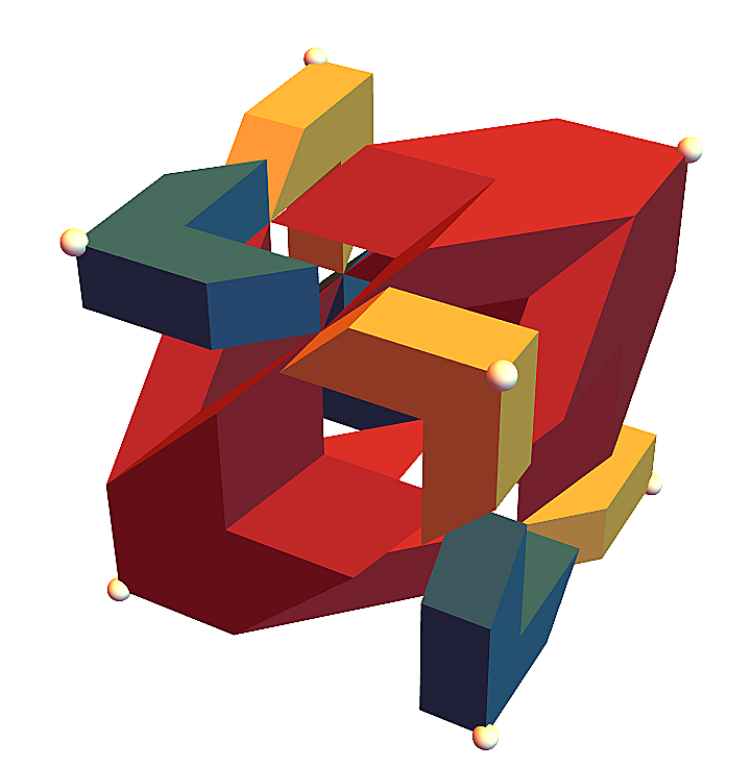}
  \caption{$\hat \Psi(u)$ for $u \in \reals_+^3$ and $\epsilon = \frac{1}{6}$. Each colored region connected to a particular node corresponds to a $v \in \{0,1 \}^3 \subseteq \V$ and at a point $u$, a calibrated link must link to one of the $v$ in the region.}
  		\label{fig:Psi-hat}
\end{figure}


We now present our link envelope $\hat\Psi$, used to construct calibrated links (Figure~\ref{fig:k2-links}, top left).

\begin{definition}\label{def:Psi-candidates}
  Let $\Vfaces := \cup_{\pi \in \Sc_k, y \in \Y} 2^{V_{\pi,y}}$ be the subsets of $\V$ whose convex hulls are faces of some $P_{\pi,y}$ polytope.
	Define $\hat \Psi:\reals^k \to 2^\V$ by $\hat\Psi(u) = \cap \{ V \in \Vfaces \mid d_\infty (\conv (V), \boxed u) < \epsilon\}$.
\end{definition}

Now we show that $\hat \Psi \subseteq \Psi^{\vec{f}}$ pointwise.
The proof uses the fact that both $\hat \Psi(u)$ and $\Psi^{\vec{f}}(u)$ are constructed by the intersections of sets, and shows that the sets generating $\hat \Psi(u)$ are subsets of those generating $\Psi^{\vec{f}}(u)$ for all $\vec{f} \in \vec{\F_k}$.
In particular, every possible optimal set in the range of $\prop{L^{\vec{f}}}$ is a union of faces generated by convex hulls of elements of $\Vfaces$.

\ignore{\begin{propositionc}\label{prop:psi-containment}\end{propositionc}}
\begin{restatable}{propositionc}{psicontainment}\label{prop:psi-containment}
  For all $\vec{f} \in \vec{\F_k}$ and $u\in\reals^k$, we have $\hat\Psi(u) \subseteq \Psi^{\vec{f}}(u)$.
\end{restatable}

We now characterize the link envelope $\hat \Psi(u)$ in terms of the coordinates of $u$.
In particular, $\hat \Psi(u)$ consists of the embedding points $v\in\V$ that make up the intersection of the faces from $\Vfaces$ that are $\epsilon$ close to $u$.
We can express these embedding points in terms of the ordered elements of $|u|$.
In particular, such a point $v\in\V$ appears in the intersection exactly when the corresponding elements of $|u|$ are $2\epsilon$-far from each other, since otherwise we can find a face not containing $v$ which is $\epsilon$-close to $u$ (Proposition~\ref{prop:hat-psi-char}).

Therefore, $\Psi$ is always nonempty when $\epsilon$ is small enough to guarantee a gap of at least $2\epsilon$ in the ordered elements of $|u|$ (Lemma~\ref{lem:hat-Psi-nonempty}).

\begin{restatable}{propositionc}{hatpsichar}\label{prop:hat-psi-char}
  Let $u\in\reals^k$, and let $\pi\in\Sc_k$ order the elements of $|u|$ (descending), and define $|u_{\pi_0}| = 1+\epsilon$ and $|u_{\pi_{k+1}}| = -\epsilon$.
  Then we have
  \begin{equation}\label{eq:hat-psi-char}
    \hat\Psi(u) = \{\onespi{\pi}{i} \odot \signstar(u) \mid i\in\{0,1,\ldots,k\}, \; |u_{\pi_i}| \geq |u_{\pi_{i+1}}| + 2\epsilon \}
  \end{equation}
\end{restatable}
\begin{restatable}{lemmac}{hatPsinonempty}\label{lem:hat-Psi-nonempty}
  $\hat\Psi$ is nonempty pointwise
  if and only if $\epsilon \in (0, \tfrac 1 {2k}]$.
\end{restatable}

\subsection{Family of calibrated link functions from $\hat \Psi$}

We now present a family of link functions for which each link paired with $L^{\vec{f}}$ is calibrated for  $\ellabsvecf$ regardless of the underlying $\vec{f}\in\vec{\F_k}$ satisfying Condition \ref{cond:inconsistency}. 
A link is selected from the said family of link functions by choosing a threshold $\tau \in [0,1]$.
The frequency of abstaining with a link from our proposed family is a function of the value of the chosen $\tau$, as demonstrated in Figure~\ref{fig:k2-links}.

\begin{definition}[Threshold-abstain link]\label{def:threshold-abstain-link}
  Fix $\epsilon \in (0, 1/2k]$ and $\tau \in [0,1]$. 
  Given any $u\in\reals^k$, take $\pi$ ordering the elements of $|u|$. 
  Moreover, let $\Iue = \{i : |u_{\pi_i}| - |u_{\pi_{i+1}}| \geq 2\epsilon \}$ where we define $|u_{\pi_0}|=1+\epsilon$ and $|u_{\pi_{k+1}}|=-\epsilon$.
  Take $i_{\tau} \in \argmin_{i \in \Iue}| \tau - \frac{|u_{\pi_i}| +|u_{\pi_{i+1}}|}{2} |$, breaking ties arbitrarily. 
  Then we define the \emph{threshold-abstain link} by
  \begin{equation}
    \label{eq:threshold-abstain-link}
    \psithresh(u) = \ones_{\{\pi_1, \ldots, \pi_{i_{\tau}}\}} \odot \sign(\boxed{u})~.
  \end{equation}
\end{definition}

\begin{theorem}\label{thm:well-defined-cal}
  Let $\epsilon \in (0, 1/2k]$, and fix any $\vec{f}\in\vec{\F_k}$.
  Then
  $(L^{\vec{f}}, \psithresh)$ is well-defined and calibrated with respect to $\ellabsvecf$.
\end{theorem}
\vskip-4pt
\begin{proof}
  Lemma~\ref{lem:hat-Psi-nonempty} shows that the index $i_{\tau}$ in Definition~\ref{def:threshold-abstain-link} always exist when
  $\epsilon \in (0,\frac{1}{2k}]$,
  which shows that $\psithresh$ is well-defined.
  By construction, we have $\psithresh(u)\in \hat \Psi (u)$  for all $u\in\reals^k$.
	As Proposition~\ref{prop:psi-containment} states that $\hat \Psi \subseteq \Psi^{\vec{f}}$ pointwise, we then have $\psithresh \in \Psi^{\vec{f}}$ pointwise.
  Finally, Theorem~\ref{thm:embed-implies-consistent} states that any link function contained in $\Psi^{\vec{f}}$ pointwise is calibrated.
\end{proof}


\begin{figure}[t]
\centering
	\begin{minipage}[b]{0.3\textwidth}
	\centering
	\includegraphics[width=\textwidth]{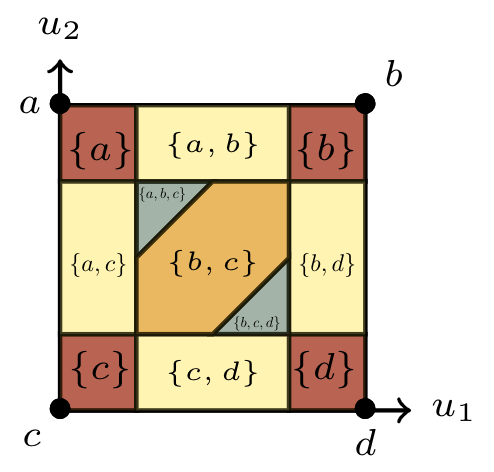}
\end{minipage}
\hfill
  \begin{minipage}[b]{0.3\textwidth}
    \includegraphics[width=\textwidth]{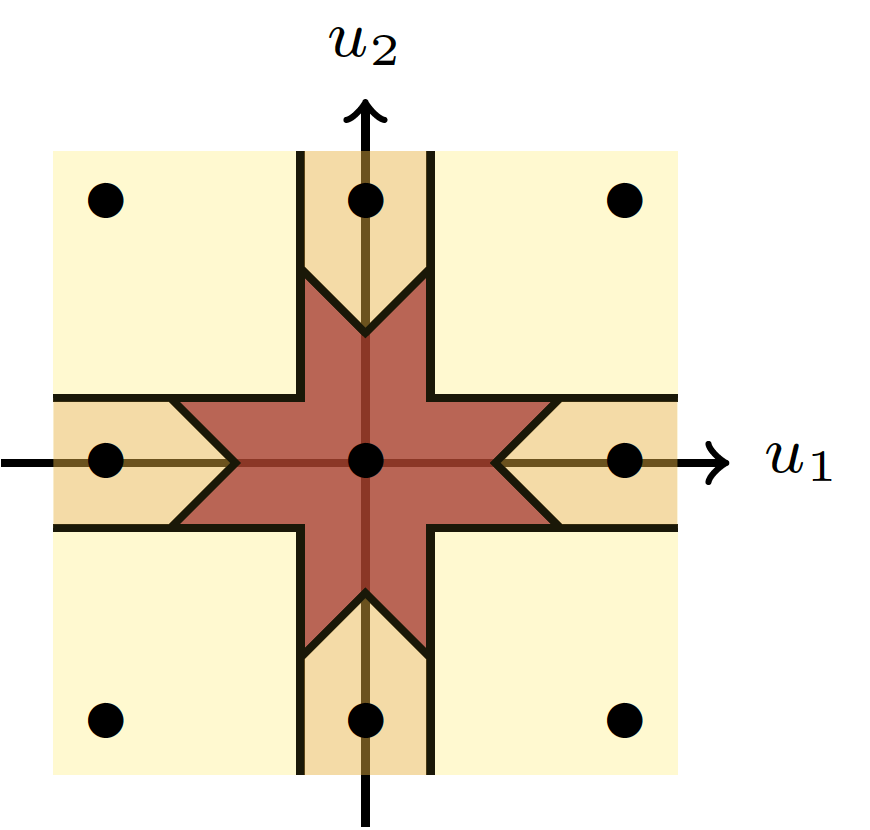}
  \end{minipage}
\hfill
  \begin{minipage}[b]{0.3\textwidth}
    \includegraphics[width=\textwidth]{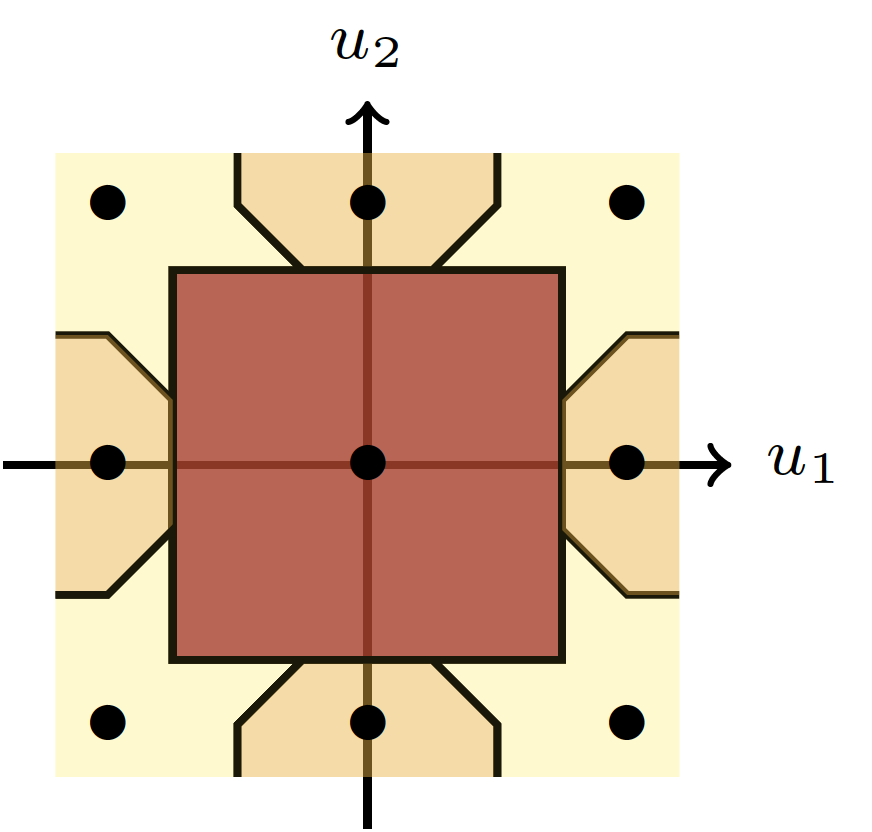}
    \end{minipage}
\vfill 
	\begin{minipage}[b]{0.3\textwidth}
	\centering
	\includegraphics[width=\textwidth]{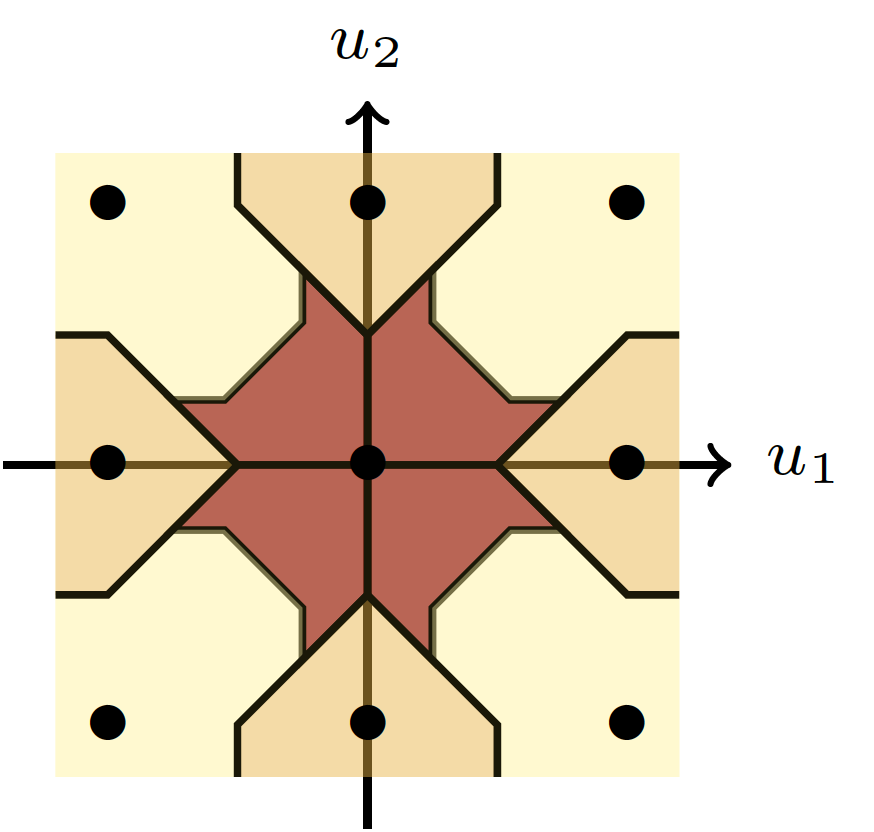}
\end{minipage}
\hfill
  \begin{minipage}[b]{0.3\textwidth}
    \includegraphics[width=\textwidth]{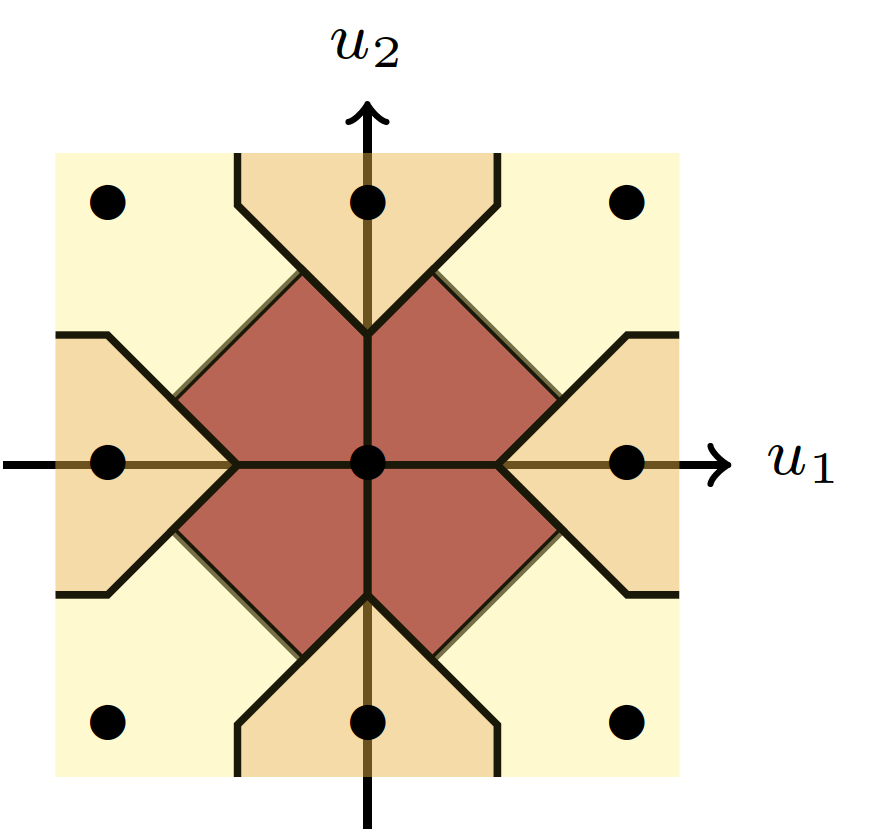}
  \end{minipage}
\hfill
  \begin{minipage}[b]{0.3\textwidth}
    \includegraphics[width=\textwidth]{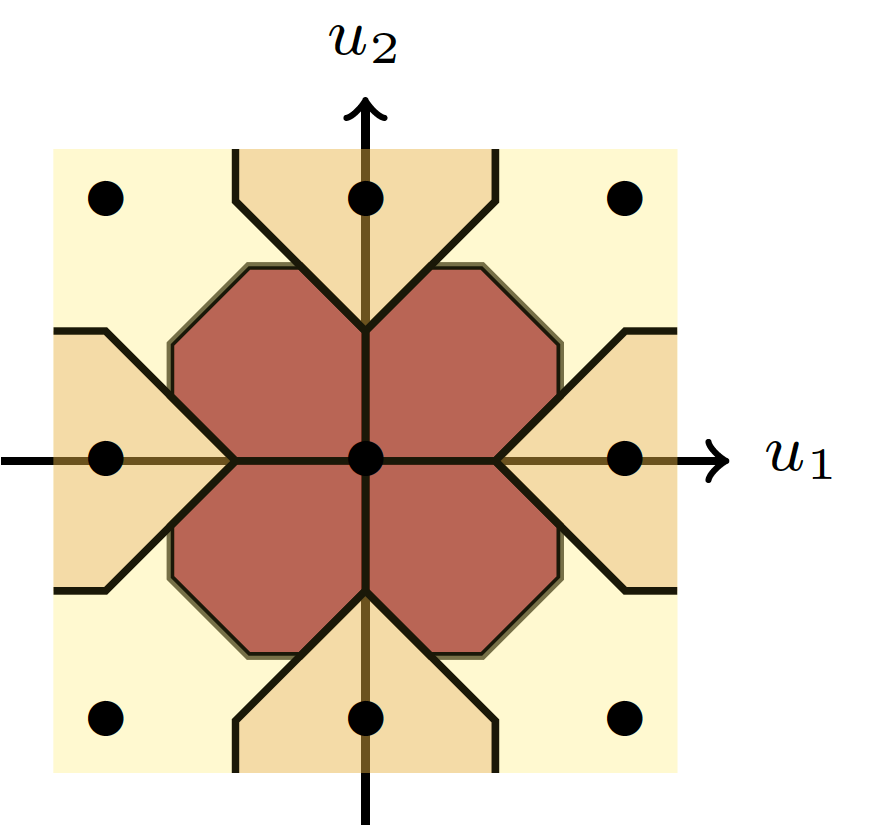}
  \end{minipage}
  \caption{The link envelope $\hat\Psi$ (top left) constructed with respect to $\|\cdot\|_\infty$ and link functions $\psithresh$ where $\tau =0$ (top middle), $\tau =1$ (top right), $\tau =.45$ (bottom left), $\tau =.5$ (bottom middle), $\tau =.55$ (bottom right) for $k=2$ and $\epsilon=\frac{1}{4}$.
    The envelope $\hat \Psi$ is pictured for $u \in \reals_+^2$, with each region labeled by the value of $\hat\Psi$; a link is calibrated if it always links to one of the nodes in the region.
    The values for the link functions $\psithresh$ are given by the unique point $v \in \V$ that each depicted region contains.
    In particular, the link $\psithresh$ for all $\tau \in [0,1]$ satisfy the constraints from $\hat\Psi$ (top left) and thus are calibrated.}
\label{fig:k2-links}
\end{figure}

\section{Tightness of the embedding}
\label{sec:tightness}

Recall that an embedding is called \emph{tight} if its representative set $\Sc$ has the minimum cardinality of any representative set~\citep{finocchiaro2022embedding}.
Tightness is significant because it implies that one has captured the minimal target loss consistent for a given surrogate, having removed all dominated predictions.
Indeed, the minimal such target loss is unique up to relabeling predictions.

In this section, we will see that some predictions of $\ellabs$ are actually dominated: namely, those which abstain on exactly one coordinate.
After removing these predictions, however, we show that the embedding is tight when the $f$ is strictly submodular and strictly increasing.
It appears from the proof that both of these assumptions on $f$ are required, though we leave open whether that is indeed the case.

\begin{theorem}\label{thm:lovasz-embed-tight}
  Let $f\in\F_k$ be strictly submodular and strictly increasing.
  Let $\V_0 = \{ v \in \V \mid \|v\|_0 \neq k-1 \}$, i.e., the set of points in $\V$ that do not have only one zero coordinate.
  Then $L^f$ tightly embeds $\ellabs|_{\V_0}$.
\end{theorem}

The proof is split into the following two lemmas.
The first shows that $\V_0$ is indeed representative.
The second shows that $\V_0$ is a subset of any representative set, and is therefore a minimal representative set.
The result then follows from the fact that $L^f$ embeds $\ellabs$.

\begin{lemma}
  \label{lem:cant-abstain-on-just-one}
  The set $\V_0$ is representative for $\ellabs$ for any $f\in\F_k$.
\end{lemma}
\begin{proof}
  Since $\V$ is representative by Theorem~\ref{thm:lovasz-embeds}, it suffices to show that any $v\in\V\setminus\V_0$ is dominated by some prediction in $\V_0$.
  Let such a $v\in\V\setminus\V_0$ be given.
  Then we have $v_i = 0$ for some $i \in [k]$ and $v_j \neq 0$ for all $j\neq i$.
  Let $y^+,y^-\in\Y$ be given by $y^+_j=y^-_j=v_j$ for $j\neq i$ and $y^+_i = 1$, $y^-_i = -1$, i.e., choosing a sign for entry $i$ of $v$.
  We will show that for all $p\in\simplex$, we have
  $v\in \prop{\ellabs}(p) \implies \{y^+,y^-\}\subseteq\prop{\ellabs}(p)$.
  As $y^+,y^-\in\V_0$, this statement implies that $v$ is weakly dominated by some prediction in $\V_0$.

  Let $p\in\simplex$, and suppose $v\in \prop{\ellabs}(p)$.
  Let $\Pr_p[Y_i=1]$ be the probability that the \nth{$i$} coordinate is $1$ under $p$, i.e., $\Pr_p[Y_i=1] = \sum_{y_{-i}\in\{-1,1\}^{k-1}} p_y$ where $y_i=1$, and similarly $\Pr_p[Y_i=-1]$.
  Let $v'$ be equal to $y^+$ if $\Pr_p[Y_i=1] \geq 1/2$ and $y^-$ otherwise.
  Then we have
  \begin{align*}
    \ellabs(v;p)
    &= \sum_{y\in\Y} p_y \Bigl( f(\mis(v,y)\setminus \abs(v)) + f(\mis(v,y))\Bigr)
    \\
    &= \sum_{y\in\Y} p_y \Bigl( f(\mis(v_{-i},y_{-i})) + f(\mis(v_{-i},y_{-i})\cup\{i\}) \Bigr)
    \\
    &= \sum_{y_{-i}\in\{-1,1\}^{k-1}} (p_{y_{-i},1} + p_{y_{-i,-1}}) \Bigl( f(\mis(v_{-i},y_{-i})) + f(\mis(v_{-i},y_{-i})\cup\{i\}) \Bigr)~.
  \end{align*}

  For any $\hat y\in\{y^+,y^-\}$, we have 
  \begin{align*}
    \ellabs(\hat y;p)
    &= \sum_{y\in\Y} p_y \Bigl( f(\mis(\hat y,y)\setminus \abs(\hat y)) + f(\mis(\hat y,y))\Bigr)
    \\
    &= \sum_{y_{-i}\in\{-1,1\}^{k-1}} \Bigl( p_{y_{-i},\hat y_i} 2 f(\mis(v_{-i},y_{-i})) + p_{y_{-i},-\hat y_i} 2 f(\mis(v_{-i},y_{-i})\cup\{i\}) \Bigr)~.
  \end{align*}

  The average expected loss $\alpha = \tfrac 1 2(\ellabs(y^+;p) + \ellabs(y^-;p))$ of $y^+$ and $y^-$ is thus
  \begin{align*}
    \alpha
    &= \sum_{y_{-i}\in\{-1,1\}^{k-1}} \Bigl( p_{y_{-i},1} f(\mis(v_{-i},y_{-i})) + p_{y_{-i},-1}  f(\mis(v_{-i},y_{-i})\cup\{i\}) \Bigr)
    \\
    &\quad + \sum_{y_{-i}\in\{-1,1\}^{k-1}} \Bigl( p_{y_{-i},-1} f(\mis(v_{-i},y_{-i})) + p_{y_{-i},1}  f(\mis(v_{-i},y_{-i})\cup\{i\}) \Bigr)
    \\
    &=\ellabs(v;p)~.
  \end{align*}
  By optimality of $v$ for $p$, we also have $\ellabs(v;p) \leq \ellabs(y^+;p),\ellabs(y^-;p)$.
  We conclude $\ellabs(v;p) = \ellabs(y^+;p) = \ellabs(y^-;p)$, which gives $\{y^+,y^-\}\subseteq\prop{\ellabs}(p)$ as desired.
\end{proof}

We next show that $\V_0$ is contained in any representative set.
As Lemma~\ref{lem:cant-abstain-on-just-one} showed $\V_0$ is representative, we conclude it must be minimal, giving Theorem~\ref{thm:lovasz-embed-tight}.
\begin{lemma}
  Let $f\in\F_k$ be strictly submodular and strictly increasing, and let $\Sc\subseteq\V$ be a representative set for $\ellabs$.
  Then $\V_0\subseteq\Sc$.
\end{lemma}
\begin{proof}
  Let $v\in\V_0$.
  If $v\in\Y$, we have $\prop{\ellabs}(\delta_v) = \{v\}$ by the fact that $\{\emptyset\} = \argmin_{S\subseteq[k]} f(S)$.
  Therefore, we can restrict our attention to $v\in\V_0\setminus\Y$, which in particular has the set $A_v = \{i : v_i = 0\}$ with at least two elements.
  We will show that $v$ is the unique $\ellabs$-optimal prediction for the distribution $p$ given by
  \begin{equation}
    \label{eq:1}
    p_y =
    \begin{cases}
      2^{-|A_v|} & y \odot v \geq 0 \\
      0 & \text{otherwise}
    \end{cases}~.
  \end{equation}
  In other words, $p$ chooses uniformly random signs within $A_v$, but otherwise agrees with $v$.
  Thus $v$ must be present in any representative set.

  For the remainder of the proof, let $S = \{i : v_i=1\}$ and $A = A_v = \{i : v_i=0\}$, and identify $v$ with this pair of sets.
  The expected loss for $v \equiv (S,A)$ under $p$ is therefore
  \begin{align*}
    \ellabs((S,A);p)
    &= \sum_{T\subseteq A} 2^{-|A|} f(S\triangle (S\cup T)\setminus A) + f(S\triangle (S\cup T)\cup A)
    \\
    &= 2^{-|A|} \sum_{T\subseteq A} f(T\setminus A) + f(T\cup A)
    \\
    &= f(\emptyset) + f(A) = f(A)~,
  \end{align*}
  where $T\subseteq A$ represents the positive signs.

  Now suppose $S',A'\subseteq [k]$ are disjoint and represent some $v' \neq v$, so $(S',A') \neq (S,A)$.
  Once again we write the expected loss under $p$.
  We will use the observation that, as $S$ and $T$ are disjoint, $S\cup T = S\triangle T$.
  Defining $\hat S = S' \triangle S$, and using the associativity of $\triangle$, we then have $S' \triangle (S\cup T) = \hat S \triangle T$.
  \begin{align*}
    \ellabs((S',A');p)
    &= 2^{-|A|} \sum_{T\subseteq A} f(S'\triangle (S\cup T)\setminus A') + f(S'\triangle (S\cup T)\cup A')
    \\
    &= 2^{-|A|} \sum_{T\subseteq A} f(\hat S\triangle T\setminus A') + f(\hat S\triangle T\cup A')
      \numberthis\label{eq:tightness-loss-rewrite}
  \end{align*}

  We first suppose $A \subseteq A'$.
  Then $\hat S\triangle T\setminus A' = \hat S \setminus A'$ and $\hat S\triangle T\cup A' = \hat S\cup A'$ for any $T \subseteq A$, giving the simpler form $\ellabs((S',A');p) = f(\hat S\setminus A') + f(\hat S\cup A')$.
  Consider two cases (i) $\hat S\setminus A' \neq \emptyset$, and (ii) $\hat S\cup A' \supsetneq A$.
  We will show either (i) or (ii) hold, and then apply the fact that $f$ is strictly increasing.
  If (ii) fails, since $A \subseteq A'$, we must have $A = A'$.
  Then $A = A'$ is disjoint from both $S$ and $S'$, and therefore from $\hat S = S'\triangle S$.
  We therefore have $\hat S\setminus A' \neq \emptyset \iff \hat S \neq \emptyset \iff S \neq S'$, which must hold as $A = A'$ yet $(S,A) \neq (S',A')$ by assumption.
  Thus, at least (i) or (ii) hold.
  As $f$ is strictly increasing, we conclude $\ellabs((S',A');p) > f(A) = f(\emptyset) + f(A) = \ellabs((S,A);p)$.

  Now suppose $A \not\subseteq A'$.
  In this case, we will apply strict submodularity.
  Let us rewrite the expected loss,
  \begin{align*}
    \ellabs((S',A');p)
    &= 2^{-|A|} \sum_{T\subseteq A} f(\hat S\triangle T\setminus A') + f(\hat S\triangle T\cup A')
    \\
    &= \frac 1 2 \; 2^{-|A|} \sum_{T\subseteq A} \biggl( f(\hat S\triangle T\setminus A') + f(\hat S\triangle T\cup A')
    \\[-12pt]
    & \hspace*{85pt}
      + f(\hat S\triangle (A\setminus T)\setminus A') + f(\hat S\triangle (A\setminus T)\cup A') \biggr)
      \numberthis\label{eq:tightness-double-counting}
    \\
    &\geq \frac 1 2 \; 2^{-|A|} \sum_{T\subseteq A} \biggl(
      f\Bigl((\hat S\triangle T\setminus A') \cup (\hat S\triangle (A\setminus T)\cup A')\Bigr)
    \\[-10pt]
    & \hspace*{85pt}
      + f\Bigl((\hat S\triangle T\setminus A') \cap (\hat S\triangle (A\setminus T)\cup A')\Bigr)
    \\[2pt]
    & \hspace*{85pt}
      + f\Bigl((\hat S\triangle (A\setminus T)\setminus A') \cup (\hat S\triangle T\cup A')\Bigr)
    \\[-2pt]
    & \hspace*{85pt}
      + f\Bigl((\hat S\triangle (A\setminus T)\setminus A') \cap (\hat S\triangle T\cup A')\Bigr)
      \biggr)
      \numberthis\label{eq:tightness-apply-submodularity}
    \\
    &= \frac 1 2 \; 2^{-|A|} \sum_{T\subseteq A} \biggl(
      f\bigl(\hat S\cup A\cup A'\bigr) + f\bigl(\hat S \setminus (A\cup A')\bigr)
    \\[-10pt]
    & \hspace*{85pt}
      f\bigl(\hat S\cup A\cup A'\bigr) + f\bigl(\hat S \setminus (A\cup A')\bigr)
      \biggr)      ~,
      \numberthis\label{eq:tightness-simplify}
    \\
    &= f\bigl(\hat S\cup A\cup A'\bigr) + f\bigl(\hat S \setminus (A\cup A')\bigr)
    \\
    &\geq f(A) + f(\emptyset)
      ~,
      \numberthis\label{eq:tightness-weak-increasing}
  \end{align*}
  where in eq.~\eqref{eq:tightness-apply-submodularity} we applied submodularity twice, to the first and fourth terms and to the second and third terms, and in eq.~\eqref{eq:tightness-weak-increasing} we used the fact that $f$ is increasing.
  We will show that for at least one $T\subseteq A$, the first and fourth terms in eq.~\eqref{eq:tightness-apply-submodularity} correspond to incomparable sets, meaning the inequality is actually tight.
  
  As $|A|\neq 1$, we must have $|A|\geq 2$.
  Let $i,j\in A$ be distinct elements such that $i\notin A'$.
  We therefore have two cases for $i$, and three for $j$, and for each we claim there exists $T\subseteq A$ such that $\hat S\triangle T\setminus A'$ and $\hat S\triangle (A\setminus T)\cup A'$ are incomparable:
  \begin{center}
    \begin{tabular}{R|CC}
      & i\in S' & i\notin S'
      \\
      \hline
      j\in S' & T=\{j\} & T=\{i,j\}
      \\
      j\in A' & T=B\setminus\{i\} & T=\{i\}
      \\
      j\notin S'\cup A' & T=\emptyset & T=\{i\}
    \end{tabular}
  \end{center}
  In particular, these choices give $i\in\hat S\triangle T$ and $i\notin\hat S\triangle (A\setminus T)$:
  in the first column, $i\in S'\setminus S \implies i\in \hat S$, so we ensure $i\notin T$, whereas we ensure $i\in T$ in the second column.
  Similarly, these choices give $j\notin\hat S\triangle T\setminus A'$ and $j\in\hat S\triangle (A\setminus T)\cup A'$:
  in the first row, $j\in \hat S$, so we ensure $j\in T$;
  in the second row, $j\in A'$, so we automatically have the condition;
  in the third row, $j\notin \hat S$, so we ensure $j\notin T$.
\end{proof}

Combining Theorem~\ref{thm:lovasz-embed-tight} with Lemma~\ref{lemma:safe-link}, one could modify any link function $\psi$ consistent with $\ellabs$ so that predictions $u$ linking to $v = \psi(u)$ with exactly one zero $v_i=0$ are instead linked to one of the two elements of $\Y$ that agree with $v$ outside coordinate $i$.
A natural choice is to take the sign of $u_i$, which means using the sign link on $u$ when the link would have only abstained on one coordinate.

We suspect that the conditions on $f$ in Theorem~\ref{thm:lovasz-embed-tight} are necessary.
The proof does shed light on other cases, however.
For example, in light of Lemma~\ref{lem:cant-abstain-on-just-one}, if $f$ is modular on some sublattice $S,S\cup\{i\},S\cup\{j\},S\cup\{i,j\}$, then abstaining on $i$ or $j$ should again be dominated by simply choosing the most likely label.
As above, one could modify the link to avoid abstaining on this sublattice.


\section{Experiments}\label{sec:experiments}

Until now, this work has analyzed the Lov\'asz hinge strictly theoretically. 
Our analysis has focused on identifying the underlying target problem embedded by the Lov\'asz hinge: structured prediction with the option to abstain.
Although abstaining is necessary to construct a calibrated surrogate, it is unclear how often one must abstain in practice. 
Motivated by the original works of \citet{yu2018lovasz,berman2018lovasz}, which proposed using Jaccard loss with the Lov\'asz hinge for image segmentation, we perform two experiments, binary foreground-background segmentation and multiclass image segmentation using the Jaccard loss with the Lov\'asz hinge.
Through these experiments, we seek to demonstrate the prevalence and importance of abstaining via our derived calibrated links.
Code for the experiments can be found on github\footnote{Github: \url{https://github.com/EnriqueNueve/LovaszAbstainExperiments}}

\subsection{Metrics}\label{metrics_section}
To measure the influence of abstaining, we propose a variety of performance and rejection (abstain) metrics inspired by the work of \citet{condessa2017performance}.
These metrics are commonly used in the image segmentation literature but have been adjusted to account for abstention on a subset of the predictions.
We also use metrics from the original work of \citet{berman2018lovasz} modified to account for abstaining, which evaluate performance by aggregating over all pixels in the test set, rather than an average of per-image performance. 
We do not condone the commonly used metrics in multiclass image segmentation, such as Mean IoU or the process of computing metrics over the entire data set compared to weighted averaging over samples, but use the same metrics to facilitate comparison with the current literature.
For a complete description of the metrics in this section, we refer the reader to \S~\ref{app:omitted-experiments}.

Let $\hat{S}=\{(v^{(j)},y^{(j)})\}_{j=1}^{N}$ express a set of predictions and matching labels where $v^{(j)}\in\V =\{-1,0,1\}^k$ and $y^{(j)}\in\Y = \{-1,1\}^k$ for all $j\in \{1,\dots ,N\}$.
Let $\abs(v) = \{ i \in [k] : v_i = 0\}$.
With respect to the elements that the link did not abstain on, we define terms such as True Positive, True Negative, False Positive, and False Negative --- for example, True Positive $TP(v,y)= \{i\in [k]:v_i=1 \wedge y_i=1 \}$.
Using the definitions above, we define metrics such as Accuracy, Recall, Precision, and IoU.
To measure the rate at which the link abstains, we introduce a metric called Rejection Rate: $R(\hat{S})=\sum_{j=1}^{N}\frac{|\abs (v^{(j)})|}{k}$.


\begin{figure}[t]
\centering
	\begin{minipage}[b]{0.32\textwidth}
	\centering
	\includegraphics[width=\textwidth]{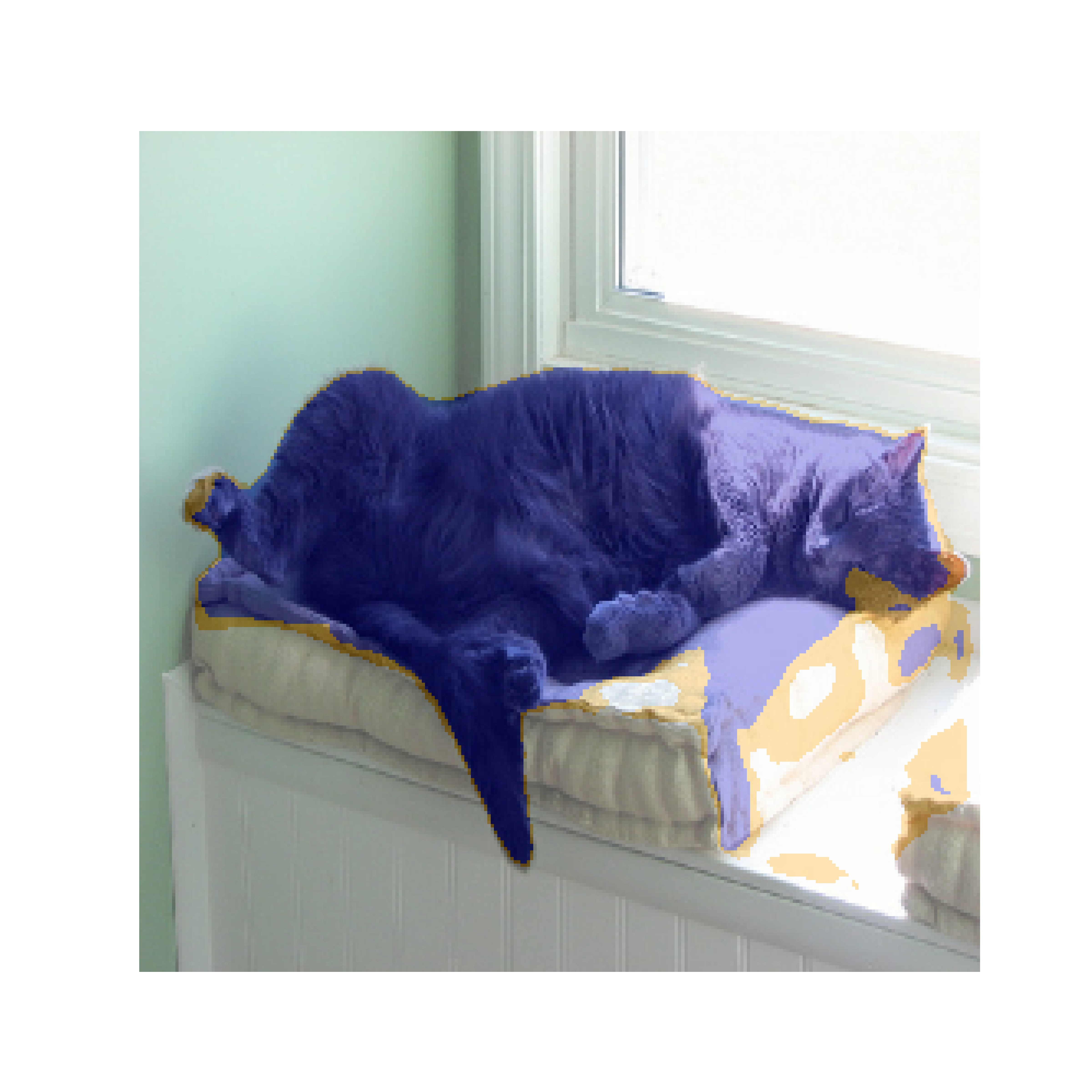}
	    \caption*{$\tau = 1$}
\end{minipage}
\hfill
  \begin{minipage}[b]{0.32\textwidth}
    \includegraphics[width=\textwidth]{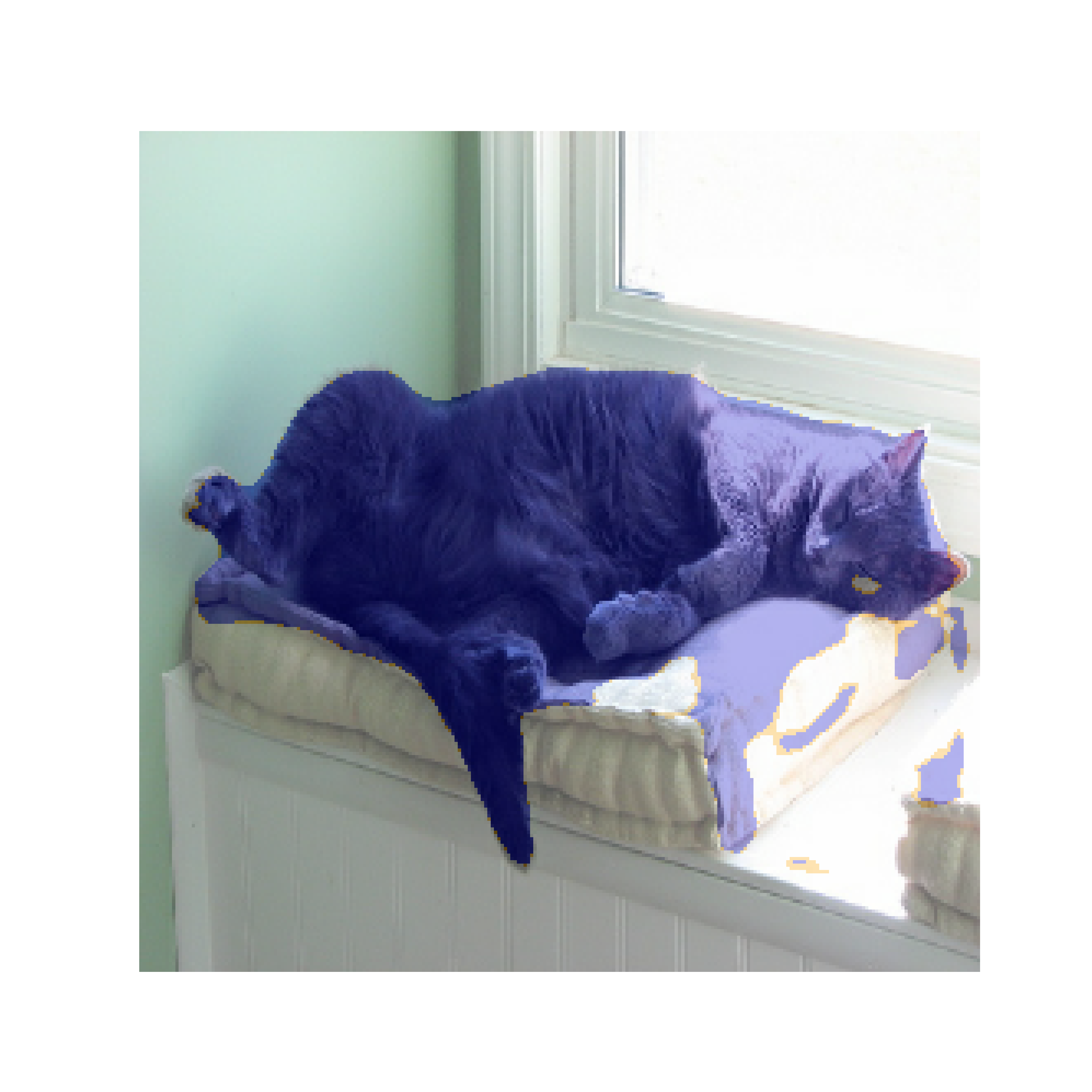}
    \caption*{$\tau = .5$}
  \end{minipage}
\hfill
  \begin{minipage}[b]{0.32\textwidth}
    \includegraphics[width=\textwidth]{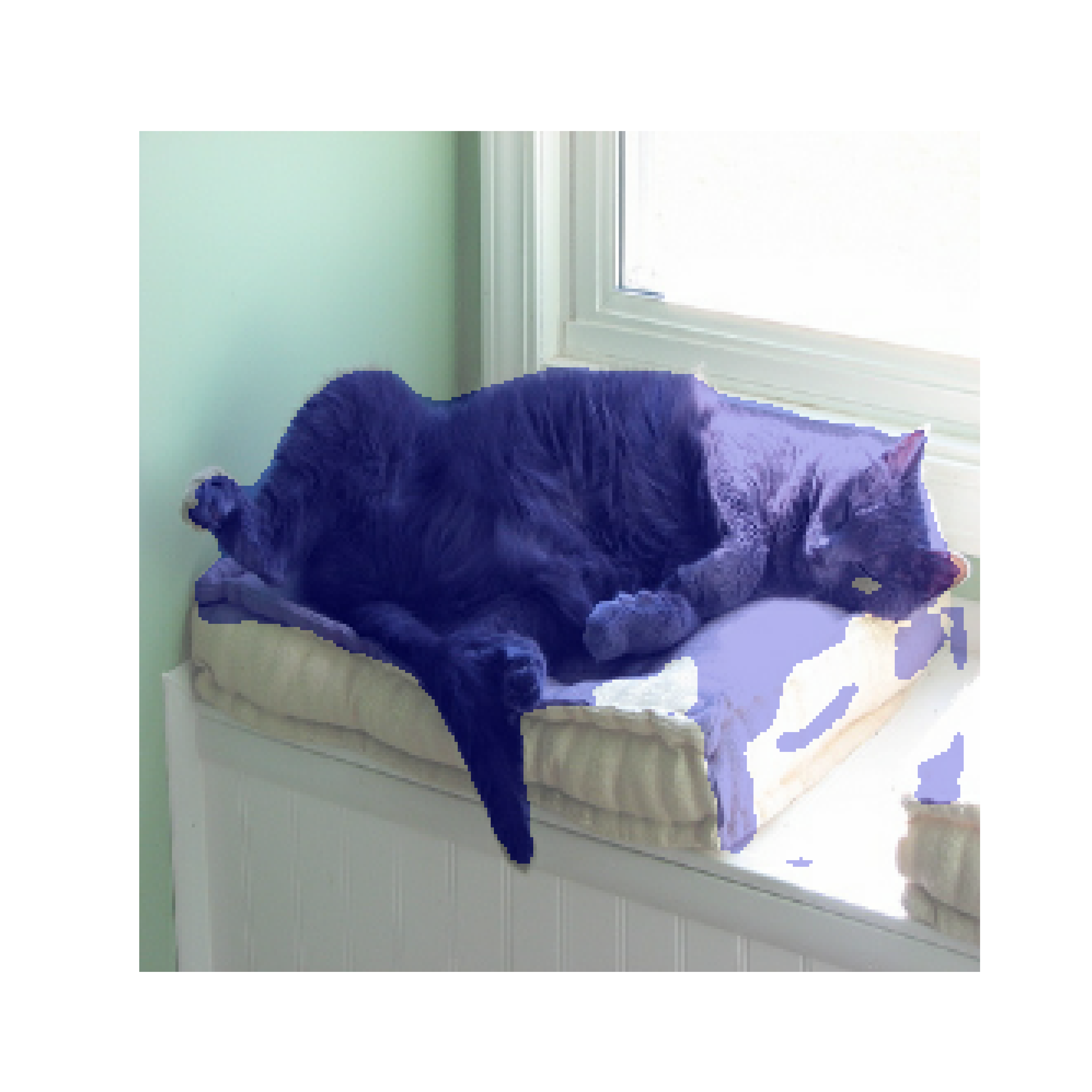}
    \caption*{$\tau = 0$}
    \end{minipage}
  \caption{Example prediction with varying values of $\tau$ on an image from Pascal VOC 2012 \citep{everingham2015pascal} where the foreground is the pixels of the cat.  Blue denotes predicted foreground, and orange denotes predicted abstain regions: as $\tau$ decreases, the abstain (orange) regions decrease. }
\label{fig:birds}
\end{figure}

\subsection{Binary image segmentation with Jaccard loss}\label{subsec:binary-image-seg}

For our first experiment, we performed foreground-background semantic segmentation with the Lov\'asz hinge as a surrogate loss function defined with respect to the Jaccard loss over the Pascal VOC 2012 \citep{everingham2015pascal} dataset along with the augmented images from \citet{hariharan2011semantic}. 
 Our choice of model was a DeeplabV3+ \citep{chen2018encoder} model, with a ResNet50 \citep{he2016deep} backbone network pre-trained on ImageNet. 
 Every class that was not background was assigned as foreground. 
 The non-labeled pixels that serve as trim from the Pascal VOC 2012 dataset \citep{everingham2015pascal} were assigned as background. 
 Furthermore, to eliminate external factors to our experiments, all input images and corresponding masks were resized to $256\times 256\times3$ and $256\times 256\times1$ respectively, using nearest-neighbor interpolation, and no forms of data augmentation were performed beyond what was present from  \citet{hariharan2011semantic}. 
 We fix $\epsilon = 1/2k$ for all experiments where $k= 256\times 256$.
 For training details, see \S~\ref{app:model-training}.

\begin{table}[!t]
\centering
\begin{tabular}{@{}lccccc@{}}
\toprule
$\tau$              & 0         & .25   & .5    & .75   & 1     \\ \midrule
Accuracy $\uparrow$            			& .8791     	& .8807 	& .8825 	& .8850 	& .8985 \\
Recall $\uparrow$                 			& .7191     	& .7221 	& .7255 	& .7303 	& .7562 \\
Precision $\uparrow$            			& .8546     	& .8563 	& .8583 	& .8611 	& .8765 \\
IoU  $\uparrow$                     			& .6407     	& .6440 	& .6479  	& .6534 	& .6833 \\
Rejection Rate      			& 2.6326e-7  & .0042 	& .0089 	& .0159 	& .0558 \\
Rejection Rate Pos. 		& .3600     	& .2929 	& .2933 	& .2927 	& .2915 \\
Rejection Rate Neg. 		& .6400     	& .7071 	& .7067 	& .7073 	& .7085 \\ \bottomrule
\end{tabular}
\caption{This table contains the binary experiment metric data. Observe that as $\tau$ increases, the rejection rate increases, and the metrics improve monotonically. This demonstrates the inherent tradeoff between performance and rejection rate. }\label{tab:binarymetrics}
\end{table}

Table \ref{tab:binarymetrics} displays the recorded metrics from this experiment for various values of $\tau$.
Metrics such as Accuracy and IoU are the highest when $\tau =1$ (the highest value of $\tau$), which is expected since that the Rejection Rate is also at its highest, while inversely, Accuracy and IoU are the lowest when $\tau =0$ (the lowest value of $\tau$).
The positive correlation between Accuracy and Rejection Rate suggests that the model abstains from making predictions that it would have gotten wrong if otherwise predicted. 
It is worth noting that when $\tau = 0$, the rejection rate is nearly negligible; however, it is also noteworthy to point out that the lowest performance metrics occur when $\tau = 0$.




\section{Multiclass structured selective prediction with the Lov\'asz hinge}\label{sec:multiclass}

While our study has been focused on binary structured prediction problems until this point, many prevalent applications of structured prediction occur in the multiclass setting.
It is not immediately clear how to extend the Lov\'asz hinge from binary to multiclass structured prediction problems.
In the context of image segmentation, \citet{berman2018lovasz} proposed generalizing the Lov\'asz hinge for multiclass semantic segmentation over a set of classes $\C$ by applying binary segmentation (using Jaccard loss) in a one-vs-all configuration.
In this approach, for each $c \in \C$, they apply binary Jaccard loss with class $c$ representing the foreground (positive class), and all other classes representing the background (negative class). 
Their target submodular function is then the average of the one-vs-all Jaccard losses,
\begin{align}
  \Jac^{\C}(y',y)
  &= \frac{1}{| \C |} \sum_{c\in \C} \Jac(\chi_{\{i : y'_i = c\}},\chi_{\{i : y_i = c\}})
    \nonumber
  \\
   \label{eq:jaccard-ova}
  &= \frac{1}{| \C |} \sum_{c\in \C}  \frac{ | \mis_c (y',y)  |}{| \{ i\in [k] \mid y_i=c  \}\cup \mis_c (y',y)   |}~,
\end{align}
where $\chi_{S}$ is the $\{-1,1\}$ indicator vector for $S\subseteq[k]$, and $\mis_c (y',y) =\{i \in [k] \mid ( y'_i\neq c \land  y_i = c )\lor (y'_i = c \land  y_i \neq c) \}$.
Since the proposed approach of \citet{berman2018lovasz} for multiclass structured prediction does not account for abstaining, when $|\C|=2$, we know this approach is inconsistent. 
Thus, the question of how the Lov\'asz hinge can be used as a consistent surrogate loss for some multiclass structured prediction target loss remains open.

This section introduces two different surrogate--target approaches for the task of multiclass structured prediction, both of which reduce to the (binary-label) Lov\'asz hinge.
The first target problem, called \emph{one-hot one-vs-all structured prediction with abstention}, employs a one-hot encoding to map classes to the binary problem.
Specifically, we map labels $\C$ to the subeset of $\{-1,1\}^{|\C|}$ of vectors with exactly one $1$.
The corresponding target problem is simply $\ellabsvecf$ with prediction set $\{-1,0,1\}^{|\C |k}$.
The Lov\'asz hinge on $\reals^{|\C|k}$ is consistent for this target, for essentially any polymatroid, including Jaccard.
A major drawback, however, is that it allows target predictions such as $(1,1,0,-1)$ for $|\C |=4$, $k=1$, which neither map to a single class nor to abstain.
It is open whether one can remove such nonsensical target predictions, or a sensible way to interpret them.

The second target problem, called \emph{multiclass structured abstain}, uses the BEP encoding from \citeauthor{ramaswamy2018consistent} to map the classes to the lower-dimension space $\{-1,1\}^{\log|\C|}$.
In other words, this approach encodes each of the possible $|\C|$ classes and $k$ individual labels into corners of the $\{-1,1\}^{dk}$ hypercube where $d = \log|\C|$.
By reducing to the binary case, we find that the Lov\'asz hinge on $\reals^{dk}$ is consistent with the usual abstain target on $\V = \{-1,0,1\}^{dk}$.
Unlike the one-hot encoding approach, assuming $|\C|=2^d$ for simplicity, we now can map every element of $\{-1,1\}^d$ to a unique class in $\C$.
If we interpret all-zeros as an abstention, the only potentially nonsensical individual predictions are those that abstain (are zero) on some bits but not others, within the same block of $d$ bits.
Our main technical contribution in this section, therefore, is showing that abstaining on all bits in a block dominates such nonsensical predictions.
We can thus restrict the target prediction set to $\{-1,1\}^d \cup \{0\}^d$, or in other words, $\C \cup \{\text{abstain}\}$, a more familiar setup for multiclass abstention.
Leveraging this result, the Lov\'asz hinge and a trimmed version of the threshold-abstain link is a calibrated surrogate-link pair for the one-vs-all multiclass structured abstain problem.
Unfortunately, the class of polymatroids for which the calibration result holds does not capture one-vs-all Jaccard loss. 
Hence, whether the Lov\'asz hinge is consistent for a sensible variant of one-vs-all Jaccard remains open.

\subsection{Multiclass notation}

Let $[k]=\{1,\ldots,k \}\subset \mathbb{N}$ index a finite set of structured individual predictions (i.e., one for each pixel in an image) and  $[C] = \{1,\dots ,C \} \subset \mathbb{N}$ index a finite set of labels (classes) for each of these individual predictions.
Then $[C]^{k}$ is the set of all possible class labels over joint labels.
Let $d = \lceil \log_2 C \rceil$ ($d=C$); we will later embed  $[C]$ into $\{-1,1\}^d$.
In this section, we will overload notation and write $y\in[C]^k$ as well as $y\in\{-1,1\}^{dk}$ to denote labels in the multiclass case either as there true values or in a binary (one-hot) encoded form.  
For a given function $f:\reals^d\to\reals$, we denote $f\square :\reals^{dk}\to \reals^k$ as the function where $f$ is applied to blocks of $u\in\reals^{dk}$ expressed by $u_{(i)}:=\{u_{d(i-1)+1},\dots ,u_{d(i-1)+d}\}$ for all $i\in[k]$.
In general, we index individual predictions in $[k]$ with $i$ and bit predictions in $[dk]$ with $j$.
For a bit prediction $j\in [dk]$, we let $\block j$ denote the unique individual prediction $i$ containing $j$, i.e., such that $j\in \{d(i-1)+1,\dots ,d(i-1)+d   \}$.
We let $S_{(i)}:= S\cap \{d(i-1)+1,\dots ,d(i-1)+d   \}$ where $S\subseteq 2^{[dk]}$.

\subsection{One-hot one-vs-all structured abstain problem}\label{subsec:one-vs-all-one-hot}

Consider the one-vs-all target problem where one defines a submodular function on \emph{class mispredictions} (e.g., ``dog'' vs ``cat'').
Specifically, generalizing one-vs-all Jaccard~\eqref{eq:jaccard-ova}, let submodular functions $g_{c,y} : 2^{[k]}\to\reals_+$ be given for each $c\in[C]$ and $y\in[C]^k$, and consider the target loss
\begin{align}
 \ell^{g,C}(y',y)
 &= \frac{1}{C} \sum_{c\in \C} g_{c,y}(\mis_c (y',y))~,
\end{align}
where again $\mis_c (y',y) =\{i \in [k] \mid ( y'_i\neq c \land  y_i = c )\lor (y'_i = c \land  y_i \neq c) \}$.
To design a surrogate for $\ell^{g,C}$, we consider reducing to the binary case, using the Lov\'asz hinge as a surrogate loss.
A tempting approach, following \citeauthor{berman2018lovasz}, is to try \emph{one-hot encoding} the classes into $\reals^C$ by the function $ \phi : c \mapsto \chi_{c}$, which takes value $1$ on the $c^{th}$ element of the vector, and $-1$ everywhere else.
Thus, the possible labels $\Y := [C]^k$ would map to $\phi \square (\Y) \subseteq \{-1,1\}^{Ck}$ as a one-hot encoding of class labels, concatenated into one vector.

Ideally, we could derive an equivalence from $g_{c,y}$, which is submodular on classes, to some $f_{c,y}$, which is submodular on bits.
In particular, this equivalence would need to map $\mis_c$ at the individual prediction level to $\mis$ at the bit level.
As we show next, this much is possible.
Yet since the Lov\'asz hinge is only consistent for a version of the target with abstention, we will then need to contend with many flavors of abstention.

To develop intuition, consider $k=2$ predictions of $C=3$ classes corresponding to labels $\{$dog, cat, bird$\}$.
\begin{align*}
y' &= (3,3) \equiv (\text{bird, bird}) & y &= (2,3) \equiv (\text{cat, bird}) \\
  \phi\square (y') &= (-1, \underset{\parbox{.6cm}{\hspace*{-1cm}\text{{\tiny first individual prediction is not ``cat''}}}}{\underbrace{-1}}, 1, -1,-1,1) & \phi\square (y) &= (-1,\underset{\parbox{.8cm}{\hspace*{-0.6cm}\text{{\tiny first label is ``cat''}}}}{\underbrace{1}},-1,-1,-1,1) 
\end{align*}
Then for the class ``dog'', we are only concerned with individual mispredictions involving the class ``dog'', of which there are none, and bit mispredictions regarding the first and fourth bits, of which are are none.
Similarly, for the class ``cat,'' there is an error on the first individual prediction, which is reflected in a bit-error on the second bit.
We can see the mispredictions in both representations align nicely: individual mispredictions correspond to two bit mispredictions, for the bits encoding the predicted and true classes.
In particular, the one-vs-all mispredictions can be read off of the bit misprediction set.
\begin{align*}
\mis_{dog}(y,y') &= \emptyset & \mis_{cat}(y,y') &= \{1\} & \mis_{bird}(y,y') &= \{1\} \\
\mis_{dog}(\phi\square (y),\phi\square (y')) &= \emptyset & \mis_{cat}(\phi\square (y),\phi\square (y')) &= \{2\} & \mis_{bird}(\phi\square (y),\phi\square (y')) &= \{3\} \\
\end{align*}

We now turn to the surrogate design problem.
To apply the Lov\'asz hinge to the one-hot encoded predictions, we need to choose submodular functions $f_{c,y}$ at the bit level.
For each $c\in \C$ and $y\in \Y$ pair, we define the polymatroid $f_{c ,y_o}:2^{[Ck]}\to \reals_{+}$ by $f_{c,y_o}(S)=g_{c,y}(\{i\in[k] : C(i-1)+c \in S\})$ such that $y_o=\phi \square (y)$.
This $f_{c,y}$ simply picks out the mispredicted bits that correspond to class $c$, condenses it to a set of individual mispredictions, and applies $g_{c,y}$.
Let us verify that we recover the desired equivalence.
For $y,y'\in[C]^k$, let $v=\phi\square(y), v'=\phi\square(y')$.
For any $c\in[C]$ we then have
\begin{align*}
  f_{c,y_o}(\mis(v',v))
  &= g_{c,y}(\{i\in[k] : C(i-1)+c \in \mis(v',v)\})
  \\
  &= g_{c,y}(\{i\in[k] : \phi(y_i')_c \neq \phi(y_i)_c\})
  \\
  &= g_{c,y}(\{i\in[k] : (\phi(y_i')_c = -1 \wedge \phi(y_i)_c = 1) \vee (\phi(y_i')_c = 1 \wedge \phi(y_i)_c = -1)\})
  \\
  &= g_{c,y}(\{i\in[k] : (y_i' \neq c \wedge y_i = c) \vee (y_i' = c \wedge y_i \neq c)\})
  \\
  &= g_{c,y}(\mis_c(y', y))~,
\end{align*}
as desired.
Thus, defining the collection $\vec f^C := \{f^C_{y_o}\}_y$ given by $f^C_y = \tfrac 1 C \sum_{c\in[C]} f_{c,y_o}$,
we have shown $f^C_{y_o}(\mis(\phi\square(y'),\phi\square(y))) = \ell^{g,C}(y',y)$.

At this point, we seem to have a faithful encoding of the desired target.
Applying the Lov\'asz hinge, from Theorem~\ref{thm:well-defined-cal}, we have that $L^{\vec f^C}$ is a consistent surrogate for the target
$\ell^{\vec f^C}_\abs : \{-1,0,1\}^{Ck}\times \phi \square (\Y) \to \reals_+$ given by
\begin{align}
\ell^{\vec f^C}_\abs(v,y) &= f^C_{y_o}(\mis(v,\phi\square(y))) + f^C_{y_o}(\mis(v,\phi\square(y))\setminus \abs(v)))~.\label{eq:multiclass-onehot-fabs}
\end{align}
(Technically, Theorem~\ref{thm:well-defined-cal} shows consistency when the set of labels is all of $\{-1,1\}^{Ck}$; we simply restrict to distributions supported only on $\phi\square(\Y)$.)

One issue with this target, and therefore the surrogate $L^{\vec f^C}$, is how to interpret the target predictions.
In particular, the target is defined on all combinations of bit predictions and abstentions $\{-1,0,1\}^{Ck}$, but many of these predictions are not meaningfully interpretable in a standard classification sense.
For example, on $C = 3$ classes, a prediction of $v = (1,1,-1)$ corresponds to a multi-label prediction in which we predict ``yes'' for the first two classes, and ``no'' for the last, despite the fact that this is not a valid label in $\phi \square (\Y)$.
Similarly, a report of $v = (0,0,1)$ implies abstaining on the first two classes, but predicting on the third.
It is not clear what such partial abstentions mean.

While there are problems to which this setup may be well-suited,%
\footnote{Partial abstention could be appropriate for problems like multi-label learning. We expect the surrogate $L^{\vec f^C}$ to be useful in this case. Reductions to single-label learning might perform well in practice through a link where one applies the argmax to the magnitudes within each block, and either predicts that class or abstain, according to what the threshold abstain link would do.  That is, for the purposes of choosing what to abstain on, plug the max of the magnitudes within each block into the threshold abstain link on the $k$ blocks.  For any $i$ not abstained on, predict the class corresponding to the argmax.}
ideally for multiclass classification with abstention we would be able to define target reports of the form $([C] \cup \{\bot\})^k$.
One potential avenue to arrive at these more natural reports would be to define $\phi(\bot) = \vec 0$ and show that such multi-label and multi-abstain predictions are dominated by predictions contained in $\V_1 := \phi\square\bigl(([C] \cup \{\bot\})^k\bigr) = (\{\chi_y : y \in \Y\} \cup \{\vec 0\})^k$.
Formally, if $\V_1$ is a representative set for $\ell^{\vec f^C}_{\abs}$, then $\phi$ becomes a bijection, and we could conclude that $L^{\vec f^C}$ embeds (and is consistent for) the target $\ell^{g,C}_{\abs} : ([C] \cup \{\bot\})^k \times [C]^k$ defined by
\begin{align*}
  \ell^{g,C}_\abs(v,y)
  &= \frac 1 C \sum_{c\in[C]} \left( g_{c,y}(\mis_c(v,y)\cup\abs(v)) + g_{c,y}(\mis_c(v,y)\setminus\abs(v)) \strut\right)
\end{align*}
Unfortunately, it is not clear whether $\V_1$ is representative for cases of interest.
We now turn to a different encoding that does allow us to restrict to target reports $([C] \cup \{\bot\})^k$.



\subsection{Multiclass structured abstain problem}\label{sec:multi-embed}
In this subsection, we encode the multiclass setting into a structured abstain problem, using the more compact \emph{binary-encoded prediction} (BEP) encoding from \citet{ramaswamy2018consistent}.
The key idea of their approach is to use only $d = \log_2 C$ bits for each class.
For example, if $C=8$, we might map class $5$ to its binary representation $(1,-1,1)$.
In exchange for fewer bits to represent each class, the BEP surrogate of \citet{ramaswamy2018consistent} needs to allow an abstain option.
As we already know the Lov\'asz hinge necessarily introduces some sort of abstention, the BEP surrogate a promising candidate for a multiclass structured abstain.

Unfortunately, the BEP encoding is fundamentally incompatible with one-vs-all targets, as we now demonstrate with an example.
\begin{example}[BEP encoding and one-vs-all targets]
  Consider $C=8$ classes, and $k=1$ individual prediction.
  Let $\phi:[8]\to\{-1,1\}^3$ use the binary representation, so that $\phi(1) = (-1,-1,1)$, $\phi(2) = (-1,1,-1)$, etc., with $\phi(8) = (-1,-1,-1)$.
  Let $y_b = (1,1,1) = \phi(7)$ encode the true class $y=7$, and $v_b = (1,-1,1) = \phi(5)$ the prediction corresponding to class $v=5$.
  Measuring one-vs-all error for class $5$, this prediction $v_b$ would be an error (predicting ``yes $5$'' when the label is ``not $5$'').
  Formally, we have $\mis_5(v,y) = \mis_5(5,7) = \{1\}$, as there is only one individual prediction.
  (Recall the definition of $\mis_c$ just following eq.~\eqref{eq:jaccard-ova}.)
  Note that $\mis(v_b,y_b) = \{2\}$, as only bit 2 was mispredicted.

  Making one more ``bit mistake'', however, to $v'_b = (1,-1,-1) = \phi(4)$ corresponding to class $v'=4$, the prediction would no longer be an error (predicting ``not $5$'' when the label is ``not $5$'').
  That is, $\mis_5(v',y) = \mis_5(4,7) = \emptyset$, while $\mis(v'_b,y_b) = \{2,3\}$, as bit 3 was mispredicted in addition to bit 2.
  
  Now consider any nontrivial submodular function $g$ at the individual prediction level for class $5$ and label $y$, and let $f$ be the corresponding function of the bit misprediction set.
  For these functions to be compatible, 
  we would need $f(\{2\})=f(\mis(v_b,y_b))=g(\mis_5(v,y))=g(\{1\})$, and $f(\{2,3\})=f(\mis(v_b',y_b))=g(\mis_5(v',y))=g(\emptyset)$, so that the loss of mispredictions at the bit level match those at the individual prediction level.
  If $g$ is nontrivial, however, then $f(\{2,3\})=g(\emptyset)<g(\{1\})=f(\{2\})$, contradicting submodularity of $f$.
  Thus $g$ could not correspond to any submodular function $f$ at the bit level.
\end{example}

We conclude that this BEP approach is not suitable for nontrivial one-vs-all targets.
In particular, it cannot capture the one-vs-all Jaccard loss.
Nonetheless, we find it to be an elegant and potentially useful approach to multiclass structured prediction with abstention.
In particular, it does accommodate the following natural multiclass target.

\begin{definition}[Multiclass structured abstain]\label{def:hypercube-ova-sap-trarget}
Given a set of polymatroids $\vec g \in \vec \F_k$, 
we define
$\ellabsvecg : ([C]\cup\{0\})^k \times [C]^k \to \reals_+$ by
\begin{align*}
  \ellabsvecg (v,y)  = g_y (\mis (v,y) \setminus  \abs(v)) +  g_y (\mis (v,y))~,
\end{align*}
where as usual $\mis(v,y) = \{i \in [k] \mid v_i  \neq y_i\}$ and $\abs (v)= \{i \in [k] \mid v_i = 0 \}$.
\end{definition}

To avoid the issue above where submodularity can fail at the bit level, we force all bit mispredictions to yield mispredictions at the individual prediction level.
To this end, define the encoded polymatroid $\vec f :2^{[dk]}\to \reals_{+}$ by $f_y(S_b) = g_y(\{ i\in [k]\mid S_{b,(i)} \neq \emptyset \})$.
Let $\phi : [C] \to \{-1,1\}^d$ be injective, and for simplicity assume $C = 2^{[d]}$.
We define the surrogate to be the Lov\'asz hinge on $\reals^{dk}$ but with the labels properly represented,
\begin{align}
  \overline{L}^{\vec f} : \reals^{dk} \times [C]^k \to \reals, \quad \overline{L}^{\vec f}(u,y) = L^{\vec f}(u,\phi(y))~.
\end{align}
We denote define the \emph{trimmed threshold-abstain link} $\overline{\psithresh}: \reals^{dk} \to ([C]\cup\{0\})^k$ by $\overline{\psithresh}(u)= \vartheta \square  \psithresh (u)$, where
$\vartheta : \{-1,0,1\}^d \to [C]\cup\{0\}$ by
\[
  \vartheta(v)= \biggl\{ 
  \begin{array}{lr}
    0 & ||v||_{0}> 0  \\
    \phi^{-1}(v) & o.w.
  \end{array} ~,
\]
which abstains on the individual prediction when any bit would abstain.
Given these ingredients, our main result of this section is the following.

\begin{theorem}\label{thm:binarymultical}
  The pair $(\overline{L}^{\vec f},\overline{\psithresh} )$ is calibrated for $\ellabsvecg $.
\end{theorem}
\begin{proof}
  From Theorem~\ref{thm:well-defined-cal}, the pair $(L^{\vec f},\psithresh)$ is calibrated with respect to $\ellabsvecf$.
  We show that partially abstaining within a block is dominated by abstaining on the whole block.
  The result then follows from Lemma~\ref{lemma:safe-link} and the definition of $\vartheta$.

  Consider any $y\in\{-1,1\}^{dk}$ and any $v\in\V:=\{-1,0,1\}^{dk}$ such that $v_j = 0$ for some $j$.
  Let $i = \block j$ and let $B = \{d(i-1)+1,\dots ,d(i-1)+d   \}$.
  Define $v'\in\V$ by $v'_{(i)} = \vec 0$ and $v'_{(a)} = v_{(a)}$ for $a \neq i$.
  (In other words, $v'$ is the same as $v$ but abstains on all of block $i$.)
  Now let $a\in[k]$.
  If $a=i$, then we have $(\mis(v,y))_{(a)} \neq \emptyset$, as $v_j = 0$, and $(\mis(v',y))_{(a)} = B$.
  We also have $(\mis(v',y)\setminus\abs(v',y))_{(a)} = \emptyset$.
  In particular, we have the following implications.
  \begin{align*}
    &(\mis(v',y))_{(a)} \neq \emptyset \implies (\mis(v,y))_{(a)} \neq \emptyset~,
    \\
    &(\mis(v',y)\setminus \abs(v'))_{(a)} \neq \emptyset \implies (\mis(v,y)\setminus \abs(v))_{(a)} \neq \emptyset~.
  \end{align*}
  If $a\neq i$, then $(\mis(v,y))_{(a)} = (\mis(v',y))_{(a)}$ and $(\abs(v))_{(a)} = (\abs(v'))_{(a)}$, and these same implications hold trivially.
  
  Now computing, we have
  \begin{align*}
    \ellabsvecf(v',y)
    &= f_y(  \mis (v',y)\setminus \abs (v') ) + f_y(  \mis (v',y) )
    \\
    &= g_y( \{a\in[k] \mid (\mis(v',y)\setminus \abs(v'))_{(a)} \neq \emptyset\} ) + g_y( \{a\in[k] \mid (\mis(v',y))_{(a)} \neq \emptyset\} )
    \\
    &\leq g_y( \{a\in[k] \mid (\mis(v,y)\setminus \abs(v))_{(a)} \neq \emptyset\} ) + g_y( \{a\in[k] \mid (\mis(v,y))_{(a)} \neq \emptyset\} )
    \\
    &= f_y(  \mis (v,y)\setminus \abs (v) ) + f_y(  \mis (v,y) )
    \\
    &= \ellabsvecf(v',y)~,
  \end{align*}
  where the inequality follows from the above implications (showing that the latter sets contain the former) and submodularity.
  Thus, $v'$ weakly dominates $v$.
\end{proof}


To demonstrate the theorem, consider the polymatroid $g\in\F_k$ given by $g(S) = \sqrt{|S|}$.
The corresponding target is
\begin{align*}
  \ellabsg (v,y)  = \sqrt{ |\mis (v,y)| - |\abs(v)|} + \sqrt{ |\mis (v,y)|}~,
\end{align*}
which captures the concept of diminishing marginal errors via the square-root function. 
Theorem~\ref{thm:binarymultical} states that $\overline{L}^f$ is a calibrated surrogate for $\ellabsg$, where
$f (S_b)=\sqrt{ | \{  i\in [k]\mid  S_{b,(i)} \neq \emptyset   \}  |} $.
(Recall that $S_b\subseteq 2^{[dk]}$ for $d = \log C$.)


As another example, consider $\vec{g}\in\vec{\F}_k$ given by $g_y(S) = \sum_{i\in S} w(y_i)$ for some weight function $w:\C\to\reals_+$.
The polymatroid $g_y$ captures the idea of a class-weighted false negative counting loss.
As $\mis$ cannot depend on $c$, these polymatroids $g_y$ can only capture false negatives and not false positives. 
The corresponding target is
\begin{align*}
  \ellabsvecg (v,y)  =  \sum_{i\in  \mis (v,y) \setminus  \abs(v)} w(y_i) + \sum_{i\in \mis (v,y)} w(y_i).
\end{align*}
Theorem~\ref{thm:binarymultical} states that $\overline{L}^{\vec{f}}$ is a calibrated surrogate for $\ellabsvecg$, where
\[ f_{y_b} (S_b)= \sum_{i\in [k] : S_{b,(i)} \neq \emptyset} w( \phi \square (y_{b})_{i})~. \]
These surrogates and variations may be useful for structured multiclass prediction problems such as semantic image segmentation.




\acks{
This material is based upon work supported by the National Science Foundation under Award Nos. 2202898 and IIS-2045347.}



\bibliography{sample}

\appendix

\section{Omitted Proofs}\label{app:omitted-proofs}


\begin{table}[h]
	\centering
	\begin{tabular}{@{\extracolsep{4pt}}ll}
		Notation & Explanation \\
		\toprule
		$k$ & Number of individual predictions\\
    $[k] := \{1,\ldots,k\}$ & Index set\\
		$y \in \Y = \{-1,1\}^k$  & Label space    \\
		$v \in \V = \{-1,0,1\}^k$  & (Abstain) prediction space    \\
   $\mis(y',y) := \{i \in [k] \mid y_i \neq y_i'\}$ & Set of mispredictions for $y',y \in\Y$\\
   $\mis(v,y) := \{i \in [k] \mid v_i \neq y_i\}$ & Set of mispredictions for $v\in\V$ and $y \in\Y$\\
  $\abs(v) = \{i \in [k] \mid  v_i = 0\}$ & Set of predictions abstained on \\
		$r \in \R$  & General prediction space \\
		$R = [-1,1]^k$ & The filled $\pm 1$ hypercube\\
		$u\in \reals^k$  & Surrogate prediction space    \\
		$\{u \leq c\} = \{i \in  [k] \mid u_i \leq c\}$ & Set of indices of $u$ less than $c$ \\
		$(u \odot u')_i = u_i u'_i$  & Hadamard (element-wise) product \\
		$U \odot u' = \{u \odot u' \mid u \in U\}$  & Hadamard product on a set $U \subseteq \reals^k$\\
		$\sign : \reals^k \to \V$  & Sign function including $0$\\
		$\signstar : \reals^k \to \Y$ & Sign function breaking ties arbitrarily at $0$\\
		$|u| \in \reals^k_+$ s.t. $|u|_i = |u_i|$ & Observe $|u| = u \odot \signstar(u) = u \odot \sign(u)$ \\
		$\boxed u = \sign(u) \odot \min(|u|, \ones)$ & ``Clipping'' of $u$ to $R$ \\
		$\ones_S \in \{0,1\}^k$ s.t. $(\ones_S)_i = 1 \iff i \in S$ & $0-1$ Indicator on set $S \subseteq [k]$\\
		$\pi \in \Sc_k$ & Permutations of $[k]$\\
         $\vec f = \{f_y \mid 2^{[k]}\to\reals_{+}\}_{y\in\Y}$ &  Collection of label  dependent \\ & polymatroid functions\\
		$\F_k= \{\vec{f}\mid  \exists \; f : 2^{[k]} \to \reals_{+} $ s.t. $ f_y =f \;\forall \; y\in \Y\}$  & Set of all symmetric polymatrioids\\
  $F(x)$ 
             & Lova\'sz extension for $x \in \reals_+^k$ \\
		$\ell^f(r,y) = f(\mis (r,y))$ & Structured binary classification s.t. $f \in \F_k$\\
		$L^f(u,y) = F((\ones - u \odot y)_+)$ & Lov\'asz hinge s.t. $f \in \F_k$ \\
		$ \ellabs ( v,y) = f(\mis(v,y)\setminus \abs(v))+ f(\mis(v,y))$  & Structured abstain problem s.t. $f \in \F_k$\\
        $\ellvecf (r,y) = f_{y}(\mis (r,y) )$ & Structured binary classification via $\vec{f}$\\
		$\Lvecf (u,y) = F_{y}((\ones - u \odot y)_+)$ & Lov\'asz hinge via $\vec{f}$ \\
		$ \ellabsvecf (v,y) = f_{y}(\mis(v,y)\setminus \abs(v))+ f_{y}(\mis(v,y))$ & Structured abstain problem via $\vec{f}$\\
		\bottomrule
	\end{tabular}
\caption{Table of general notation}\label{tab:notation}
\end{table}


\begin{table}[h]
	\centering
	\begin{tabular}{@{\extracolsep{4pt}}ll}
		Notation & Explanation \\
		\toprule
		$\ones_{\pi,i} = \ones_{\{\pi_1, \ldots, \pi_i\}}$ with $\ones_{\pi,0} = \vec 0$ & Indicator of first $i$ elements of $\pi$\\
	$V_{\pi} = \{ \ones_{\pi,i} \mid i \in \{0, \ldots, k\}\}$ & Elements of $\V$ ordered by $\pi$\\
		$V_{\pi,y} = V_\pi \odot y$ & Signed elements of $\V$  ordered by $\pi$\\
		  $V_{\pi,y}^{u} = V_\pi \odot y$ s.t. $\alpha_{i}(|u|)\neq 0$ & \\
		$P_\pi = \{ x \in \reals^k_+ \mid x_{\pi_1} \geq \ldots \geq x_{\pi_k}\}$ & elements of $\reals_+^k$ ordered by $\pi$\\
		$P_{\pi,y} = \conv (V_\pi ) \odot y$ & Elements of $P_\pi$ signed by $y$\\
		 $\Vfaces = \cup_{\pi \in \Sc_k, y \in \Y} 2^{V_{\pi,y}}$ & Subsets of $\V$ whose convex hulls are \\
		 & faces of some $P_{\pi,y}$ polytope.\\
		$\hat \Psi(u) = \cap \{V \in \Vfaces \mid d_\infty (\conv (V), u) < \epsilon\}$ & Proposed general link envelope.\\
		$\U^{\vec{f}} = \prop{L^{\vec{f}} } (\simplex )$ & Range of property elicited by Lov\'asz hinge\\
		$\Psi^{\vec{f}}(u) = \cap \{U \in \U^{\vec{f}} \mid d_\infty (U, u) < \epsilon\} \cap \V$ & Link envelope for given $f \in \F_k$.\\
		\bottomrule
	\end{tabular}
	\caption{Table of notation used for proofs}\label{tab:notation-proofs}
\end{table}



\subsection{Omitted proofs from \S~\ref{sec:our-embedding}}\label{app:omitted-sec-3}

\lovaszhypercubedominates*
\begin{proof}
	Fix $y\in\Y$.
	Let $w = \ones - u\odot y$ and $\boxed w = \ones - \boxed u\odot y$, so that $L^{\vec f}(u,y) = F_y(w_+)$ and $L^{\vec f}(\boxed u,y) = F_y(\boxed w_+)$.
	We will first show that $\boxed w_+ = \min(w_+,2)$, where the minimum is element-wise.

	For $i\in[k]$ such that $|u_i| \leq 1$, we have $\boxed u_i = u_i$.
	Thus $(w_+)_i = (1 - u_i y_i)_+ = (1 - \boxed u_i y_i)_+ = (\boxed w_+)_i$.
	Furthermore, we have $0 \leq (w_+)_i = (\boxed w_+)_i \leq 2$.
	Now suppose $|u_i| > 1$.
	If $y_iu_i > 0$, i.e., $\sign(u_i) = y_i$, then $1-y_iu_i = 1-|u_i| < 0$, so $(w_+)_i = 0$.
	For $\boxed u$, we similarly have $(\boxed w_+)_i = (1-|\boxed u_i|)_+ = 0$.
	In the other case, $y_iu_i < 0$, so $(w_+)_i = 1+|u_i| > 2$ and $(\boxed w_+)_i = 1+|\boxed u_i| = 2$.
	Therefore, we have $\boxed w_+ = \min(w_+, 2)$.

	Now, let $\pi\in\Sc_k$ be a permutation that orders the elements of $w_+$.
	Observe that $\pi$ orders the elements of $\boxed w_+$ as well, since the vectors are identical except for values above 2, which are all mapped to 2.
	By eq.~\eqref{eq:lovasz-ext-pi-u}, we thus have
	\begin{align*}
	F_y(w_+) - F_y(\boxed w_+)
	&= \sum_{i=1}^k (w_+)_{\pi_i} (f_y(\Spi{i})-f_y(\Spi{i-1}))
	\\
	& \qquad     -  \sum_{i=1}^k (\boxed w_+)_{\pi_i} (f_y(\Spi{i})-f_y(\Spi{i-1}))
	\\
	&= \sum_{i=1}^k (w_+-\boxed w_+)_{\pi_i} (f_y(\Spi{i})-f_y(\Spi{i-1}))
	\\
	& \geq 0~,
	\end{align*}
	where we have used the fact that $f_y$ is increasing and $\boxed w_+ \leq w_+$ element-wise.
	As $y$ was arbitrary, this holds for all $y \in \Y$.
\end{proof}



\ppiy*
\begin{proof}
	For (i), take any $u\in[-1,1]^k$.
	Letting $y = \signstar(u)$, we have $u \odot y = |u| \in \reals^k_+$.
	Taking $\pi$ to be any permutation ordering the elements of $u \odot y$, we have
	$u\odot y \in P_{\pi} \cap \reals^k_+$.
	Notice, since $u\odot y \in P_{\pi} \cap \reals^k_+$ and $u\in[-1,1]^k$, we additionally have $u \odot y = |u| \in P_\pi \cap [0,1]^k$.
	Since $\onespi{\pi}{i}$ for $i \in \{0,\ldots,k\}$ form $V_{\pi}$ and $P_{\pi}$ is the convex hull of points in $V_\pi$, showing there is an $\alpha$ such that $u = \sum_i \alpha_i \onespi{\pi}{i}$ suffices to conclude $u \in P_{\pi,y}$.
	We can write $u \odot y$ as the convex combination $u\odot y = \sum_{i=0}^k \alpha_i(u\odot y) \onespi{\pi}{i}$, as in eq.~\eqref{eq:alpha-combination}.
	Thus $u = u\odot y\odot y = \sum_{i=0}^k \alpha_i(u\odot y) \onespi{\pi}{i} \odot y$, so $u\in P_{\pi,y}$.
	Therefore,  every $u\in[-1,1]^k$ is in some $P_{\pi,y}$, we have $\cup_{y\in\Y,\pi\in\Sc_k} P_{\pi,y} \supseteq [-1,1]^k$.
	Moreover, every $P_{\pi,y}\subseteq [-1,1]^k$ by construction, and equality follows.

	For (ii), first observe for all $\pi\in\Sc_k$, the function $F_{y''}:\reals^{k}\to \reals_{+}$ is affine with respect to $u\in\reals^{k}$ for all $y''\in\Y$, immediately from eq.~\eqref{eq:lovasz-ext-pi-u}.
	To show $L^{\vec f}(\cdot, y') := F_{y''}((\ones - u\odot y')_+)$ is affine on $P_{\pi,y}$ for all $y,y',y''\in\Y, \pi\in\Sc_k$, it suffices to show there exists some $\pi'$ such that $\{\ones - u \odot y' \mid u \in P_{\pi,y}\} \subseteq P_{\pi'}$.
  Since $u \in P_{\pi,y} \implies u \in [-1,1]^k \implies L^{\vec f}(u,y') = F_{y''}(\ones - u\odot y')$, the result will follow.

	We construct $\pi'$, unraveling the permutation $\pi$ into two permutations, depending on the sign of $y\odot y'$.
	Recall from the discussion following eq.~\eqref{eq:lovasz-ext} that, since $y = \signstar(u)$, $\pi$ orders the elements of $u\odot y = |u|$ in decreasing order.
	Observe that $u \odot y' = u\odot (y\odot y)\odot y' = (u\odot y) \odot (y\odot y') = |u| \odot (y\odot y')$.
	Therefore $\pi$ orders the elements of $\ones - u\odot y'$ in increasing order among indices $i$ with $y_i y'_i > 0$, and decreasing order on the others.
	We construct $\pi'$ by concatenating the elements in $\{y\odot y' < 0\}$ sorted by $\pi$ with $\{y \odot y' \geq 0\}$ sorted by the reverse $\pi$, we have shown $\ones - u\odot y' \in P_{\pi'}$.
\end{proof}


We now introduce a lemma used in the proof of Lemma~\ref{lem:UcapR-union-faces}.
\begin{lemmac}\label{lem:subgradients-on-relint-affine}
	Let $L : \reals^k \to \reals_+$ be  a polyhedral function that is affine on the polyhedron $C$.
	For any $x \in \relint(C)$ and any $z \in C$, we have $\partial L(x) \subseteq \partial L(z)$.
\end{lemmac}
\begin{proof}
  Fix $x\in\relint(C)$.
  Since $L$ is affine on $C$, then there exists some $w' \in \reals^k, b\in\reals$ such that $L(z) = \inprod {w'} z + b$ for all $z \in C$.
  Thus, we have $L(z) - L(x) = (\inprod{w'}{z} + b) - (\inprod{w'}{x} + b) = \inprod{w'}{z-x}$ for all $z\in C$.

We claim that for all $w \in \partial L(x)$, and all $z \in C$, we have $\inprod{w}{z-x} = \inprod{w'}{z-x}$.
To prove this claim, observe that
\begin{equation}
  \label{eq:blahblah}
  \inprod{w'}{z-x} = L(z) - L(x) \geq \inprod{w}{z-x} \text{ for all } z\in C~,
\end{equation}
by the subgradient inequality and affineness of $L$ on $C$.

Assume for a contradiction that $\inprod{w'}{z-x} > \inprod{w}{z-x}$ for some $z \in C$.

Since $x \in \relint(C)$, there is an $\epsilon < 0$ such that $z' := x + \epsilon(z-x) \in C$.
Therefore, we have 
\begin{align*}
\inprod{w'}{z'-x} = \inprod{w'}{\epsilon(z-x)} = \epsilon\inprod{w'}{z-x} < \epsilon\inprod{w}{z-x} = \inprod{w}{\epsilon(z-x)} = \inprod{w}{z'-x}~,
\end{align*}
where we use the fact that $\epsilon < 0$ to flip the inequality.
We have now contradicted eq.~\eqref{eq:blahblah} for the point $z'$.

Since we now have $L(z) - L(x) = \inprod{w'}{z-x} = \inprod{w}{z-x}$ for all $z \in C$,
consider $w \in \partial L(x)$.
Then we have, for all $v \in \reals^k$,
\begin{align*}
L(v) - L(z) &= (L(v) - L(x)) + (L(x) - L(z)) & \\
&\geq \inprod{w}{v-x} + \inprod{w}{x-z} & \\
&=\inprod{w}{v-z}~, &
\end{align*}
where the inequality follows from the subgradient inequality and the claim.
Thus $w \in \partial L(z)$, which completes the proof.
\end{proof}

A corollary of Lemma~\ref{lem:subgradients-on-relint-affine} is that subdifferentials are constant on $\relint(C)$ for any face $C$ such that $L$ is affine as the subset inclusion holds in both directions.


\UcapRunionfaces*
\begin{proof}
	Fix $p \in \simplex$.
	For any $u \in \hat{C} \cap \prop{L}(p)$, there is some $\C' \subseteq \C$ such that $u \in C_j$ for all $C_j\in\C'$.
	For now, let us simply consider any $C_j \in \C'$.
	Observe that $u\in \relint (F_j)$ for exactly one face $F_j$ of $C_j$.

	By convexity of $L$, we have $u \in \prop{L}(p) \iff 0 \in \partial L(u;p)$.
	Moreover, as $u \in \relint(F_j)$, we have $\partial L(u;p)\subseteq \partial L(z;p)$ for all $z \in F_j$ by Lemma~\ref{lem:subgradients-on-relint-affine}.
	Thus, $0 \in \partial L(u;p)$ implies $0 \in \partial L(z;p)$ for all $z \in F_j$.
	Moreover, $0 \in \partial L(z;p)$ for all $z \in F_j$ if and only if $z \in \prop{L}(p)$ for all $z \in F_j$, and thus we have $F_j \subseteq \prop{L}(p)$.

	As the value $u$ and the index $j$ were arbitrary, this holds for all such faces in $G(u) := \cup \{F_j \subseteq C_j \in \C' \mid u \in \relint(F_j)\}$.
	Now, take $\mathfrak{F} = \{ G(u) \mid u \in \hat{C} \cap \prop{L}(p)\}$; hence $\prop{L}(p) \cap \hat{C} = \cup \mathfrak{F}$.
	Moreover, $\mathfrak{F} \subseteq \cup_i \faces(C_i)$.
\end{proof}


\subsection{Omitted proofs from \S~\ref{sec:inconsistency}}
\label{app:omitted-sec-4}

When relating to submodularity, we will often find it useful to rewrite the misprediction set $\mis (r,y)$ above in terms of two sets of labels: $S_v = \{ \signstar(v) > 0 \}$ and $S_y = \{ y > 0 \}$.
Then $\mis (v,y) = S_v \triangle S_y$, and thus
\begin{equation}
  \label{eq:lovasz-embeds-symdiff}
  \ellabsvecf (v,y) = f_y(S_v\triangle S_y\setminus \abs (v)) + f_y(S_v\triangle S_y\cup \abs (v))~,
\end{equation}
where $\triangle$ is the symmetric difference operator $S\triangle T := (S\setminus T) \cup (T\setminus S)$.
To avoid additional parentheses, throughout we assume $\triangle$ has operator precedence over $\setminus$, $\cap$, and $\cup$.

\twobarf* 
\begin{proof}
	Let $A_v =\{i\in[k] \mid v_i = 0\} $ and $B_v = [k]\setminus A_v $.
	Recall that $\bar p$ is the uniform distribution on $2^k$ labels.
	Then we have
	\begin{align*}
	\ellabs(v;\bar p)
	&= 2^{-k} \sum_{S\subseteq [k]} f(S_v\triangle S\setminus A_v ) + f(S_v\triangle S\cup A_v)
	\\
	&= 2^{-|B_v|} \sum_{T\subseteq B_v} f(T) + f(T\cup A_v )
	\\
	&= \frac 1 2 \; 2^{-|B_v|} \sum_{T\subseteq B_v} f(T) + f(B_v\setminus T) + f(T\cup A_v) + f((B_v\setminus T)\cup A_v )
	\\
	&\geq \frac 1 2 \left( f(B_v) + f(\emptyset) + f([k]) + f(A_v )  \right)
	\\
	&\geq \frac 1 2 \left( f([k]) + f([k]) \right) = f([k])
	\end{align*}
	where we use submodularity in both inequalities.
	The second statement follows from the second equality above after setting $A_v =\emptyset$, as then $B_v = [k]$ and thus $T$ ranges over all of $2^{[k]}$.
\end{proof}

\barf*
\begin{proof}
	The inequality follows from Lemma~\ref{lem:2-bar-f} with $r \in \Y$.
	Next, note that if $f$ is modular we trivially have $\bar f = f([k])/2$.
	If $f$ is submodular but not modular, we must have some $S\subseteq [k]$ and $i\in S$ such that $f(S) - f(S\setminus\{i\}) < f(\{i\})$.
	By submodularity, we conclude that $f([k]) - f([k]\setminus\{i\}) < f(\{i\})$ as well; rearranging, $f(\{i\}) + f([k]\setminus\{i\}) > f([k]) = f([k]) + f(\emptyset)$.
	Again examining the proof of Lemma~\ref{lem:2-bar-f}, we see that the first inequality must be strict, as we have one such $T\subseteq [k]$, namely $T=\{i\}$, for which the inequality in submodularity is strict.
\end{proof}

\begin{lemma}\label{lem:lovasz-symmetry}
  For all $u\in\reals^k$, $y,y'\in\Y$, and $f\in \F_k $, $L^{ f}(u,y) = L^{f}(u\odot y',y\odot y')$.
\end{lemma}
\begin{proof}
$ L^f(u\odot y',y\odot y') = F\bigl((\ones - (u \odot y') \odot (y \odot y'))_+\bigr) = F\bigl(((\ones - u \odot (y' \odot y') \odot y )_+\bigr) = F\bigl((\ones - u \odot y )_+\bigr) = L^f(u,y) $.
\end{proof}

\lovaszpropertysymmetry*
\begin{proof}
	We define $p\odot r \in \Delta(\Y)$ by $(p\odot r)_y = p_{y\odot r}$.
	\begin{align*}
	\prop{L^f}(p\odot r)
	&= \argmin_{u\in\reals^k} \sum_{y\in\Y} (p\odot r)_y L^f(u,y) \\
	&= \argmin_{u\in\reals^k} \sum_{y\in\Y} p_{y\odot r} L^f(u,y) & \text{Definition of $p\odot r$}\\
	&= \argmin_{u\in\reals^k} \sum_{y\in\Y} p_{y\odot r} L^f(u\odot r,y\odot r) & \text{Lemma~\ref{lem:lovasz-symmetry}}\\
	&= \argmin_{u\in\reals^k} \sum_{y'\in\Y} p_{y'} L^f(u\odot r,y') & \text{Substituting $y=y'\odot r$} \\
	&= \left(\argmin_{u'\in\reals^k} \sum_{y'\in\Y} p_{y'} L^f(u',y')\right)\odot r\\
	&= \prop{L^f}(p)\odot r
	\end{align*}
\end{proof}

\begin{lemma}\label{lem:opt-ones-zero-abs}	
Let $\overline p$ be the uniform distribution on $\Y$, i.e.  $p_y = 2^{-k}$ for all $y\in\Y$, and assume $k \geq 3$.
Moreover, let $\vec f$ be a family of polymatroid functions satisfying Condition~\ref{cond:inconsistency}.
Then $\{\vec 0\} \subseteq \prop{\ellabsvecf}(\overline p) \subseteq \{\vec 0, \ones\}$.
\end{lemma}
\begin{proof}
Let $\Gamma = \prop{\ellabsvecf}$, where $\vec f$ is an arbitrary polymatroid collection satisfying Condition~\ref{cond:inconsistency}.
Fix $v \in \{-1,0,1\}^k \setminus \{\vec 0, \ones\}$.
We abuse notation by denoting $f_T = f_y$, where $y = \chi_T$. 
We let $ A_v$ be used to denote $\abs (v)$, $S_v =\{i\in [k]: \signstar (v)>0   \}$, and $S_v \triangle T$ to denote $\mis (v,y)$.
When $v$ is understood from context, we will simply write $S$ and $A$.

First, we show $\{ \vec 0 \}  \in \Gamma(\overline p)$.
  \begin{align*}
    \ellabsvecf(v;\overline p)
    &= 2^{-k} \sum_{T\subseteq [k]} f_T(S\triangle T\setminus A) + f_T(S\triangle T\cup A)
    \\
    &= \frac 1 2 \; 2^{-k} \sum_{T\subseteq [k]} \biggl( f_T(S\triangle T\setminus A) + f_T(S\triangle T\cup A)
    \\[-12pt]
    & \hspace*{85pt}
      + f_{\overline T}(S\triangle \overline T\setminus A) + f_{\overline T}(S\triangle \overline T\cup A) \biggr) \qquad \text{Condition~\ref{cond:inconsistency}}
     \\
    &\geq \frac 1 2 \; 2^{-k} \sum_{T\subseteq [k]} \bigl(
      f_T([k]) + f_T\bigl(\emptyset) + f_{\overline T}([k]) + f_{\overline T}(\emptyset) \bigr)
    \\
    &= 2^{-k} \sum_{T\subseteq [k]} \bigl(
      f_T([k]) + f_T\bigl(\emptyset) \bigr)
    \\
    &= \ellabsvecf(\vec 0;\overline p)~.
  \end{align*}
  Hence, $\{\vec 0\} \subseteq \Gamma(\overline p)$.

Now, in order to show that $\Gamma(\overline p) \subseteq \{\vec 0, \ones\}$, it suffices to show that, for all $v \in \{-1,0,1\}^k \setminus \{\vec 0, \ones\}$, there is a $T \subseteq [k]$ such that the inequality in Condition~\ref{cond:inconsistency} is strict on either $S_v \Delta T \cup A_v$ or $S_v \Delta T \setminus A_v$.  
In particular, the inequality is strict on one of these if there is a $T$ such that either (i) $0 < |S \Delta T \setminus A| < k$ and $S \Delta T \setminus A \neq \overline T$ or (ii) $0 < |S \Delta T \cup A| < k$ and $S \Delta T \cup A \neq \overline{T}$.

In what follows, we show that for all $S$ and $A$ where $S\cap A=\emptyset$ and either $A\neq [k]$ or $S\neq [k]$ , there exists a $T\subseteq [k]$ such that $T\notin \{\emptyset ,[k] \}$ and $S\Delta T\setminus A\neq \overline{T}$ which makes the inequality in Condition~\ref{cond:inconsistency} strict.
To do so, we will often use the fact that $S\triangle T\setminus A = \overline{S\triangle \overline T\cup A}$ and $S\Delta T\cup A = \overline{S\Delta\overline{T}\setminus A}$.

\textbf{Case 1a.} ($0<|S|<k$, $A = \emptyset$):
  If $A=\emptyset$, let $T=\{i\}\cup\{j\}$ such that $i\in S$ and $j \in \overline{S}$, then $S\Delta T\setminus A = S\Delta T=(S\setminus \{i\})\cup \{j\}$ and hence $0< |S|-1+1=|S| = |S\Delta T\setminus A | < k$.
  Moreover, since $(S \Delta T\setminus A) \cap T = ((S\setminus \{i\})\cup \{j\}) \cap T \supseteq \{j\}$, then $S \Delta T\setminus A  = (S\setminus \{i\})\cup\{j\}\neq \overline{T}$. 
 
 \textbf{Case 1b.} ($0<|S|<k$, $A = \overline S$): 
 In general, when  $A=\overline{S}$, we have $S\Delta T\setminus A = S\Delta T\setminus \overline{S}=S\setminus T$. 
 If $|S| < k-1$, then consider $T = \{i\}$ for any $i \in \overline S$.
 As $|\overline S| > 1$, we immediately have $T \neq \overline S$.
 Moreover, $S \setminus T = S$, which implies $0 < |S \setminus T| = |S| < k-1 < k$, and the conclusion follows.
 
 Now if $|S| = k-1$, consider $T = \{i,j\}$ for $i \in \overline S$ and any $j \in S$.
 By construction we have $T \neq \overline S$.
 Moreover, $|S \setminus T| = |S| - 1 = k-2$.
 As $k \geq 3$, we immediately have $0 < |S \setminus T| = |S| = k-1 < k$.

  \textbf{Case 1c.} ($0<|S|<k$, $A \notin \{ \emptyset, [k], \overline S\}$):
  Consider $T = A$: as $A \neq \overline S$, we immediately have $T \neq \overline S$.
  Moreover, as $A$ and $S$ are disjoint, $S \Delta T \setminus A = (S \cup A) \setminus A = S$.
  By the case, $0 < |S| = |S \Delta T \setminus A| < k$ immediately follows.

\textbf{Case 2a.} ($S=\emptyset$, $A = \emptyset$):
In this case,  $S\Delta T\setminus A = T\setminus A = T$ for any $T\subseteq [k]$.
Thus, any choice of $T$ such that $T\notin \{\emptyset,[k]\}$ yields $ 0<|S\Delta T\setminus A |= |T|<k$.
Moreover, $S\Delta T\setminus A = T \neq \overline{T}$ holds as well.

\textbf{Case 2b.} ($S=\emptyset$, $A \not \in \{ \emptyset, [k] \}$): If $A\notin \{\emptyset ,[k] \}$, let $T=A$, then $S\Delta T \cup A = A$ and thus $0<|S\Delta T\cup A|= |A|<k$.
Moreover, $S\Delta T\cup A = A \neq \overline{A} = \overline{T}$ holds as well. 

This concludes the cases.
Thus showing strict suboptimality of any prediction $v \in \{-1, 0, 1\}^k \setminus \{\vec 0, \ones\}$, and the result $\Gamma(\overline p) \subseteq \{\vec 0, \ones\}$ follows.
\end{proof}


\subsection{Omitted proofs from \S~\ref{sec:constructing-link}}\label{app:omitted-sec-5}

Since $\boxed u \in R$, ``clipping'' $u'$ to $\boxed{u'}$ can only reduce element-wise distance, and therefore $d_\infty(\boxed u, \cdot)$ is still small, which allows us to restrict our attention to $R$.

\ignore{\begin{lemmac}\label{lem:truncated-u-also-close}\end{lemmac}}
\begin{restatable}{lemmac}{truncatedualsoclose}\label{lem:truncated-u-also-close}
	Let $\vec{f} \in \vec{\F_k}$.
	For all $U \in \prop{L^{\vec{f}}}(\simplex)$, $u\in\reals^k$, and $0<\epsilon<2$, if $d_\infty(U,u) < \epsilon$ then $d_\infty(U\cap [-1,1]^k, \boxed u) < \epsilon$.
\end{restatable}
\begin{proof}
	Since $U$ is closed, we have some closest point $u'\in U$ to $u$, meaning $d_\infty(u',u) = d_\infty(U,u) < \epsilon$.
	As $\boxed u' \in U$ by a corollary of Lemma~\ref{lem:lovasz-hypercube-dominates}, it suffices to show $d_\infty(\boxed u, \boxed u') < \epsilon$.
	For each $i \in [k]$, we consider three cases.
	It suffices to show distance does not increase on each element by the choice of the $d_\infty(\cdot,\cdot)$ distance.

	The cases are as follows: (i) $u_i = \boxed u_i$ and $u' = \boxed u'_i$, (ii) $u_i \neq \boxed u_i$ and $u'_i \neq \boxed u'_i$, and (iii) $u_i = \boxed u_i$ and $u'_i \neq \boxed u'_i$ (WLOG).
	Case (i) is trivial as $|u_i - u'_i| = |\boxed u_i - \boxed u'_i| < \epsilon$.
	In case (ii), we must have $\sign(u)_i = \sign(u')_i$ as $d_\infty(u,u') < \epsilon \implies |u_i - u'_i| < \epsilon$.
	If both $u_i$ and $u'_i$ are outside $[-1,1]^k$, this inequality is only true (for $\epsilon < 2)$ if the sign matches.
	Therefore $|\boxed u_i - \boxed u'_i| = |\sign(u)_i - \sign(u')_i| = 0 < \epsilon$.
	In case (iii), we have $\epsilon > |u_i - u'_i| > |u_i - 1| = |u_i - \boxed u'_i|$.
	As absolute difference in each element does not increase, the $d_\infty(\cdot, \cdot)$ distance does not increase.
\end{proof}

We now proceed to statements about the link envelope construction $\hat \Psi$.

\psicontainment*
\begin{proof}
  Let us define
  \begin{align*}
    \A(u) &:= 	\{V \in \Vfaces \mid d_\infty(\conv (V), \boxed u) < \epsilon\}~,
    \\
    \B(u) &:= \{U \cap \V \mid U\in \U^{
    \vec{f}}, \, d_\infty(U, u) < \epsilon\}~,
  \end{align*}
	so that $\hat \Psi(u) = \cap \A(u)$ and $\Psi^{\vec f}(u) = \cap \B(u)$.
  We wish to show $\cap \A(u) \subseteq \cap \B(u)$.
  It thus suffices to show the following claim: for all $B \in \B(u)$ we have some $A\in \A(u)$ with $A \subseteq B$.
  Since $v \in \cap \A(u)$ implies $v\in A$ for all $A\in\A(u)$, which by the claim implies $v\in B$ for all $B\in\B(u)$ and thus $v\in\cap\B(u)$.

  To do so, let $B \in \B(u)$, so we may write $B = U\cap \V$ for $U\in \U^{\vec{f}}$ with $d_\infty(U, u) < \epsilon$.
  By Lemma~\ref{lem:truncated-u-also-close} we have $d_\infty(U \cap R, \bar u) = d_\infty(U,u) < \epsilon$.
	From Lemma~\ref{lem:p-pi-y}, the set $R = [-1,1]^k = \cup_{\pi \in \Sc_k, y \in \Y} P_{\pi,y}$ is the union of polyhedral subsets of $\reals^k$, and $L^{\vec{f}}(\cdot,y')$ is affine on each $P_{\pi,y}$ for all $y,y'\in\Y$.
  By Lemma~\ref{lem:UcapR-union-faces}, we then have $U \cap R = \cup \F$ for some $\F \subseteq \cup_{\pi,y}\faces(P_{\pi,y})$.
  As each such face can be written as $\conv (V)$ for some $V\in\Vfaces$, we have some $\V' \subseteq \Vfaces$ such that $U \cap R = \cup \F = \cup_{V \in \V'}\conv (V)$.
  Now $\min_{V\in\V'} d_\infty(\conv (V), \boxed u) = d_\infty(U \cap R, \boxed u) < \epsilon$, so we have some $V \in \V'$ such that $d_\infty(\conv (V), \boxed u) < \epsilon$.
  Thus $V \in \A(u)$ by definition.
  As $\conv (V) \subseteq U\cap R$, we have $V = \conv (V) \cap \V \subseteq (U\cap R)\cap \V = U\cap \V = B$, which proves the claim.
\end{proof}

\begin{lemmac}\label{lem:smallest-vert-rep}
	Fix $u \in [-1,1]^k$, and consider $ \pi , y$ such that $u \in P_{\pi,y}$.
	Then $V_{\pi , y}^{u} := \{\onespi{\pi}{i} \odot y \mid i \in \{0,\ldots,k\}, \alpha_i(|u|) \neq 0 \}$
	is the
	smallest (in cardinality) set of vertices such that $V_{\pi , y}^{u} \subseteq   V_{\pi,y} $ and $u \in \mathrm{conv}(V_{\pi , y}^{u})$.
\end{lemmac}
\begin{proof}
	First, observe that $V^u_{\pi,y} \subseteq V_{\pi,y}$ by construction, as the first set is constructed the same as the second, with one additional constraint.
	Moreover, we have $u = \sum_{i=1}^k \alpha_i(|u|) u_i = \sum_{i : \alpha_i(|u|) \neq 0} \alpha_i(|u|) u_i \in \conv (V^u_{\pi,y})$.

	Now recall $P_{\pi,y}$ is a simplex (see ``Linear interpolation on simplices'' \citet[pg.\ 167]{bach2013learning}) thus, by properties of simplex, each $u\in P_{\pi,y}$ has a unique convex combination expressed by the vertices of $  V_{\pi,y}$ which are affinely independent~\citep[pg.\ 14, Thm 2.3]{brondsted2012introduction}.
	Therefore, every vertex $i$ with a non-zero weighting $\alpha_i(|u|)\neq0$ is necessary in order to express $u$ as a convex combination due to the affine independence of the vertices.
	Thus, $V_{\pi , y}^{u} := \{\onespi{\pi}{i} \odot y \mid i \in \{0,\ldots,k\}, \alpha_i(|u|) \neq 0 \}$, and as $|V_{\pi , y}^{u}| < \infty$, has to be the smallest (in cardinality) set of vertices such such that $V_{\pi , y}^{u} \subseteq V_{\pi,y}$ and $u \in \mathrm{conv}(V_{\pi , y}^{u})$.
\end{proof}

Moreover, $\hat\Psi$ is symmetric around signed permutations.
\begin{lemmac}\label{lem:hat-Psi-symmetry-y}
	For all $u\in \reals^k$, $y\in \Y$, and $\pi\in \Sc_k$, we have $\hat \Psi (\pi(u\odot y))=\pi(\hat \Psi (u)\odot y)$, where we define $(\pi x)_i = x_{\pi_i}$ and we extend this operation to sets.
\end{lemmac}
\begin{proof}
	The proof that the permutation part $(\hat\Psi(\pi u) = \pi\hat\Psi(u))$ is straightforward from the definition.
	For sign changes, observe $\boxed{u\odot y}=\sign(u\odot y)\dot \min(|u\odot y|,1)=\sign(u)\odot y\odot \min(|u|,1)=\boxed{u}\odot y$.
  The operation $u\mapsto u\odot y$ is an isometry for the infinity norm as a special case of signed permutations, here the identity permutation~\citep[Theorem 2.3]{chang1992certain}.
	For all closed  $U\subseteq\reals^k$, we therefore have
	$d_{\infty}(\boxed{u\odot y},U\odot y)=d_{\infty}(\boxed{u}\odot y, U\odot y)=d_{\infty}(\boxed{u},U)$.
	Therefore,
	\begin{align*}
	\hat\Psi(u\odot y)
	&= \cap \{ V \in \Vfaces \mid d_\infty (\conv (V), \boxed{u\odot y}) < \epsilon\} &
	\\
	&= \cap \{ V \in \Vfaces \mid d_\infty (\conv (V)\odot y, \boxed{u}) < \epsilon\} & \quad\text{$\boxed{u \odot y} = \boxed u \odot y$, and $\boxed u \odot y \odot y = \boxed u$} \\[-2pt]
	& & \text{with $d_\infty$ preserved under $\odot$.}  \\
	&= \cap \{ V \in \Vfaces \mid d_\infty (\conv (V), \boxed{u}) < \epsilon\}\odot y &
	\\
	&= \hat\Psi(u)\odot y~. &
	\end{align*}
\end{proof}

\hatpsichar*
\begin{proof}
	We will show the statement for $u\in \reals^k_+$ with $u_1 \geq \cdots \geq u_{k}$, i.e., where $u\in P_{\pi^*}$ where $\pi^*$ is the identity permutation.
	Lemma~\ref{lem:hat-Psi-symmetry-y} then gives the result, as we now argue.
	For any $u\in\reals^k$, let $\pi\in\Sc_k$ order the elements of $|u|$, and let $y=\signstar(u)$.
	Then $\pi(u\odot y) = \pi|u| \in P_{\pi^*}$.
	Once we show eq.~\eqref{eq:hat-psi-char} is true on the unsigned, ordered case, eq.~\eqref{eq:hat-psi-char} gives $\hat\Psi(\pi|u|) = \{\onespi{\pi^*}{i} \mid i\in\{0,1,\ldots,k\}, \; |u_{\pi_i}| \geq |u_{\pi_{i+1}}| + 2\epsilon \}$.
	Thus $\hat\Psi(u) = \hat\Psi(\pi(u\odot y)) = \pi(\hat\Psi(u)\odot y) = \{\pi(\onespi{\pi^*}{i}\odot y) \mid i\in\{0,1,\ldots,k\}, \; |u_{\pi_i}| \geq |u_{\pi_{i+1}}| + 2\epsilon \} = \{\onespi{\pi}{i}\odot \signstar(u) \mid i\in\{0,1,\ldots,k\}, \; |u_{\pi_i}| \geq |u_{\pi_{i+1}}| + 2\epsilon \}$.

	To begin, we show that for any $i\in \{ 0 ,1,\dots ,k\}$ where $u_{_i}<u_{i+1}+2\epsilon$, $\onespi{\pi^*}{i}\notin \hat \Psi (u)$ by the contrapositive.
	First, suppose that there exists an $i\in \{0, 1,\dots ,k \}$ such that $u_{i}<u_{i+1}+2\epsilon$.
	Since $u$ is ordered, we know that $0\leq u_{i} - u_{i+1} < 2\epsilon$.
	Let $z=\frac{u_i + u_{i+1}}{2}$ and define $\hat u$ such that $\hat{u}_i=z$ and $\hat{u}_{i+1}=z$ while every other index of $\hat u$ is equal to $u$.
	Observe $u_i -z < \epsilon$ and $z - u_{i+1} <  \epsilon$, and thus $d_{\infty}(u,\hat u) < \epsilon$ as $d_\infty(\cdot,\cdot)$ is measured component-wise.
	By Lemma~\ref{lem:smallest-vert-rep} and construction of $\alpha$ in the first paragraph of \S~\ref{sec:affine-decomposition}, we have $\alpha_{i}(\hat u)=\hat u_{i}-\hat u_{i+1}=0$, we have $\hat u \in \conv (V_i)$, where $V_i := V_{\pi^* , y}^{u}\setminus \{\onespi{\pi^*}{i}\}$.
	Since $\hat u \in \conv (V_i)$ and $d_{\infty}(\hat u,u)<\epsilon$, we have $V_i \supseteq \hat{\Psi}(u)$, and therefore, for any $i\in \{ 0 ,1,\dots , k\}$ such that $u_i<u_{i+1}+2\epsilon$, $\onespi{\pi^*}{i}\notin \hat \Psi (u)$.

	Now, for the converse, fix any $u\in P_{\pi^*}$ with $i \in \{ 0,1,\dots , k\}$ such that $u_i \geq u_{i+1}+2\epsilon$.
	For any $u'\in \reals^k$ such that $d_{\infty}(u,u')<\epsilon$, we claim that $\alpha_i (u') \neq 0$, and therefore $\onespi{\pi^*}{i}\in \hat \Psi(u)$.

	Assume there exists a $u'\in \reals^k$ such that $d_{\infty}(u,u')<\epsilon$ for some $i\in \{ 0,1,\dots , k\}$.
	Given that $d_{\infty}(u,u')<\epsilon$, $u'_{j}\in (u_j-\epsilon,u_j+\epsilon) \, \forall j \in \{0,\ldots, k\}$: namely, for $j = i$ and $i+1$. 
	However, since $u_i-u_{i+1}\geq 2 \epsilon$, $(u_i-\epsilon,u_i+\epsilon)\cap (u_{i+1}-\epsilon,u_{i+1}+\epsilon) = \emptyset$.
	Therefore, $\alpha_i(u') = u'_{i} - u'_{i+1} > 0$.
	By Lemma~\ref{lem:smallest-vert-rep},  we then have $\onespi{\pi^*}{i}\in V^u_{\pi^*,y}$, which is the smallest set $V$ such that $d_\infty(\conv (V), u) < \epsilon$, and is therefore in the intersection of all such sets; this intersection yields $\hat \Psi(u)$.
	Thus, we have $\hat\Psi(u) = \{\onespi{\pi}{i} \odot \signstar(u) \mid i\in\{0,1,\ldots,k\}, \; u_i \geq u_{i+1} + 2\epsilon \}$.
\end{proof}

\hatPsinonempty*
\begin{proof}
	By Lemma~\ref{lem:hat-Psi-symmetry-y}, it suffices to show the statement for $u\in\reals^k_+$.
	We will show the contrapositive in both directions: there exists $u\in\reals^k_+$ such that $\hat\Psi(u) = \emptyset$ if and only if $\epsilon > \tfrac 1 {2k}$.

	For any $u\in\reals^k_+$, define $u_{k+1} = -\epsilon$ and $u_{0} = 1+\epsilon$ as in Proposition~\ref{prop:hat-psi-char}.
	From the characterization in Proposition~\ref{prop:hat-psi-char} (eq.~\eqref{eq:hat-psi-char}), we have $\hat\Psi(u)=\emptyset$ if and only if
	\begin{equation}\label{eq:hat-psi-nonempty}
	u_i - u_{i+1} < 2\epsilon \text{ for all } i\in\{0,1,\ldots,k\}~.
	\end{equation}

	We may also write
	\begin{equation}\label{eq:u-telescope-trick}
	1+\epsilon = u_0 = u_{k+1} + \sum_{i=0}^{k} (u_i-u_{i+1}) = \sum_{i=0}^{k} (u_i-u_{i+1}) - \epsilon~.
	\end{equation}
	If there exists $u\in\reals^k_+$ with $\hat \Psi(u) = \emptyset$, then eq.~\eqref{eq:hat-psi-nonempty}~and~\eqref{eq:u-telescope-trick} together imply $1+2\epsilon = \sum_{i=0}^{k} (u_i-u_{i+1}) < (k+1)(2\epsilon)$, giving $\epsilon > \tfrac 1 {2k}$.
	For the converse, if $\epsilon > \tfrac 1 {2k}$, take $u\in\reals^k_+$ given by $u_i = \tfrac {2i-1}{2k}$.
	Then $u_0 - u_1 = 1+\epsilon - (1- \tfrac 1 {2k}) < 2\epsilon$ and $u_k-u_{k+1} = \tfrac 1 {2k} + \epsilon < 2\epsilon$, and for $i\in\{1,\ldots,k-1\}$, we have $u_{i+1}-u_i = \tfrac 1 k < 2\epsilon$, giving eq.~\eqref{eq:hat-psi-nonempty}.
\end{proof}


\section{Additional Experiment Details}\label{app:omitted-experiments}

\subsection{Binary metric details}

We define the following: True positive $TP(v,y)= \{i\in [k]:v_i=1 \wedge y_i=1 \}$, True negative $TN(v,y)= \{i\in [k]:v_i=-1 \wedge y_i=-1 \}$, False positive $FP(v,y) = \{i\in [k]:v_i=1 \wedge y_i=-1 \}$, and False negative $FN(v,y) = \{i\in [k]:v_i=-1 \wedge y_i=1 \}$ .

\begin{enumerate}
    \item Nonrejected Performance: measures the ability of the classifier to accurately classify nonrejected samples.
    \begin{enumerate}
        \item Accuracy: $\text{Acc}(\hat{S})= \frac{\sum_{j=1}^{N}|TP(v^{(j)},y^{(j)})|+|TN(v^{(j)},y^{(j)})|}{\sum_{j=1}^{N}|TP(v^{(j)},y^{(j)})|+|FP(v^{(j)},y^{(j)})|+|TN(v^{(j)},y^{(j)})|+|FN(v^{(j)},y^{(j)})|}$
        \item  Recall: $R_{nr}(\hat{S})= \frac{\sum_{j=1}^{N}|TP(v^{(j)},y^{(j)})|}{\sum_{j=1}^{N}|(TP(v^{(j)},y^{(j)})\cup FN(v^{(j)},y^{(j)}) |}$
        \item  Precision: $P_{nr}(\hat{S})= \frac{\sum_{j=1}^{N}|TP(v^{(j)},y^{(j)})|}{\sum_{j=1}^{N}|(TP(v^{(j)},y^{(j)})\cup FP(v^{(j)},y^{(j)}) |}$
        \item  IoU: $IoU(\hat{S})= \frac{\sum_{j=1}^{N}|TP(v^{(j)},y^{(j)})|}{\sum_{j=1}^{N}|(TP(v^{(j)},y^{(j)})\cup FP(v^{(j)},y^{(j)})\cup FN(v^{(j)},y^{(j)}) |}$
    \end{enumerate}
    \item Rejection quality: measures the ability of the classifier with rejection to make errors on rejected samples only.
    \begin{enumerate}
        \item Rejection Rate: $R(\hat{S})=\sum_{j=1}^{N}\frac{|\abs (v^{(j)})|}{k}$
        \item Rejection Rate Positive: $RP(\hat{S})=\frac{\sum_{j=1}^{N}|\{i\in[k]:y^{(j)}_{i} =1 \wedge v^{(j)}_{i}=0\}|}{\sum_{j=1}^{N}|\abs (v^{(j)})|}$
        \item Rejection Rate Negative: $RN(\hat{S})=\frac{\sum_{j=1}^{N}|\{i\in[k]:y^{(j)}_{i} =-1 \wedge v^{(j)}_{i}=0\}|}{\sum_{j=1}^{N}|\abs (v^{(j)})|}$
    \end{enumerate}
\end{enumerate}



\subsection{Model training details}\label{app:model-training}

 For training of the model, we performed 20 epochs using SGD as our optimizer.
 We set an initial learning rate of $.1$ along with exponential decay over the learning rate with a decay rate of $.96$ for every $10000$ steps, and gradient clipping between $\pm1$ to help stabilize the training.
 We additionally saved the model's best weights with respect to the validation loss after each epoch and performed evaluation over test data using the best weights from training with respect to validation loss.



\end{document}